\documentclass[conference]{IEEEtran}

\usepackage{times}

\usepackage[numbers]{natbib}
\usepackage{multicol}
\usepackage[bookmarks=true]{hyperref}
\usepackage{wrapfig}
\usepackage{graphics} %
\usepackage{epsfig} %
\usepackage{times} %
\usepackage{amsmath} %
\usepackage{amssymb}  %
\usepackage{xcolor} %
\usepackage[ruled,vlined]{algorithm2e}
\usepackage{mdframed}
\usepackage{subcaption}
\usepackage{hyperref}
\usepackage{graphicx}
\usepackage{placeins}
\usepackage{booktabs}
\usepackage{float}
\usepackage{tikz}
\usetikzlibrary{arrows.meta, positioning}
\usepackage{amsthm}
\usepackage[breakable]{tcolorbox}
\newtheorem{theorem}{Theorem}
\newtheorem{corollary}{Corollary}
\newtheorem{remark}{Remark}

\newif\ifshowcomments
\showcommentsfalse

\newcounter{mypar}[subsection]

\newcommand{\myparagraph}[1]{%
  \refstepcounter{mypar}%
  \vspace{0.05in}\noindent\textbf{#1}%
}

\newcounter{result}[subsection]
\renewcommand{\theresult}{\thesubsection-\Roman{result}}  %

\newenvironment{myresult}[2][]{%
  \refstepcounter{result}%
  \ifx&#1&\else\label{#1}\fi%
  \begin{mdframed}[linewidth=0.5pt, linecolor=black, backgroundcolor=gray!5]
  \textbf{Result~\theresult}~(#2):%
}{%
  \end{mdframed}%
}

\newcommand{\resultref}[1]{\hyperref[#1]{Result~\ref*{#1}}}

\usepackage{etoc}

\begin{document}

\title{Tune to Learn: How Controller Gains \\ Shape Robot Policy Learning}

% \author{
% \authorblockN{Antonia Bronars$^*$}\\
% \authorblockA{
% MIT\\
% $^*$Equal Contribution\\}
% \and
% \authorblockN{Younghyo Park$^*$}\\

% \authorblockA{
% MIT\\
% $^*$Equal Contribution\\}
% \and
% \authorblockN{Pulkit Agrawal}\\
% \authorblockA{
% MIT\\
% Improbable AI Lab
% }
% }

\author{
\begin{tabular}[t]{c}
{\large Antonia Bronars$^*$}\\
MIT\\
$^*$Equal Contribution
\end{tabular}
\and
\begin{tabular}[t]{c}
{\large Younghyo Park$^*$}\\
MIT\\
$^*$Equal Contribution
\end{tabular}
\and
\begin{tabular}[t]{c}
{\large Pulkit Agrawal}\\
MIT\\
Improbable AI Lab
\end{tabular}
}

\maketitle

\begin{abstract}

Position controllers have become the dominant interface for executing learned manipulation policies. Yet a critical design decision remains understudied: how should we choose controller gains for policy learning? The conventional wisdom is to select gains based on desired task compliance or stiffness. However, this logic breaks down when controllers are paired with state-conditioned policies: effective stiffness emerges from the interplay between learned reactions and control dynamics, not from gains alone. We argue that gain selection should instead be guided by learnability: how amenable different gain settings are to the learning algorithm in use. In this work, we systematically investigate how position controller gains affect three core components of modern robot learning pipelines: behavior cloning, reinforcement learning from scratch, and sim-to-real transfer. Through extensive experiments across multiple tasks and robot embodiments, we find that: (1) behavior cloning benefits from compliant and overdamped gain regimes, (2) reinforcement learning can succeed across all gain regimes given compatible hyperparameter tuning, and (3) sim-to-real transfer is harmed by stiff and overdamped gain regimes. These findings reveal that optimal gain selection depends not on the desired task behavior, but on the learning paradigm employed. Project website: \href{https://younghyopark.me/tune-to-learn}{https://younghyopark.me/tune-to-learn}

\end{abstract}

\section{INTRODUCTION}

Position controllers are rapidly becoming the de facto choice for low-level control in robot learning. Their wide hardware support and intuitive nature have made them the dominant interface for executing learned manipulation policies. Yet while classical control theory provides clear guidance on selecting gains to achieve desired tracking bandwidth, disturbance rejection, or impedance characteristics, no analogous principles exist for the learning setting. An important design decision remains overlooked: how should we choose controller gains when \textit{learning} data-driven manipulation policies?

The standard approach treats gain selection as a problem of achieving desired task behavior—contact-rich manipulation calls for compliant gains to better comply with unexpected contacts, while precision tasks call for stiff gains to accurately track position commands. But this framing conflates two distinct roles that position controllers play. When tracking open-loop trajectories, the controller \textit{is} the \textit{behavior}—gains directly determine how the robot responds. When paired with a learned policy, however, the controller becomes an \textit{interface} between the policy and the physical world. The policy learns through this interface during training and acts through this interface at deployment. Viewed this way, gains function less as behavioral parameters and more as an \textit{inductive bias}—an implicit prior over the space of closed-loop behaviors that shapes what the policy can easily express and learn.

This distinction matters because learned policies are reactive: they observe the current state and output corrective commands. A policy can achieve stiff or compliant task-level behavior regardless of the underlying joint-level gains, simply by modulating the magnitude and timing of its outputs. The gains, therefore, do not determine the set of achievable closed-loop behaviors. We hypothesize that the gains instead shape the learning problem: how easy it is to fit action labels and how errors compound during closed-loop execution, which training configurations yield successful RL policies, and whether modeling discrepancies amplify into instability during sim-to-real transfer.

\begin{figure}[t]
    \centering
    \includegraphics[width=\columnwidth]{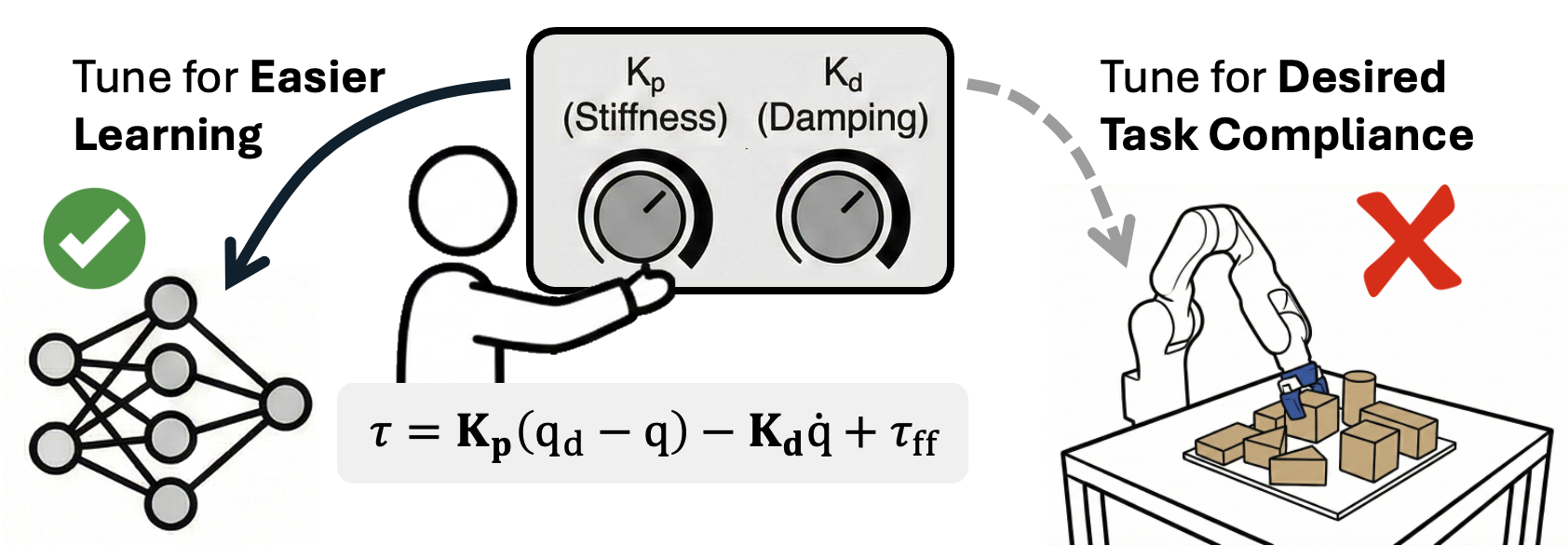}
    \vfill

    \begin{subfigure}[t]{0.32\columnwidth}
        \centering
        \includegraphics[width=\linewidth]{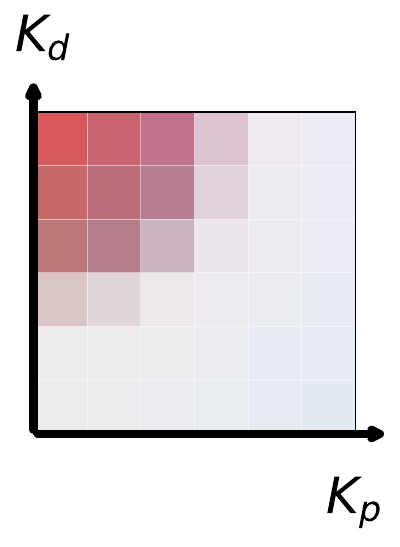}
        \caption{BC}
        \label{fig:sub1}
    \end{subfigure}
    \hfill
    \begin{subfigure}[t]{0.32\columnwidth}
        \centering
        \includegraphics[width=\linewidth]{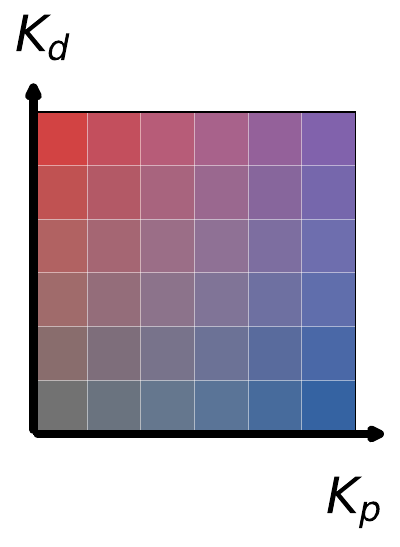}
        \caption{RL}
        \label{fig:sub2}
    \end{subfigure}
    \hfill
    \begin{subfigure}[t]{0.32\columnwidth}
        \centering
        \includegraphics[width=\linewidth]{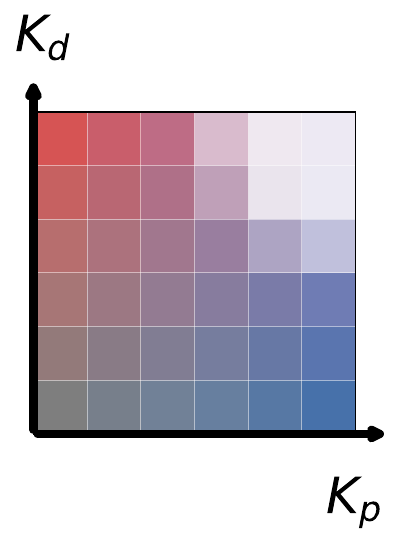}
        \caption{Sim2Real}
        \label{fig:sub3}
    \end{subfigure}

\caption{\textbf{Different robot learning paradigms prefer different controller gain interfaces.} Colored regions indicate gain regimes where each paradigm succeeds. Contrary to conventional wisdom of tuning gains for desired task compliance, optimal gains depend on the learning paradigm. Based on our experimental findings, heatmaps illustrate representative gain preferences for (a) behavior cloning, which favors compliant, overdamped gains, (b) reinforcement learning, which adapts to nearly any setting, and (c) sim-to-real transfer, which is degraded by stiff and overdamped gains.}
\label{fig:three-horizontal}
\end{figure}

\begin{figure*}[t]
\centering

\begin{tabular}{@{}c@{\hspace{1.5em}}c@{}}

&

\begin{minipage}{0.95\textwidth}
  \centering
  \subcaptionbox{\underline{C}ompliant and \underline{O}verdamped (CO)\label{fig:resp-unload}}{%
    \includegraphics[width=0.48\linewidth]{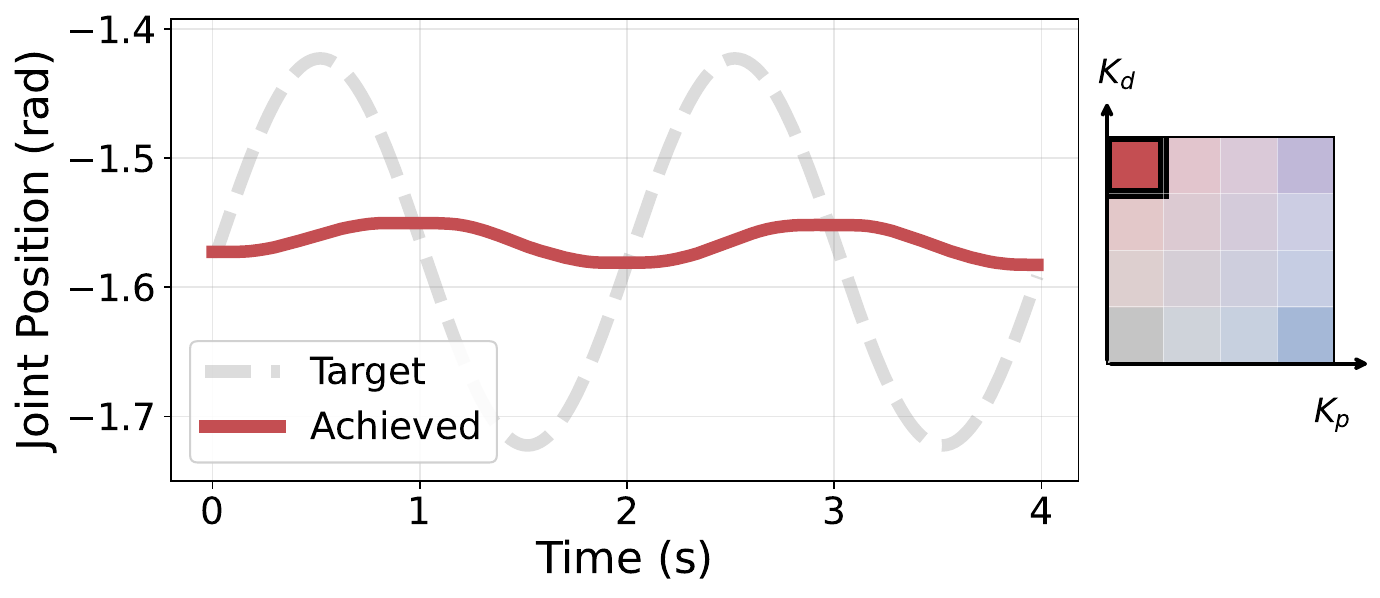}
  }\hfill
  \subcaptionbox{\underline{S}tiff and \underline{O}verdamped (SO) \label{fig:resp-dw}}{%
    \includegraphics[width=0.48\linewidth]{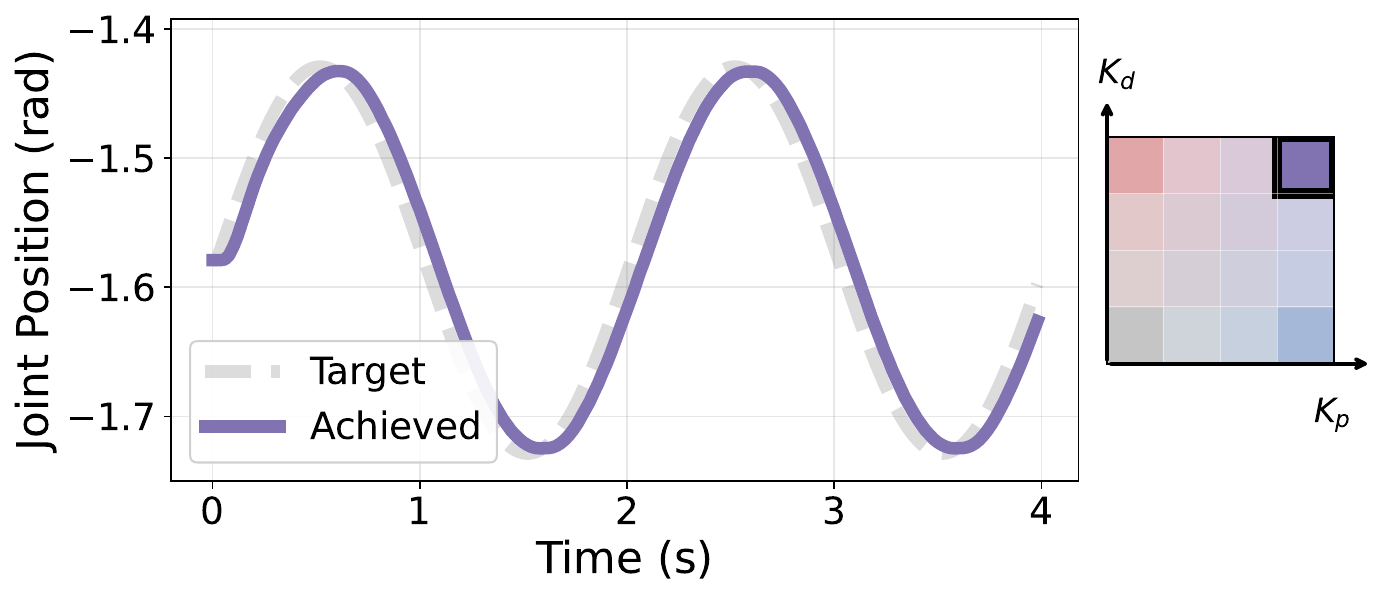}
  }

  \vspace{2mm}

  \subcaptionbox{\underline{C}ompliant and \underline{U}nderdamped (CU)\label{fig:resp-mug}}{%
    \includegraphics[width=0.48\linewidth]{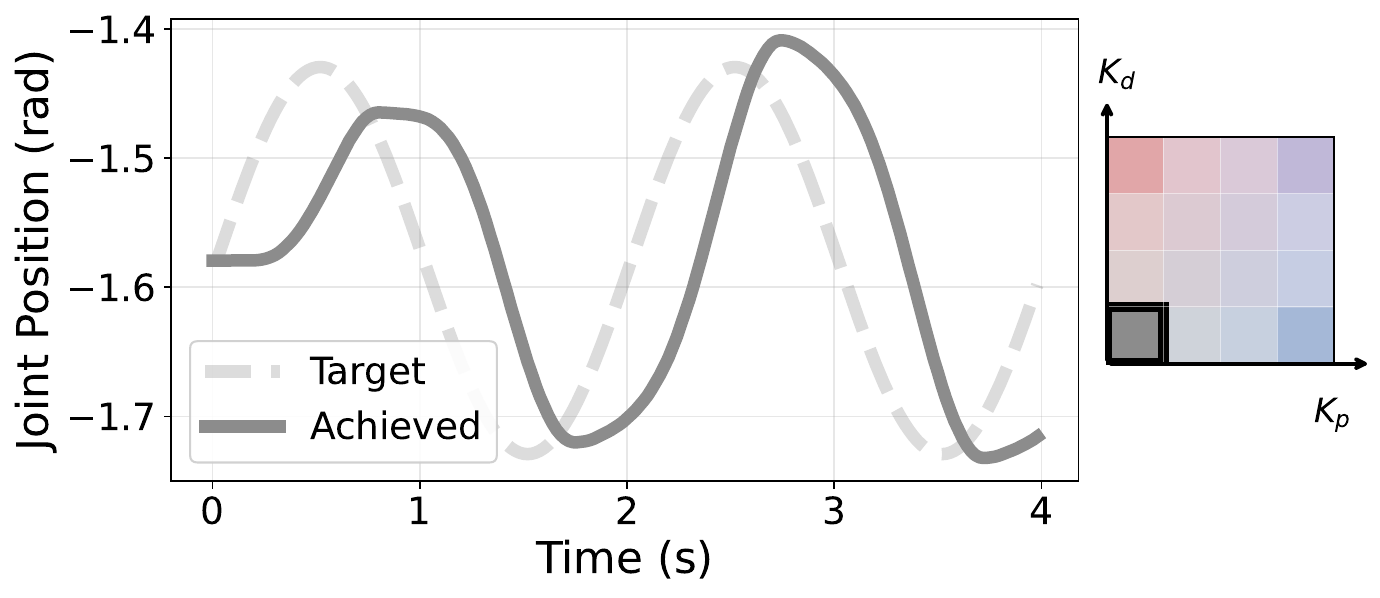}
  }\hfill
  \subcaptionbox{\underline{S}tiff and \underline{U}nderdamped (SU)\label{fig:resp-ur}}{%
    \includegraphics[width=0.48\linewidth]{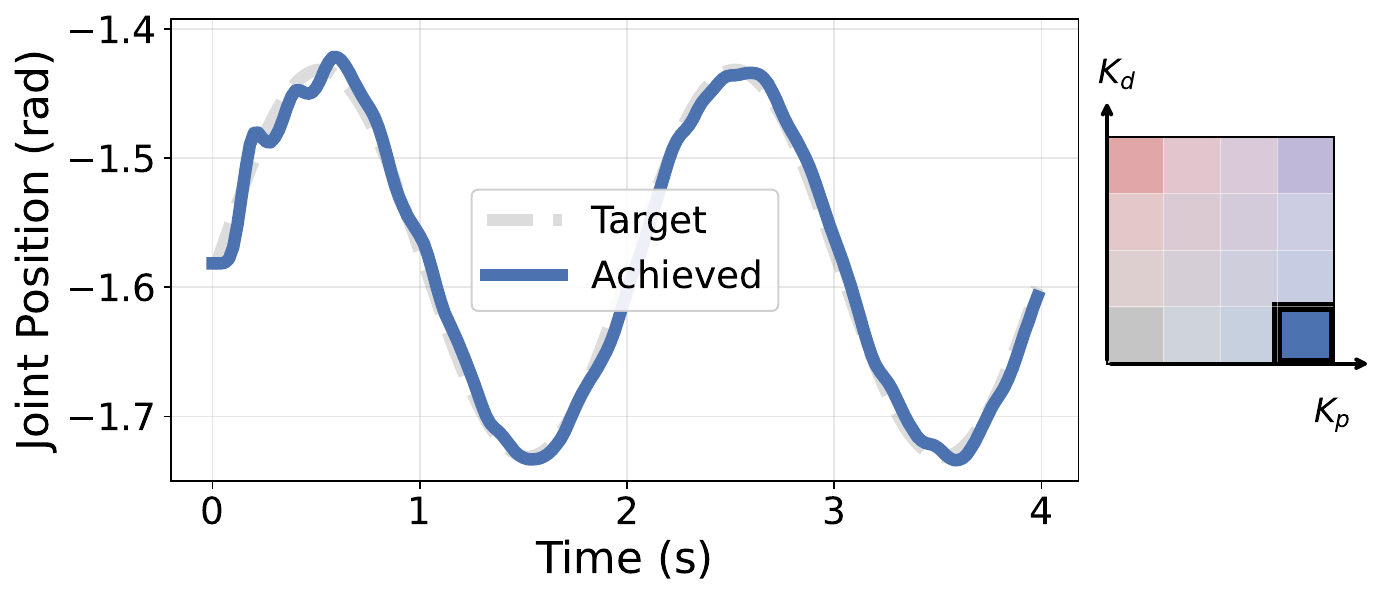}
  }
\end{minipage}

\end{tabular}

\caption{\textbf{Controller gains induce diverse action–response dynamics.} We evaluate a broad range of representative gain configurations and their resulting dynamic responses to assess their impact on learnability.}
\label{fig:gain-configs}
\end{figure*}

Once we recognize controller gains as learning interface parameters rather than behavioral parameters, the design question becomes: which interface properties facilitate learning? And critically, do different learning paradigms prefer different interfaces, serving as a conducive \textit{inductive bias}? We investigate these questions systematically across three paradigms of modern robot learning and present the following findings:
\begin{enumerate}
\item Behavior cloning performs best with \textit{compliant} and \textit{overdamped} gains. Across multiple manipulation tasks with controlled datasets that isolate the effect of gains, we show this regime yields higher closed-loop policy success rates without penalizing teleoperation efficiency. (Sec.~\ref{sec:exp-bc} and \ref{sec:result-bc})

\item Reinforcement learning (RL) from scratch is agnostic to gain setting, as long as the remaining hyperparameters are tuned to be compatible with the given gain setting. We verify this by obtaining equivalently successful RL policies for all gain regimes across multiple manipulation and locomotion tasks. (Sec. \ref{sec:exp-rl} and \ref{sec:results-rl})

\item When transferring policies from simulation to the real-world, \textit{stiff} and \textit{overdamped} controllers exacerbate the motor-level sim-to-real gap. (Sec. \ref{sec:exp-sim2real} and \ref{sec:results-sim2real}) 
\end{enumerate}
Our findings converge on a unified picture of how controller gains shape learning, providing both conceptual clarity and practical guidance for this widely used yet underexplored design decision.

\section{Related Works}\label{sec:related-works}

\subsection{Position and Impedance Control}

Position and impedance control have long been foundational for robot manipulation. Consider a robot manipulator with joint positions $\mathbf{q} \in \mathbb{R}^n$ governed by the dynamics:
\begin{equation}
\mathbf{M}(\mathbf{q})\ddot{\mathbf{q}} + \mathbf{C}(\mathbf{q}, \dot{\mathbf{q}})\dot{\mathbf{q}} + \mathbf{g}(\mathbf{q}) = \boldsymbol{\tau} + \boldsymbol{\tau}_{\text{ext}}
\end{equation}
where $\mathbf{M}(\mathbf{q})$ is the inertia matrix, $\mathbf{C}(\mathbf{q}, \dot{\mathbf{q}})$ captures Coriolis and centrifugal effects, $\mathbf{g}(\mathbf{q})$ is the gravity vector, $\boldsymbol{\tau}$ is the control torque, and $\boldsymbol{\tau}_{\text{ext}}$ represents environmental torques.

Takegaki and Arimoto \cite{takegaki1981new} established the global asymptotic stability of PD control with gravity compensation:
\begin{equation}
\boldsymbol{\tau} = \mathbf{K}_p(\mathbf{q}_d - \mathbf{q}) + \mathbf{K}_d(\dot{\mathbf{q}}_d - \dot{\mathbf{q}}) + \mathbf{g}(\mathbf{q})
\label{eq:pd_control}
\end{equation}
where $\mathbf{K}_p, \mathbf{K}_d \in \mathbb{R}^{n \times n}$ are gain matrices representing joint stiffness and damping, respectively, and $\mathbf{q}_d, \dot{\mathbf{q}}_d$ are the desired joint positions and velocities. This control law can be interpreted as \textit{joint-space impedance control}: in the absence of external torques, the closed-loop system behaves as a virtual spring-damper attached to the desired configuration, with the impedance relationship:
\begin{equation}
\boldsymbol{\tau}_{\text{ext}} = \mathbf{K}_p(\mathbf{q} - \mathbf{q}_d) + \mathbf{K}_d(\dot{\mathbf{q}} - \dot{\mathbf{q}}_d).
\end{equation}
This formulation is prevalent in modern robot learning, where policies typically output joint position targets $\mathbf{q}_d$ that are tracked by a low-level PD controller. Kelly \cite{kelly1997pd} provided a comprehensive review analyzing equilibrium uniqueness and stability robustness against parametric uncertainties. Despite the theoretical foundations, gain selection remains largely heuristic in practice.

\subsection{Low-Level Control in Robot Learning}

\myparagraph{Action Spaces.} Position-controlled action spaces (whether commanding joint positions or end-effector poses) convert policy outputs to motor torques via feedback control laws, making controller gains an implicit component of action space design. Aljalbout \emph{et al.}~\cite{Aljalbout_2024} find that action spaces based on control abstractions (e.g., PD-controlled positions) generally outperform torque control, though they do not vary gains within each paradigm. Kim \emph{et al.}~\cite{Kim_2023} argue that torque control's inherent compliance mitigates the sim-to-real gap. E{\ss}er \emph{et al.}~\cite{esser2024action} frame action space selection as an \emph{inductive bias} for locomotion---a perspective we extend to gain selection within position-controlled manipulation. Our study complements these works by isolating the effect of PD gains while holding the action representation fixed.

\myparagraph{Learning Task-level Impedance Policies.} Several works train policies that exhibit compliant manipulation behavior, either by learning variable stiffness profiles from demonstrations~\cite{wu2021framework, kronander2014learning} or by conditioning on user-specified task-space stiffness~\cite{margolis2025softmimic}. Notably, Margolis \emph{et al.}~\cite{margolis2025softmimic} observe that a policy's compliance is dictated by its training incentives, not the underlying PD gains, i.e., a policy can learn soft or stiff interactions regardless of the low-level controller. This distinction motivates our study: rather than learning compliance as a task-level behavior, we investigate how \emph{fixed} gain settings shape the learning process itself. Arachchige \emph{et al.}~\cite{arachchige2025sail} also vary gains of their underlying position controllers, but focus on speeding up execution of pretrained policies rather than understanding how gains affect learning.

\myparagraph{Sim-to-Real Transfer and Controller Fidelity.} Domain randomization~\cite{tobin2017domainrandomizationtransferringdeep, Peng_2018} has become standard for bridging the sim-to-real gap, with dynamics randomization typically including controller-related parameters such as PD gains, motor strengths, and joint damping~\cite{openai2019solvingrubikscuberobot}. Muratore \emph{et al.}~\cite{muratore2022robotlearningrandomizedsimulations} provide a comprehensive review of sim-to-real transfer via randomized simulations, noting that contact and friction models (which interact strongly with controller gains) remain among the most challenging aspects to transfer. Control frequency has also 
been shown to affect transfer fidelity: Gangapurwala\emph{et al.}~\cite{gangapurwala2023learning} show that low-frequency policies are less sensitive to actuation dynamics, enabling successful sim-to-real transfer without dynamics randomization. Despite this, systematic study of how the nominal gain settings (around which randomization occurs) affect sim-to-real transfer is lacking. Understanding which gain regimes transfer more robustly can inform both the choice of nominal gains and the design of more targeted randomization ranges, rather than treating all gain configurations as equally viable starting points.

\subsection{Gain Settings in Large-Scale Robot Datasets}\label{sec:gain-settings-datasets}

\begin{figure}[t]
    \centering

    \begin{subfigure}[t]{0.49\columnwidth}
        \centering
        \includegraphics[width=\linewidth]{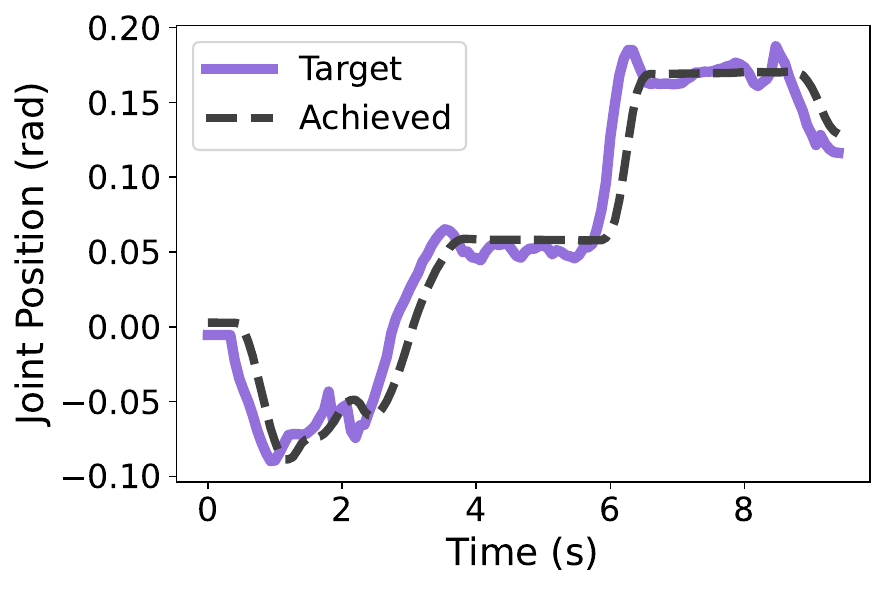}
        \caption{DROID}
        \label{fig:droid-response-curve}
    \end{subfigure}
    \begin{subfigure}[t]{0.49\columnwidth}
        \centering
        \includegraphics[width=\linewidth]{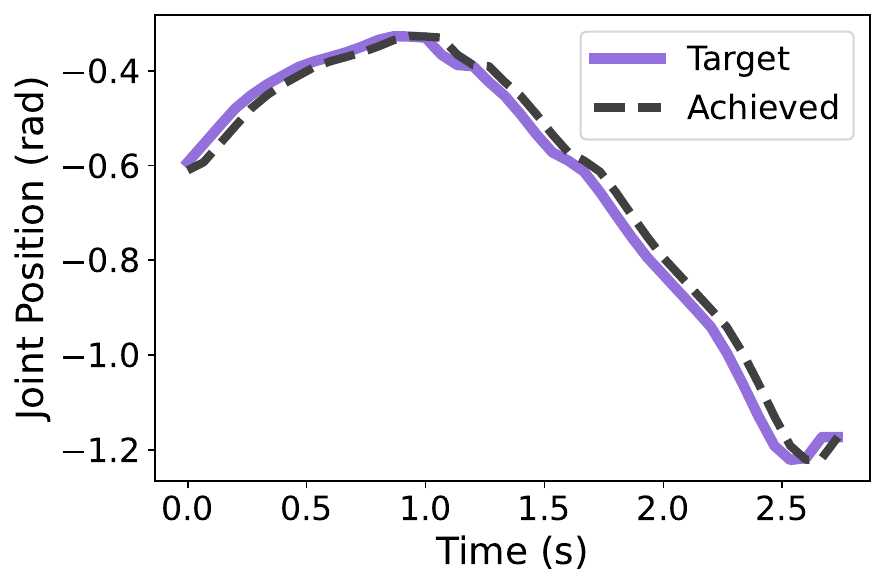}
        \caption{RT-X NYU Franka Play}
        \label{fig:rtx-response-curve}
    \end{subfigure}

\caption{Tracking response curves from existing robot datasets reveal tight command-following behavior, suggesting stiff controller gains are prevalent in existing data collection pipelines.}
    \label{fig:dataset-tracking}
\end{figure}

While controller gains fundamentally shape the learning interface, their configuration in existing large-scale datasets remains largely undocumented. To understand current practices, we analyzed DROID~\cite{khazatsky2024droid} and several datasets within the Open X-Embodiment collection~\cite{vuong2023open} by examining the relationship between commanded and achieved joint positions. Although exact gain values are rarely reported, tracking behavior reveals controller characteristics. As shown in Fig.~\ref{fig:dataset-tracking}, achieved positions closely track commands with minimal lag and overshoot, characteristic of stiff controllers. This pattern was prevalent across datasets, suggesting stiff gains have become an implicit default in data collection.

\section{Decoupling Gains from Task Compliance}

\begin{figure}[t]
    \centering

    \begin{subfigure}[t]{0.48\columnwidth}
        \centering
        \includegraphics[width=\linewidth]{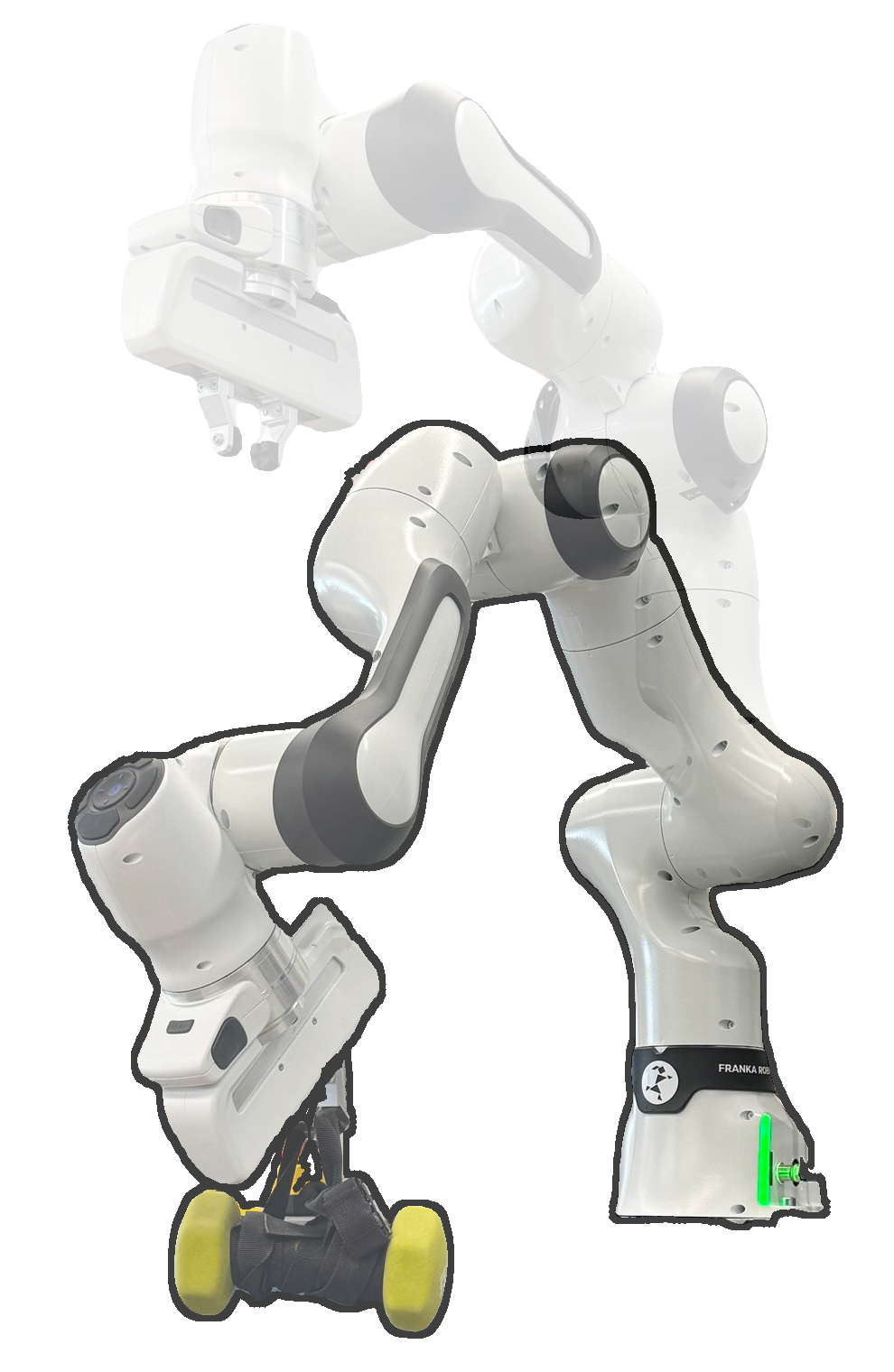}
        \caption{Compliance w/ stiff gain}
        \label{fig:compliant-yet-stiff}
    \end{subfigure}
    \begin{subfigure}[t]{0.48\columnwidth}
        \centering
        \includegraphics[width=\linewidth]{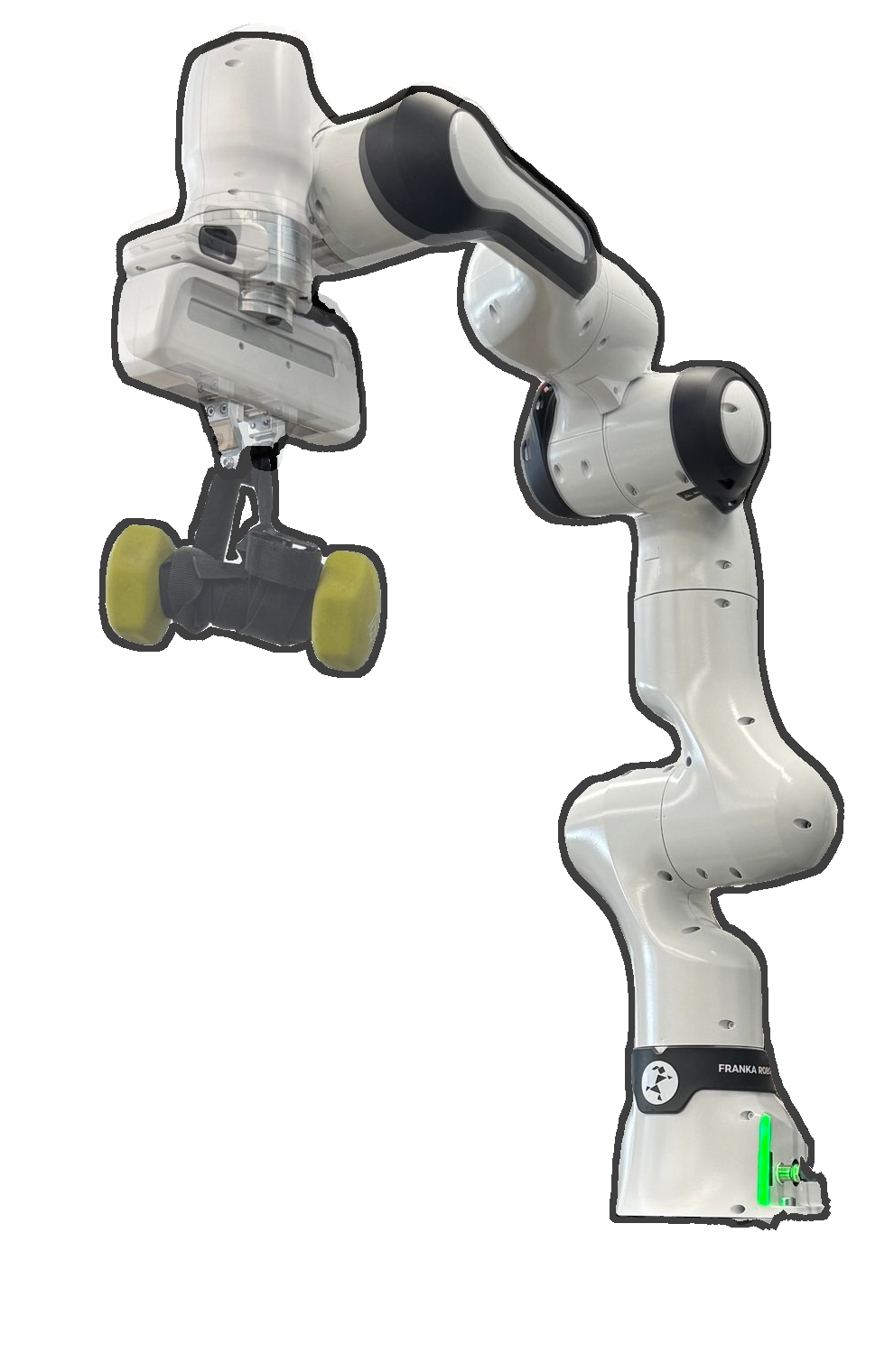}
        \caption{Stiffness w/ compliant gain}
        \label{fig:stiff-yet-compliant}
    \end{subfigure}

\caption{\textbf{Task-level impedance can be decoupled from low-level controller gains with learned policies.} A learned policy can achieve (a) \textit{compliant} behavior despite stiff low-level gains, and (b) \textit{stiff} behavior despite compliant gains.}
    \label{fig:existence-proof}
\end{figure}
In this section, we validate a central claim:
\emph{A policy's task-level compliance is predominantly determined by its learned reactions, rather than its underlying gains.}
In this section, we validate this claim empirically through two intentionally counterintuitive pairings: 

\myparagraph{\underline{Stiff} behavior with \underline{compliant} gains.}
We train a reinforcement learning policy to maintain a fixed pose under external disturbances. Although the low-level controller operates with compliant (low-gain) impedance, we induce stiff task-level behavior by randomly applying force disturbances during training and rewarding the policy for remaining close to a target pose. Specifically, we use a sharp distance-based reward, i.e., 
\begin{equation}\label{eq:dist}
    r(\mathbf{q})= 1-\tanh(\|\mathbf{q-g}\|^2 /\lambda)
\end{equation} where a small $\lambda$ strongly penalizes deviations from the goal. This encourages the policy to actively counteract disturbances and maintain its pose, resulting in stiff task-level responses despite compliant low-level control, as shown in Fig.~\ref{fig:stiff-yet-compliant}. 

\myparagraph{\underline{Compliant} behavior with \underline{stiff} gains.}
To elicit compliant task-level responses despite a stiff (high-gain) low-level controller, we soften the goal-tracking objective and discourage aggressive controller corrections: 
\begin{equation}
r(\mathbf{q})= 1-\tanh(\|\mathbf{q-g}\|^2 /\lambda_{\text{large}}) - \alpha \|\Delta a_t\|^2.
\end{equation}
A large $\lambda$ reduces the incentive to tightly regulate the goal pose, while the $\Delta a$ penalty suppresses high-frequency corrective behavior, encouraging the policy to yield smoothly under disturbances even when executed with stiff gains (Fig.~\ref{fig:compliant-yet-stiff}).
We provide quantitative support for this in Appendix \ref{sec:appendix-existence-proof}.

\section{Experiments}\label{sec:experiments}
In this section, we detail the experimental procedures we use to study the effect of low-level position controllers on learning dynamics for behavior cloning (Sec. \ref{sec:exp-bc}), reinforcement learning (Sec. \ref{sec:exp-rl}), and zero-shot sim-to-real transferability (Sec. \ref{sec:exp-sim2real}). Candidate gain setpoints used for all analysis are represented as a grid of $\mathbf{K}_p$ and $\mathbf{K}_d$ as shown in Fig. \ref{fig:gain-configs}.

\subsection{Behavior Cloning}\label{sec:exp-bc}

We investigate how controller gains affect behavior cloning performance and data collection experience through controlled dataset generation and a user study.

\myparagraph{Gain-Dependent Demonstration Dataset.} Behavior cloning distills state-action pairs $\mathcal{D}(s,a)$ into a policy $\pi(a|s)$. With position targets as actions, the controller gains $\mathbf{K} = (\mathbf{K}_p, \mathbf{K}_d)$ implicitly shape learning by altering the state transition dynamics $p(s'|s,a)$. Isolating this effect requires datasets $\mathcal{D}(s,a;\mathbf{K})$ that share a common state distribution $p(s)$ while exhibiting gain-induced action distributions $p(a;\mathbf{K})$.

Naively collecting separate demonstrations per gain setting conflates state and action variation: differing closed-loop dynamics and collection stochasticity (initial conditions, environmental randomness, demonstrator variability) cause both distributions to shift, obscuring the effect of gains on learning. 

\myparagraph{Torque-to-Position Retargeting.} We instead achieve nearly identical state trajectories across all gain settings while varying only the position target actions through Torque-to-Position Retargeting (TPR), a two-stage dataset generation procedure. First, we generate demonstration trajectories for each task at high frequency (500Hz) using \textit{torque} commands as the gain-agnostic action representation. We then apply a position-command-retargeting method to convert these torque trajectories into position targets for arbitrary $(\mathbf{K}_p, \mathbf{K}_d)$ settings:
\begin{equation}
\mathbf{q}_{\text{des}}(t) = \mathbf{q}(t) + \mathbf{K}_p^{-1}\left(\boldsymbol{\tau}(t) + \mathbf{K}_d\dot{\mathbf{q}}(t)\right),
\label{eq:tpr}
\end{equation}
where $\tau(t), q(t)$ and $\dot{q}(t)$ are the torque command, joint position, and joint velocity from the original 500Hz torque control demonstration, respectively. Finally, for each gain configuration, we replay the retargeted position commands at the desired policy frequency (50Hz) using zeroth-order holding and save only the successful rollouts, ensuring that our datasets capture the same task outcomes across different controller settings while maintaining the distinct action distributions induced by each gain profile. We conduct this entire process in simulation to ensure controlled experimental conditions. 

We quantitatively validate TPR fidelity: retargeted trajectories maintain ${\geq}90\%$ success rate and joint-position MSE ${<}10^{-3}$ across gain configurations up to $25\times$ decimation (20\,Hz); at higher decimation, success degrades slightly for contact-rich tasks where TPR's trajectory-matching assumption is less robust (Appendix~\ref{sec:appendix-tpr-validation}). TPR also extends naturally to task-space position control using operational space control (OSC)~\cite{khatib2003unified} with SE(3) end-effector pose targets; we detail this extension in Appendix~\ref{sec:appendix-tpr-taskspace}.

\captionsetup{font=small}
\captionsetup[subfigure]{font=small, skip=2pt}
\captionsetup[subfigure]{labelformat=parens}

\begin{figure*}[t]
\centering

\begin{minipage}[t]{0.42\textwidth}
  \vspace{0pt}  %

  \centering
  \subcaptionbox{Bimanual Handover\label{fig:bc-mugh}}{%
    \includegraphics[width=\linewidth]{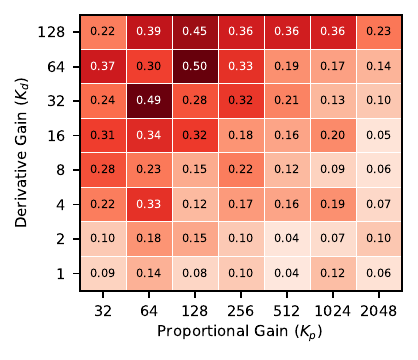}
  }
\end{minipage}
\hfill
\begin{minipage}[t]{0.56\textwidth}
  \vspace{0pt}
  \centering
  \subcaptionbox{Dishrack Unloading\label{fig:bc-dishrack-unload}}{%
    \includegraphics[width=0.31\linewidth]{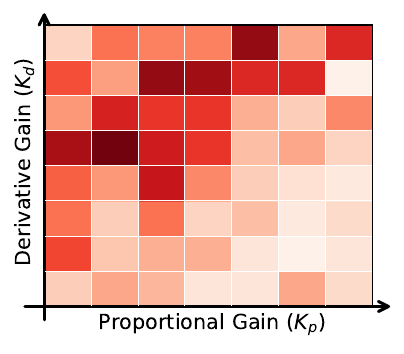}
  }\hfill
  \subcaptionbox{Dishwasher Opening\label{fig:bc-dishwasher}}{%
    \includegraphics[width=0.31\linewidth]{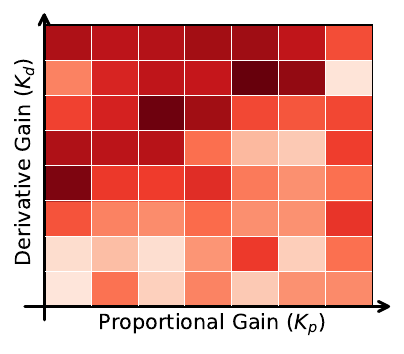}
  }\hfill
  \subcaptionbox{Dishrack Loading\label{fig:bc-dishrack-load}}{%
    \includegraphics[width=0.31\linewidth]{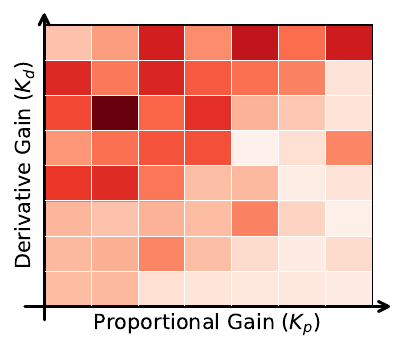}
  }\hfill
  
  \vspace{2mm}
  
  \subcaptionbox{\centering Block Stacking \\w/ Task Space Control\label{fig:bc-blockstack-osc}}{%
    \includegraphics[width=0.31\linewidth]{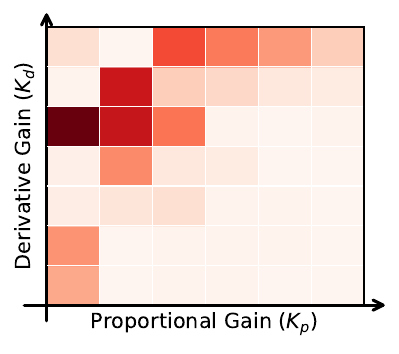}
  }
  \hfill
  \subcaptionbox{\centering Block Stacking\\/w GravComp \label{fig:bc-blockstack-gravcomp}}{%
    \includegraphics[width=0.31\linewidth]{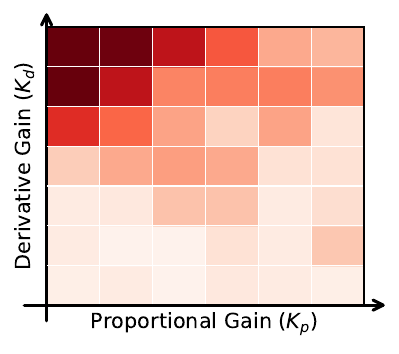}
  }\hfill
  \subcaptionbox{\centering Block Stacking \\/wo GravComp\label{fig:bc-blockstack-no-gravcomp}}{%
    \includegraphics[width=0.31\linewidth]{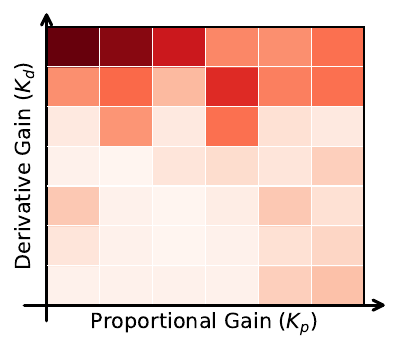}
  }
\end{minipage}

\caption{\textbf{Behavior cloning prefers \textit{compliant} and \textit{overdamped} controller gains.} Closed-loop rollout success rates across a grid of proportional ($\mathbf{K}_p$) and derivative ($\mathbf{K}_d$) gains for diverse manipulation tasks and robot embodiments. Each heatmap reports success averaged over evaluation rollouts. Across tasks, higher success rate (darker red) consistently concentrates in the compliant, overdamped regime (upper-left), while stiff or weakly damped controllers yield degraded performance.}
\label{fig:bc-results}
\end{figure*}

\myparagraph{Training Configurations.} 
We then train BC policies for each gain configuration $\mathbf{K} \in \{\mathbf{K}_1, \cdots, \mathbf{K}_n\}$, using gain-dependent demonstration datasets $\mathcal{D}(s,a(\mathbf{K}))$. Our nominal configuration uses a VAE generative model with an MLP network, observation history length 10, and action chunk size 10. We use privileged simulation states (i.e. object poses) as inputs, and absolute joint-space actions as outputs. To verify that our findings are not artifacts of this particular setup, we ablate across 
network architectures (MLP vs.\ Transformer), policy model classes (regression, VAE, and diffusion~\cite{chi2025diffusion}), temporal structure (observation history length and action chunk size), input modalities (privileged simulation states vs.\ robot state with dual-camera RGB from global and wrist-mounted views), and output representations (joint-space vs.\ task-space actions). %We verify that gain preferences are consistent across all variations; full ablation results are reported in Appendix \ref{sec:appendix-bc}. 

\begin{figure}[b]
\centering
\includegraphics[width=0.95\columnwidth]{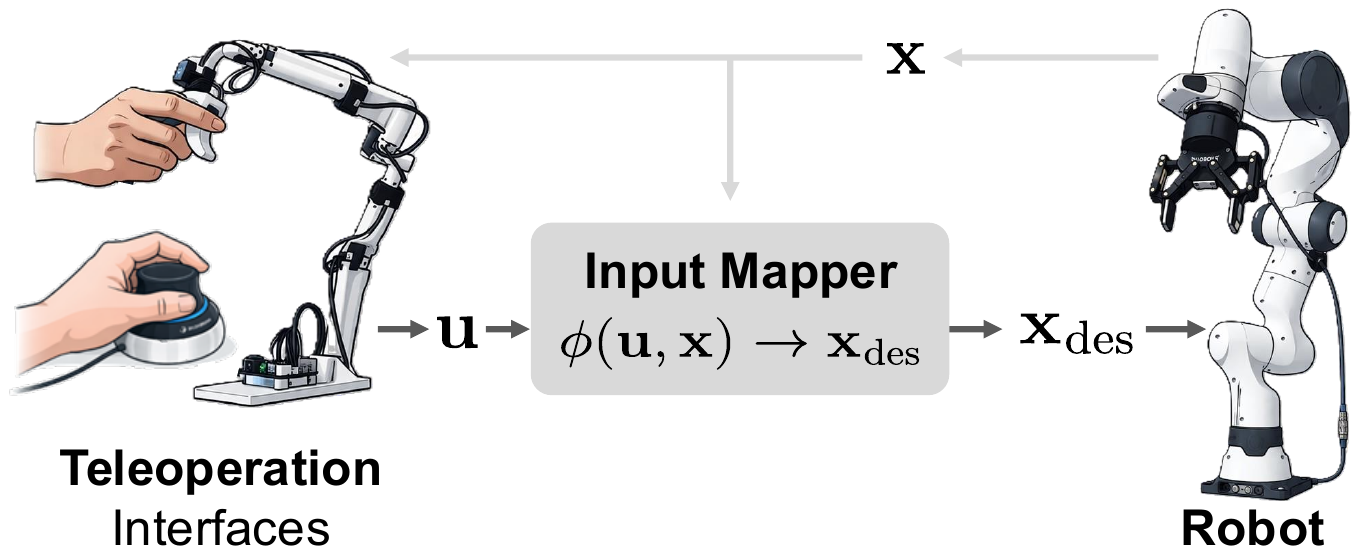}
\caption{Any teleoperation system requires a mapping $\phi$ from user inputs $\mathbf{u}$ to desired position targets $\mathbf{x}_{\text{des}}$ for the controller, which substantially shapes how the robot is perceived under different controller gains.}
\vspace{-10pt}
\label{fig:input-mapper}
\end{figure}

\myparagraph{Gain-Dependent Teleoperation User Study.} To complement our analysis on offline policy training, we conducted a user study examining how controller gains affect human teleoperation performance. We designed a contact-rich, non-prehensile box manipulation task: operators teleoperate a Franka Research 3 robot 
with a 6-DoF SpaceMouse to push a box from a randomized initial pose to a 
\begin{wrapfigure}{r}{0.48\columnwidth}
  \vspace{-10pt}
  \centering
  \includegraphics[width=0.45\columnwidth]{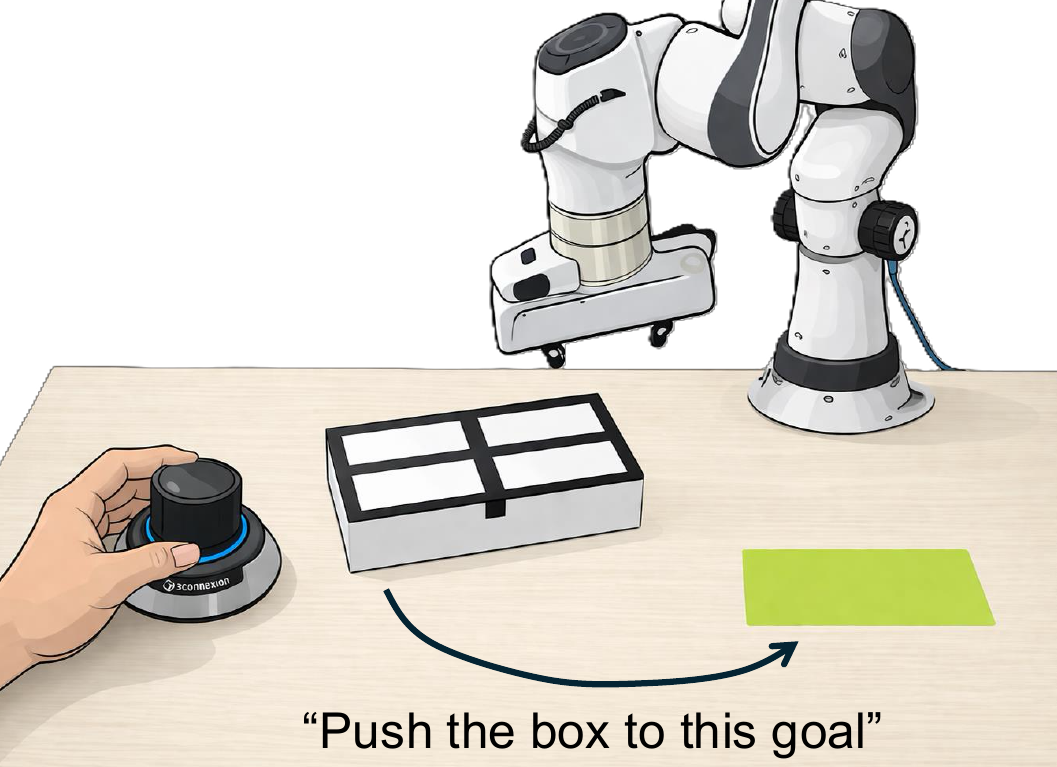}
    \setlength{\abovecaptionskip}{3pt}
  \caption{Box Pushing Task.}
  \label{fig:policy-freq-jitter}
  \vspace{-15pt}
\end{wrapfigure}
fixed goal pose. 
We chose this task because it requires both precision and sustained contact, yet remains achievable even under unintuitive gain configurations.
% \vspace{10pt}

A critical consideration is that the mapping from user input to commanded position target, $\phi(\mathbf{u}, \mathbf{x}) \rightarrow \mathbf{x}_{\text{des}}$ (Fig. \ref{fig:input-mapper}), can substantially modulate how the robot feels to the teleoperator. In our teleoperation setup, the mapping takes the form $\mathbf{x}_{\text{des}}(t) = \alpha \mathbf{u}(t) + \beta \mathbf{x}(t) + (1-\beta) \mathbf{x}_{\text{des}}(t-1)$, where $\alpha \in \mathbb{R}_0^+$ is a scaling factor and $\beta \in \{0, 1\}$ selects between the current robot pose or the previous target as the integration base. To ensure a fair comparison, each user study participant adjusts $\alpha$ and $\beta$ during a practice period before evaluation for each gain setting, yielding the gain-specific optimum
\begin{equation}\label{eq:input-mapping-optimization}
\phi^\star(\mathbf{K}) = \arg\max_\phi \mathcal{Q}(\phi; \mathbf{K}),
\end{equation}

\noindent where $\mathcal{Q}$ denotes the operator's perceived control quality. This lets operators compare each gain configuration at its best achievable experience. 

We asked 12 users to perform non-prehensile box pushing task over 1-hour sessions, collecting 1,297 total trials. Gain configurations were randomly sampled and blindly presented for each trial to control for learning effects across the session. Trials fail if the robot faults (position, velocity, or torque limit violation) or the operator pushes the box out of the workspace. For each trial, we recorded task success, completion time, and a subjective 1--5 control quality rating. The results are presented in \resultref{res:bc-user-study}.

\myparagraph{End-to-End Evaluation.} The above experiments isolate the learning and data collection effects independently. A natural concern is whether these effects compose favorably: data collected under compliant, overdamped gains may visit a different state distribution than data collected under stiff gains, potentially offsetting the learning advantage. To test this, we conduct an end-to-end experiment on a real Franka Research 3 robot. For each of the four corner gain configurations, we train a BC policy on 100 teleoperated demonstrations, collected per-gain (400 unique demonstrations). The results are presented in \resultref{res:bc-e2e}.

\subsection{Reinforcement Learning}
\label{sec:exp-rl}

In this section, we study how controller gains affect online RL, where gains shape both the transition dynamics and exploration during training. 

\myparagraph{Gain-Dependent Environment Shaping.} \label{sec:shaping} A key challenge in isolating the effect of controller gains on online RL is that performance can be highly sensitive to environment design and algorithm hyperparameters. These choices collectively determine the learning regime, a dependence we refer to as \textit{environment shaping}~\cite{park2024automatic}. To avoid conflating gain effects with suboptimal training configurations, we re-tune key hyperparameters for each gain setting using computational hyperparameter optimization~\cite{optuna_2019}\footnote{Specifically, we optimize the action space design parameters $h:=(\alpha, \beta, \gamma)$ in $\mathbf{x}_{\text{des}}(t) = \alpha \mathbf{u}(t) + \gamma \beta \mathbf{x}(t) + \gamma (1-\beta) \mathbf{x}_{\text{des}}(t-1)$, where $\mathbf{u}(t)$ is the policy output, $\alpha \in \mathbb{R}_0^+$ scales the policy output, $\gamma$ selects between absolute ($\gamma{=}0$) and relative ($\gamma{=}1$) actions, and $\beta$ determines whether relative actions are integrated on current state ($\beta{=}1$) or the previous target ($\beta{=}0$).}\textsuperscript{,}\footnote{Policies are trained using the SKRL implementation~\cite{serrano2023skrl} of PPO~\cite{schulman2017proximalpolicyoptimizationalgorithms} on tasks modified from IsaacLab~\cite{nvidia2025isaaclabgpuacceleratedsimulation}.}:
\begin{equation} \label{eq:shaping}
h^\star(\mathbf{K}) = \arg\max_h \; J\bigl(\pi^\star(h; \mathbf{K})\bigr),
\end{equation}
where $\pi^\star(h; \mathbf{K})$ denotes the converged policy under gains $\mathbf{K}$ and hyperparameters $h$. This ensures each gain configuration is evaluated at its best achievable performance, allowing us to determine whether RL can discover successful behaviors regardless of gain settings. Findings on solution existence are presented in \resultref{res:rl-existence}.

\myparagraph{Hyperparameter Optimization Landscape.} Beyond solution existence, we also investigate whether certain gain regimes make hyperparameter optimization easier. We consider gain regimes advantageous if they yield large, continuous regions of successful hyperparameters that are easily discoverable via optimization. To investigate how gain settings modulate the shape of this successful region, we record success rates across the hyperparameter landscape during 50 trials per gain setting, continuing all trials to completion even after finding a working configuration. Analysis of the optimization landscape is presented in \resultref{res:rl-shapeability}.

\myparagraph{Sample Efficiency.} In additon to the hyperparameter landscape, we also investigate whether certain gain regimes yield more efficient or stable learning once a successful hyperparameter configuration has been identified. We consider a gain regime \textit{favorable} if it enables policies to achieve high reward quickly and with low variance across random seeds. To isolate this effect, we run 5 random seeds for each hyperparameter combination that yielded ${>}95\%$ success during hyperparameter optimization, and compare the mean and standard deviation of training reward over the course of training, aggregated across all successful configurations. Analysis of sample efficiency and training stability under different gain regimes is presented in \resultref{res:rl-efficiency}.

\subsection{Sim-to-Real}
\label{sec:exp-sim2real}
Finally, we examine whether certain gain settings transfer more reliably from simulation to real hardware. We study reaching tasks with a Franka Research 3 robot to directly isolate the motor-level sim-to-real gap.

\myparagraph{System Identification.} To ensure a fair comparison across gain settings, we first perform gain-specific system identification. For each gain configuration $\mathbf{K} \in \{\mathbf{K}_1, \cdots, \mathbf{K}_n\}$, we excite the real-world robot with sinusoidal position targets $\mathbf{q}^{\text{des}}(t) = A\sin (\omega t)$ and optimize simulation parameters $\psi$ to match the resulting state trajectories, i.e., 
\begin{equation}
\psi^\star (\mathbf{K}) = \arg\min _\psi \sum_{t=0}^T \|\mathbf{x}(t; \mathbf{K}) - \bar{\mathbf{x}}(t;\psi) \|^2
\label{eq:sysid-error-definition}
\end{equation}
where $\mathbf{x}= (\mathbf{q}, \dot{\mathbf{q}})$ denotes the real robot state and $\bar{\mathbf{x}}(\cdot;\psi)$ its simulated counterpart with simulation parameters $\psi$.

This yields gain-dependent simulation environments that faithfully reproduce the closed-loop dynamics of each controller configuration. Analysis of how gain settings affect system identification quality is presented in \resultref{res:sim2real-sysid}.

\myparagraph{Gain-Dependent Sim-to-Real Transfer.} For each gain setting, we train RL policies in the corresponding calibrated simulation environment. We discover succcessful and transferable solutions by adapting Eq.~\ref{eq:shaping} as: 
\begin{equation} \label{eq:shaping-s2r}
h^\star(\mathbf{K}) = \arg\max_h \; \tilde{J}\bigl(\pi^\star(h; \mathbf{K})\bigr),
\end{equation}
where $\tilde{J}$ augments the original objective with a penalty for violating real-world robot limits. Additional details are available in Appendix~\ref{app:sim2real}.
We deploy the policies directly on the real robot without further fine-tuning; this zero-shot transfer protocol isolates the effect of gains on transferability. We also evaluate an ablation with domain randomization, where simulation parameters are perturbed within 10\% of their system-identified values during training. We additionally evaluate an ablation over control frequency, training new policies across all gain configurations at $f \in \{10, 20, 50, 100\}$\,Hz (nominal: 50\,Hz), adjusting only the zero-order hold duration $\Delta t = 1/f$.
We evaluate sim-to-real performance via \textit{trajectory error}, i.e. the mean squared error (MSE) between real and simulated state trajectories when initialized from matched configurations:
\begin{equation} \label{eq:traj_fidelity}
\mathcal{E}
=
\underbrace{\|\mathbf{q}_{\text{sim}} - \mathbf{q}_{\text{real}}\|^2}_{\text{position error}}
+
\underbrace{\|\dot{\mathbf{q}}_{\text{sim}} - \dot{\mathbf{q}}_{\text{real}}\|^2}_{\text{velocity error}}
\end{equation}
For each gain setting, we report the average trajectory error across 30 real-world rollouts. Results on how gain settings affect zero-shot transfer performance are presented in \resultref{res:sim2real-transfer}.

\section{Results}\label{sec:results}

\begin{figure}[t]
  \begin{minipage}[t]{0.24\textwidth}
    \vspace{0pt}
    \centering
    \subcaptionbox{Validation Loss\label{fig:bc-loss}}{%
      \includegraphics[width=\linewidth]{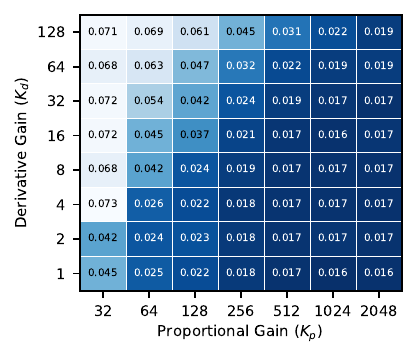}
    }
  \end{minipage}
  \hfill
  \begin{minipage}[t]{0.24\textwidth}
    \vspace{0pt}
    \centering
    \subcaptionbox{\centering Open-loop Success Rate \\ with Action Noise\label{fig:bc-noise-dampening}}{%
      \includegraphics[width=\linewidth]{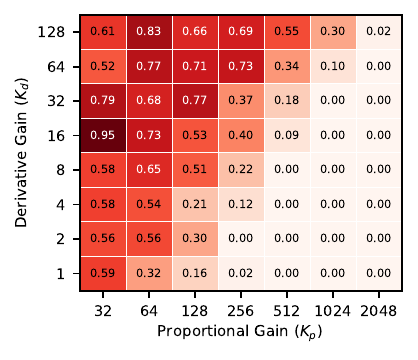}
    }
  \end{minipage}
  \hfill
  \begin{minipage}[t]{0.48\textwidth}
    \vspace{5pt}
    \centering
    \subcaptionbox{\centering Noised Open-loop Rollout \\
    for Compliant and Overdamped Gains\label{fig:bc-rollout-compliant}}{%
      \includegraphics[width=\linewidth]{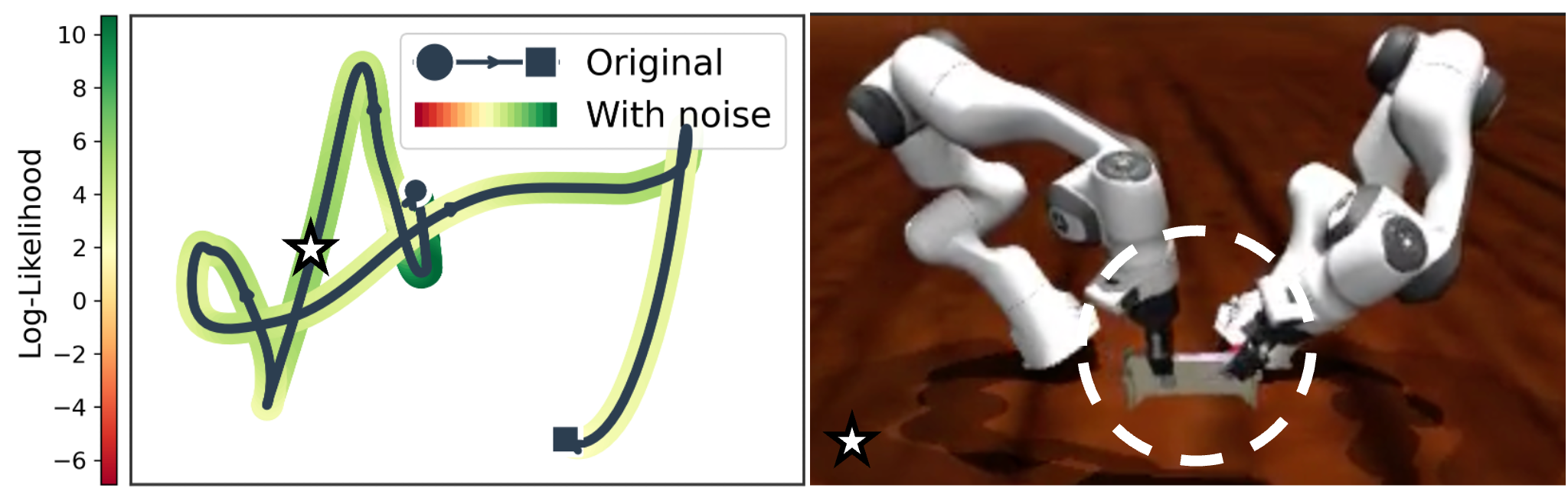}
    }
  \end{minipage}
  \hfill
  \begin{minipage}[t]{0.48\textwidth}
    \vspace{5pt}
    \centering
    \subcaptionbox{\centering Noised Open-loop Rollout \\ for Stiff and Underdamped Gains\label{fig:bc-rollout-stiff}}{%
      \includegraphics[width=\linewidth]{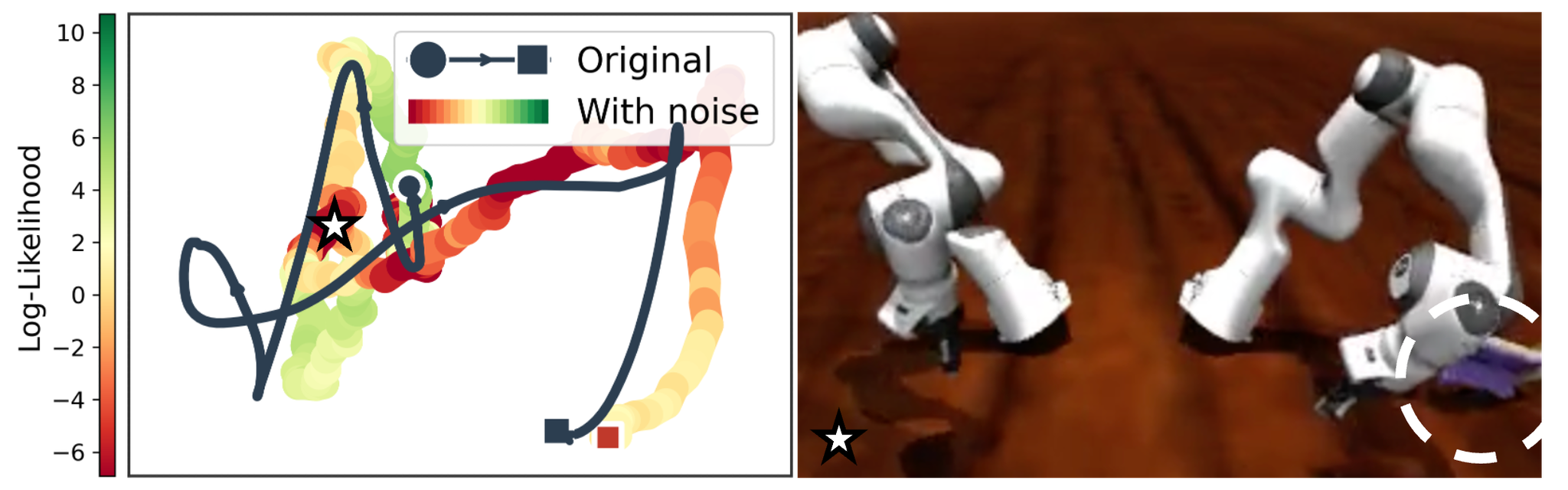}
    }
  \end{minipage}
\caption{\textbf{Compliant controllers attenuate action errors.} (a) Validation MSE loss during training: compliant gains yield higher loss, while stiff gains achieve lower loss. (b) Open-loop success rate under action noise: compliant gains maintain high success while stiff gains completely fail. (c) Compliant gains keep the perturbed trajectory close to the original, while (d) stiff gains cause large deviations that lead to task failure.}
\vspace{-10pt}
  \label{fig:bc-noise}
\end{figure}

\subsection{Behavior Cloning} \label{sec:result-bc}

Gain settings influence behavior cloning performance in two ways: (1) through the controller's response to action prediction errors during closed-loop execution, and (2) through the controller's effect on the human demonstrator during teleoperated data collection. We first study each factor in isolation, then verify their combined effect in an end-to-end real-world experiment.

\begin{myresult}[res:bc-learnability]{Effect on Learning}
Under state-matched demonstrations (via TPR), behavior cloning performs best with \textit{compliant} and \textit{overdamped} gains (i.e., top left region of Fig. \ref{fig:gain-configs}).
\end{myresult}

Using TPR (Eq. \ref{eq:tpr}) to hold the state distribution constant across gain settings and vary only the action distribution, we consistently observe that \textit{compliant} and \textit{overdamped} gain setpoints yield significantly better closed-loop policy performance. Figure \ref{fig:bc-results} illustrates this trend over a broad grid of controller gains and manipulation tasks. 
% The same trend persists across diverse training configurations, including (a) state-based versus image-based policies, (b) action-chunked versus non–action-chunked policies, and (c) policies trained with or without state histories (Appendix~\ref{sec:appendix-bc}), as well as (d) task-space versus joint-space control (Fig.~\ref{fig:bc-blockstack-osc}), and (e) with or without gravity compensation (Fig.~\ref{fig:bc-blockstack-no-gravcomp}). 
Additional results and visualizations are available on our \href{https://younghyopark.me/tune-to-learn}{project website}. 

To verify that the observed advantage of compliant, overdamped gains is statistically significant, we conduct formal hypothesis testing across all tasks. First, we fit a binomial logistic regression on $\log_2 \mathbf{K}_\text{p}$ and $\log_2 \mathbf{K}_\text{d}$ using $N{=}100$ rollouts per gain cell, confirming that lower $\mathbf{K}_\text{p}$ and higher $\mathbf{K}_\text{d}$ are significant predictors of success. We then apply Bonferroni-corrected one-sided Barnard's exact tests ($\alpha \approx 0.0083$, correcting for 6 tasks) under the null hypothesis
\begin{equation}
\mathcal{H}_0 \colon P(\text{success} \mid \mathcal{G}^{\text{CO}}) \leq P(\text{success} \mid \mathcal{G} \setminus \mathcal{G}^{\text{CO}})
\end{equation}
where $\mathcal{G}^{\text{CO}}$ denotes the compliant-overdamped gain region. $\mathcal{H}_0$ is rejected in all six tasks with $p \ll \alpha$ (Table~\ref{tab:statistical_analysis} in Appendix~\ref{sec:appendix-bc-stats}).

\myparagraph{Higher MSE Loss, Better Performance.} Policies trained under compliant and overdamped gains exhibit higher training and validation MSE loss than those trained under stiff gains (Fig.~\ref{fig:bc-loss}). Yet these same policies achieve higher closed-loop success rates. Lower imitation loss does not translate to better policy performance in our experiments.

\myparagraph{Robustness to Action Noise.}
To further characterize this effect, we execute identical open-loop action sequences across all gain configurations while injecting random action noise at each timestep, with maximum noise magnitudes matched to the average validation loss observed during training (Fig.~\ref{fig:bc-noise-dampening}). Under identical perturbations, compliant and overdamped gains maintain higher success rates than stiff gains.

\myparagraph{Discussion.} Together, these observations are consistent with the error-attenuation properties of compliant and overdamped controllers: low stiffness reduces the force produced by a given action error, while high damping dissipates perturbations faster, jointly limiting the resulting state deviation (formalized in Appendix~\ref{app:lipschitz}). This explains the apparent paradox --- compliant gains produce action targets that are harder to fit (higher MSE), but the controller attenuates the resulting errors during execution, yielding better closed-loop performance.

\begin{myresult}[res:bc-user-study]{Effect on Teleoperation}
\textit{Compliant} and \textit{overdamped} gain settings yield comparable teleoperated data collection efficiency compared to other gain regimes, achieving similar yield and operator preference.
\end{myresult}

\begin{figure}[t]
    \centering

    \begin{subfigure}[t]{0.32\columnwidth}
        \centering
        \includegraphics[width=\linewidth]{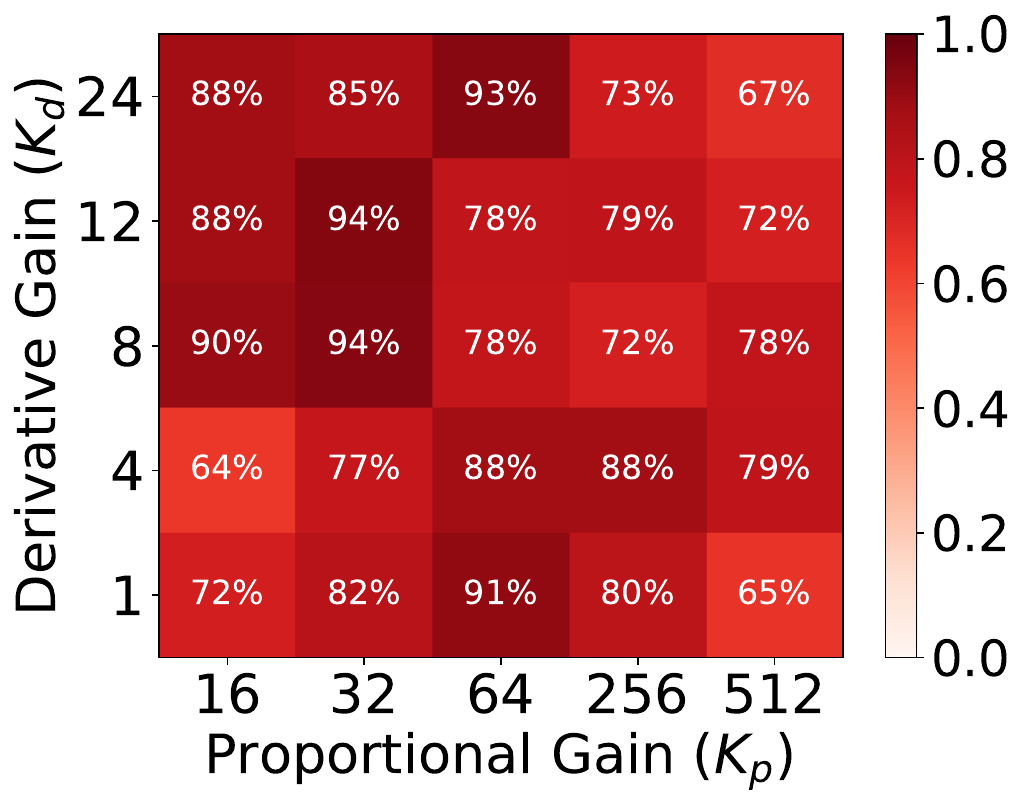}
        \caption{Success Rate}
        \label{fig:user-study-success-rate}
    \end{subfigure}\hfill
    \begin{subfigure}[t]{0.32\columnwidth}
        \centering
        \includegraphics[width=\linewidth]{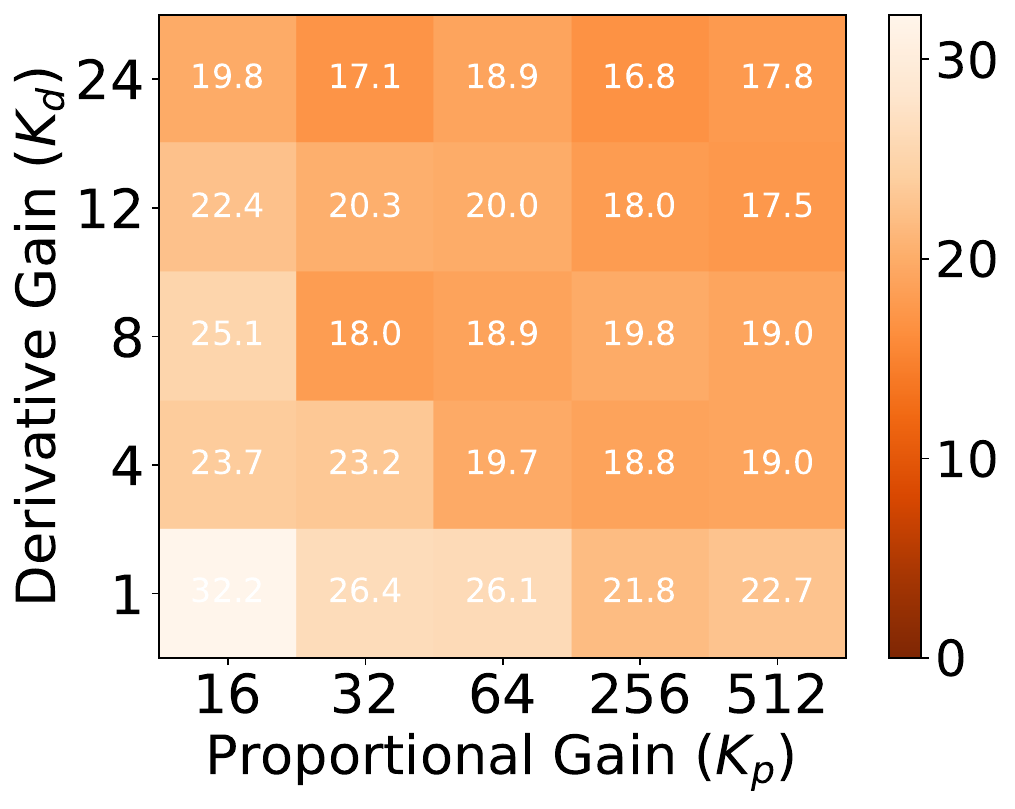}
        \caption{Completion Time}
        \label{fig:user-study-completion-time}
    \end{subfigure}\hfill
    \begin{subfigure}[t]{0.32\columnwidth}
        \centering
        \includegraphics[width=\linewidth]{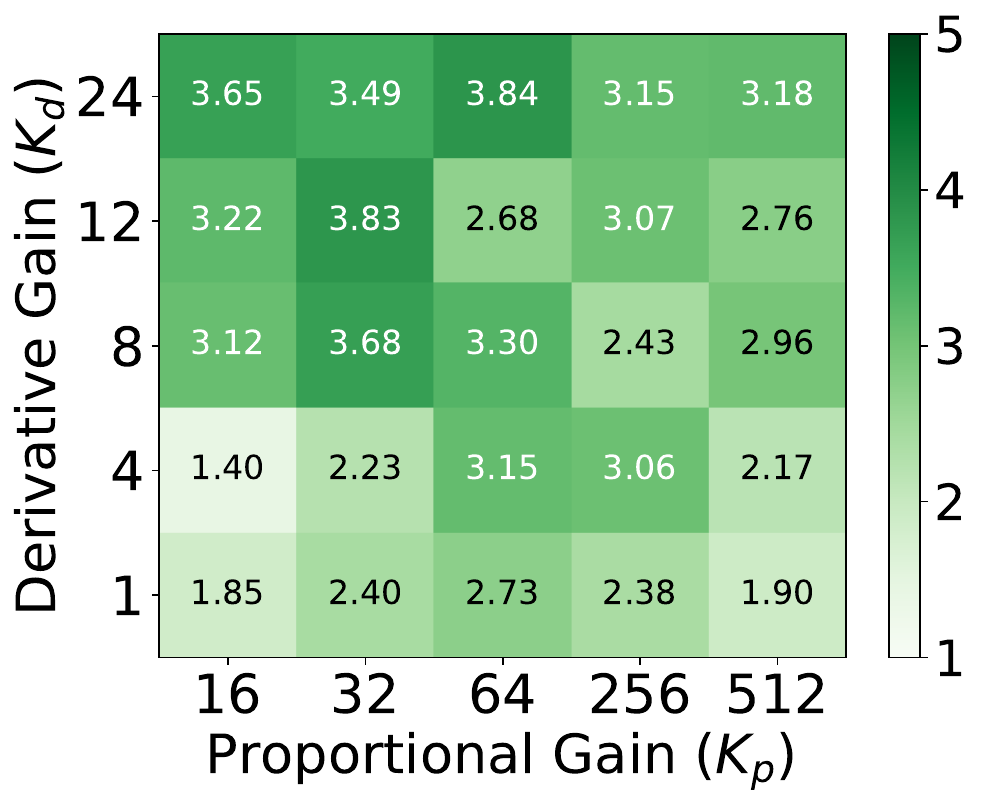}
        \caption{User Rating}
        \label{fig:user-study-rating}
    \end{subfigure}

\caption{\textbf{Teleoperation performance under different gain regimes.}
With optimized input mapping $\phi^\star(\mathbf{K})$ (Eq. \ref{eq:input-mapping-optimization}), compliant and overdamped controllers (grid top-left) achieve similar or better success rates, user ratings, and shorter completion time to stiffer settings.}
\vspace{-10pt}
\label{fig:user-study}

\end{figure}

One might expect that the \textit{compliant} and \textit{overdamped} gains (favorable for policy learning) would hinder teleoperation, as sluggish robot response could frustrate operators and reduce collection efficiency. 
Indeed, as we show in Sec.~\ref{sec:gain-settings-datasets}, stiff controllers appear to be the implicit default across major robot learning datasets. 
Our user study reveals that this tradeoff is not as stark as it appears (Fig.~\ref{fig:user-study}). 
When each gain configuration is evaluated with its own optimized input mapping $\phi^\star(\mathbf{K})$ (Eq. \ref{eq:input-mapping-optimization}), compliant and overdamped settings achieve comparable or better success rates and receive high subjective ratings; appropriate tuning of the mapping function, particularly the scaling factor $\alpha$, compensates for reduced responsiveness, allowing operators to command sufficiently responsive movements when needed. These results indicate that adopting compliant, overdamped gains for imitation learning does not impose a penalty on data collection. Full experimental details are in Appendix~\ref{sec:appendix-user-study}.

\begin{figure*}[t]
  \centering

  \setlength{\tabcolsep}{2pt} %

  \begin{tabular}{c|c}

    \subcaptionbox{\textit{FR3 Lift-Cube} Hyperparameter Sensitivity\label{fig:three-b}}{%
      \includegraphics[width=0.47\linewidth]{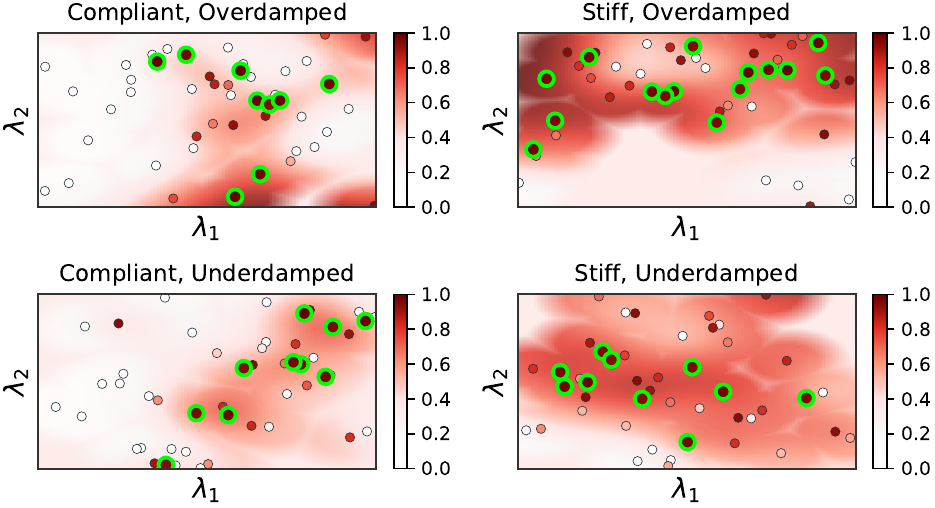}
    }
    &
    \subcaptionbox{\textit{G1 Track-Velocity} Hyperparameter Sensitivity\label{fig:three-c}}{%
      \includegraphics[width=0.47\linewidth]{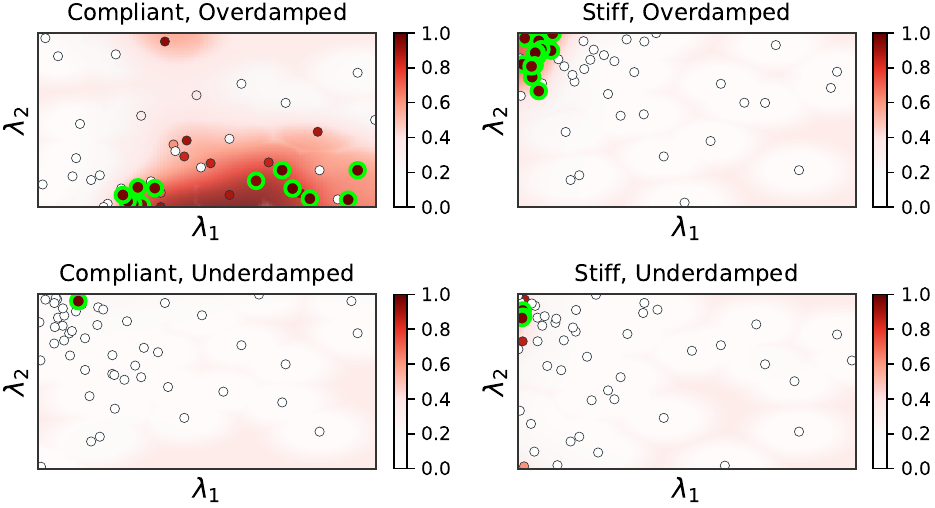}
    }
    \\[4pt]

    \subcaptionbox{\textit{FR3 Lift-Cube} Training Curves\label{fig:convergence-lift}}{%
      \includegraphics[width=0.4\linewidth]{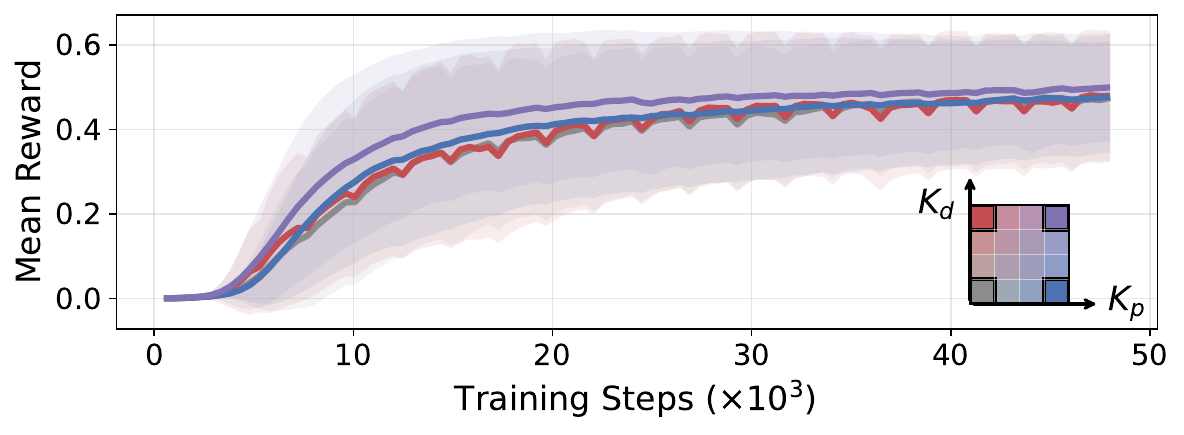}
    }
    &
    \subcaptionbox{\textit{G1 Track-Velocity} Training Curves\label{fig:convergence-g1}}{%
      \includegraphics[width=0.4\linewidth]{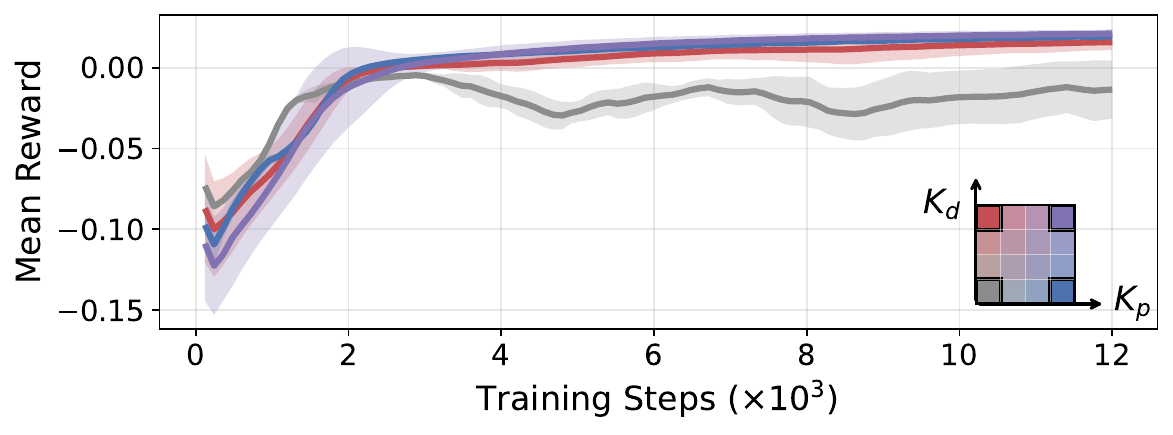}
    }

  \end{tabular}

  \caption{\textbf{RL training across gain regimes.}
  (a--b) Success rate across the hyperparameter landscape varies among gain settings and tasks; policies with 95\%+ success rate (green circles) are found across all conditions.
  (c--d) Sample efficiency and training stability of PPO is comparable across gain regimes for both tasks.
  }
  \label{fig:hyperparam_landscape}

\end{figure*}

\begin{myresult}[res:bc-e2e]{Composed Effect}
When data collection and policy learning are performed end-to-end under each gain setting, the \textit{compliant} and \textit{overdamped} regime still yields the best policy performance.
\end{myresult}

\begin{wrapfigure}{r}{0.48\columnwidth}
  \vspace{-12pt}
  \centering
  \includegraphics[width=0.46\columnwidth]{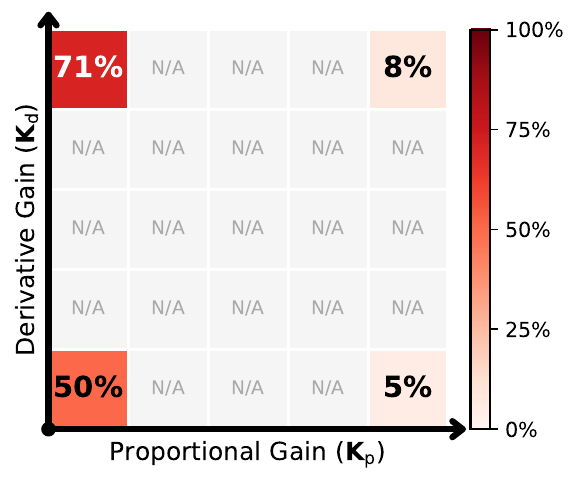}
  \setlength{\abovecaptionskip}{2pt}
  \caption{End-to-end BC pipeline still favors compliant and overdamped gain regime.}
  \label{fig:bc-e2e-heatmap}
  \vspace{-5pt}
\end{wrapfigure}
When data collection and policy training are performed end-to-end under each gain setting (see Section~\ref{sec:exp-bc}), the compliant, overdamped regime achieves the highest success rate (Fig.~\ref{fig:bc-e2e-heatmap}), despite collecting data under its own---potentially different---state distribution. This confirms that the learning advantage of compliant gains is not offset by distributional differences in teleoperated data, and that the two effects reinforce rather than cancel each other.

\subsection{Reinforcement Learning}\label{sec:results-rl}

\begin{myresult}[res:rl-existence]{RL Solution Existence}
Reinforcement learning \textit{can} discover behaviors regardless of gain setpoints. 
\end{myresult}

We find that all gain regimes spanning over two orders of magnitude in both $\mathbf{K}_p$ and $\mathbf{K}_d$ \textit{can} yield working controllers given appropriate environment shaping (Table~\ref{tab:rl_existence}). Unlike behavior cloning, on-policy RL trains on data generated by its own exploration, which may allow it to learn compensatory behaviors rather than relying on the controller's error attenuation properties.

\begin{table}[b]
\centering
\caption{\textbf{RL solution existence across gain regimes.} For each task, we verify that at least one successful policy can be discovered in every gain configuration  given appropriate environment shaping. A checkmark indicates that gain regimes corresponding to four corner extremes in Fig. \ref{fig:gain-configs} yield working controllers (99\%+ success rate). Videos or live policy rollouts of discovered behaviors for each gain settings are available on our \href{https://younghyopark.me/tune-to-learn}{project website}.}
\label{tab:rl_existence}
\begin{tabular}{l c}
\toprule
\textbf{Task / Platform} & \textbf{Existence Proof} \\
\midrule
FR3 Joint-Reach            & \checkmark \\
FR3 EE-Reach               & \checkmark \\
FR3 Lift-Cube              & \checkmark \\
FR3 Open-Drawer            & \checkmark \\
G1 Track-Velocity  & \checkmark \\
Allegro In-Hand Cube Manipulation  & \checkmark \\
FR3 Nonprehensile Cube Reorientation & \checkmark \\
\bottomrule
\end{tabular}
\end{table}

\begin{myresult}[res:rl-shapeability]{RL Hyperparameter Sensitivity}
The gain setting modulates the hyperparameter optimization landscape, but no regime is consistently easier to optimize.
\end{myresult}

Given that successful policies exist across gain regimes, we ask whether any regime is easier to discover a working policy via hyperparameter optimization. Fig.~\ref{fig:hyperparam_landscape}(a--b) visualizes the optimization landscape across environment parameters sampled during hyperparameter optimization for two tasks. While the gain setting clearly modulates the optimization landscape, we do not observe a consistent trend across tasks. This could reflect genuine task-dependence, or it could be an artifact of optimizing only a subset of environment parameters, which represents a low-dimensional slice of a larger optimization space. Resolving this would require larger experiments such as opening up more hyperparameters or leveraging automated reward shaping, which we leave to future work.

\begin{myresult}[res:rl-efficiency]{RL Sample Efficiency}
Sample efficiency and training stability are comparable across gain regimes.
\end{myresult}

We next ask whether any gain regime yields more efficient or stable learning once 
a working configuration has been identified. We find that training dynamics are comparable across gain regimes for both tasks (Fig.~\ref{fig:hyperparam_landscape}(c--d)). The exception is the compliant, underdamped regime on \textit{G1 Track-Velocity}, where only one 
viable configuration was found; the resulting curve is marginally successful and 
less smooth, though low seed variance indicates consistent rather than unstable learning. 
Together, these results suggest that the choice of gain regime does not meaningfully compromise the efficiency or stability of RL training.
\begin{figure}[t]
  \begin{minipage}[t]{0.24\textwidth}
    \vspace{0pt}
    \centering
    \subcaptionbox{SysID Modeling Error\label{fig:sysid_fit}}{%
      \includegraphics[width=\linewidth]{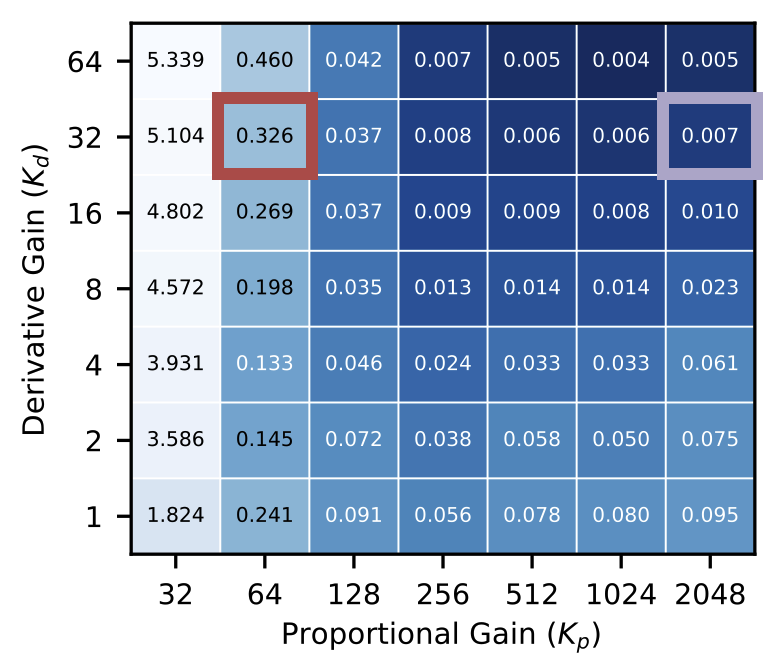}
    }
  \end{minipage}
  \hfill
  \begin{minipage}[t]{0.235\textwidth}
    \vspace{0pt}
    \centering
    \subcaptionbox{\centering Real vs. Sim Rollouts \label{fig:sim2real-ood}}{%
      \includegraphics[width=\linewidth]{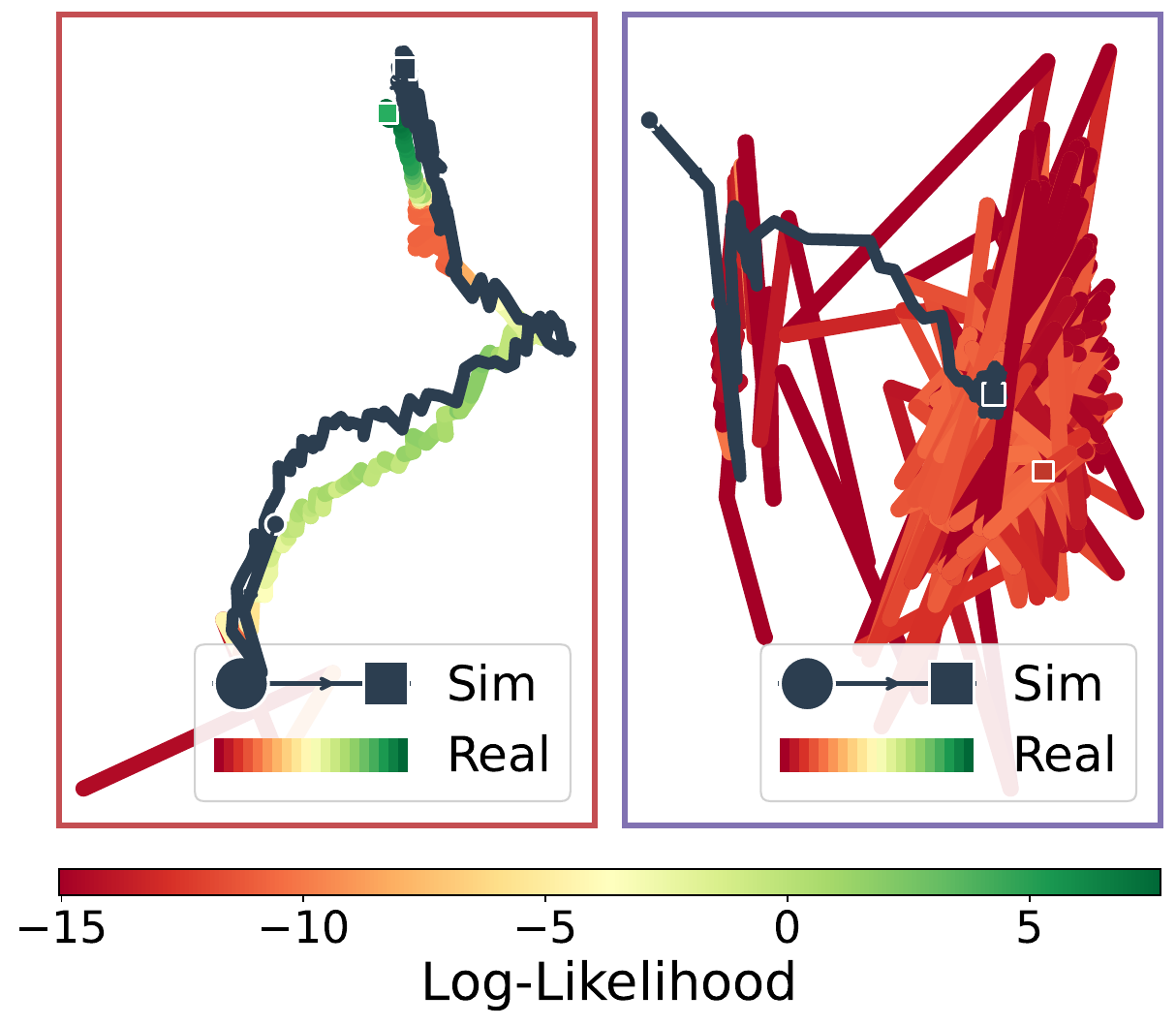}
    }
  \end{minipage}
\caption{\textbf{Stiff and overdamped gain settings yield lower SysID modeling errors, but exhibit larger closed-loop Sim2Real errors}. Policy observations during closed-loop rollout evolve similarly between sim and real (b-left) for compliant, overdamped gains, but very dissimilarly (b-right) for stiff, overdamped gains.
}
% \vspace{-5pt}
  \label{fig:sysid_ood}
\end{figure}

% \vspace{5pt}

\subsection{Sim-to-Real}\label{sec:results-sim2real}

\begin{myresult}[res:sim2real-sysid]{System Identification}
System identification achieves the lowest modeling error under \textit{stiff} and \textit{overdamped} gains (i.e., upper right region of Fig. \ref{fig:gain-configs}). 
\end{myresult}

The MSE between simulated and real response curves after system identification, i.e., $\mathcal{S}^\star(\mathbf{K}) = \min_\psi \mathcal{S}(\mathbf{K}, \psi)$ (Eq. \ref{eq:sysid-error-definition}), is over an order of magnitude lower for the stiff, overdamped regime compared to other gain settings (Fig.~\ref{fig:sysid_fit}). The highest system identification errors appear in the compliant, and particularly compliant and overdamped, gain regime.

\begin{myresult}[res:sim2real-transfer]{Sim2Real Transferability}
Sim2Real transferability, however, is lower with \textit{stiff} and \textit{overdamped} gain setpoints, with its main failure mode being high frequency oscillation. 
\end{myresult}

\begin{figure}[t]
  \centering
  
  \begin{minipage}[t]{0.33\textwidth}
    \vspace{0pt}
    \centering
    \subcaptionbox{\centering Joint-Reach Sim2Real trajectory error\\without Domain Randomization\label{fig:s2r-joint}}{%
      \includegraphics[width=\linewidth]{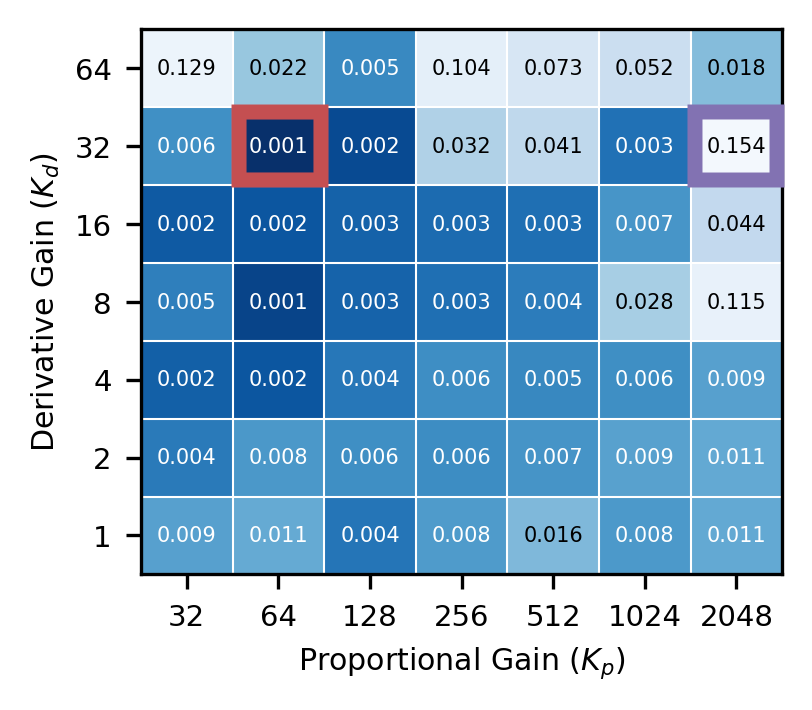}
    }
  \end{minipage}%
  \hfill%
  \begin{minipage}[t]{0.15\textwidth}
    \vspace{2pt}
    \centering
    \subcaptionbox{\centering Joint-reach \\with DR\label{fig:s2r-joint-nodr}}{%
      \includegraphics[width=\linewidth]{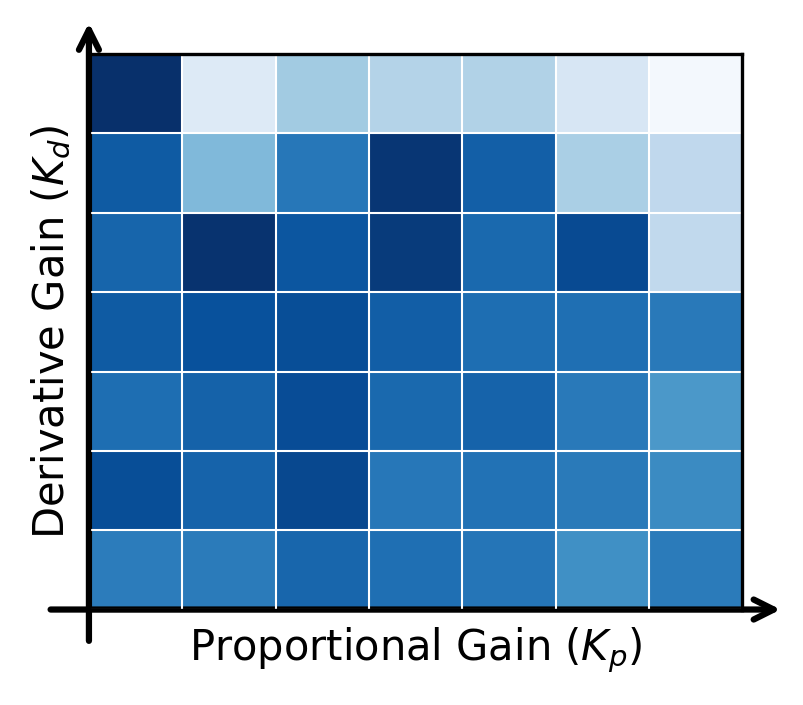}
    }
    \vspace{1mm}
    \subcaptionbox{EE no DR\label{fig:s2r-ee}}{%
      \includegraphics[width=\linewidth]{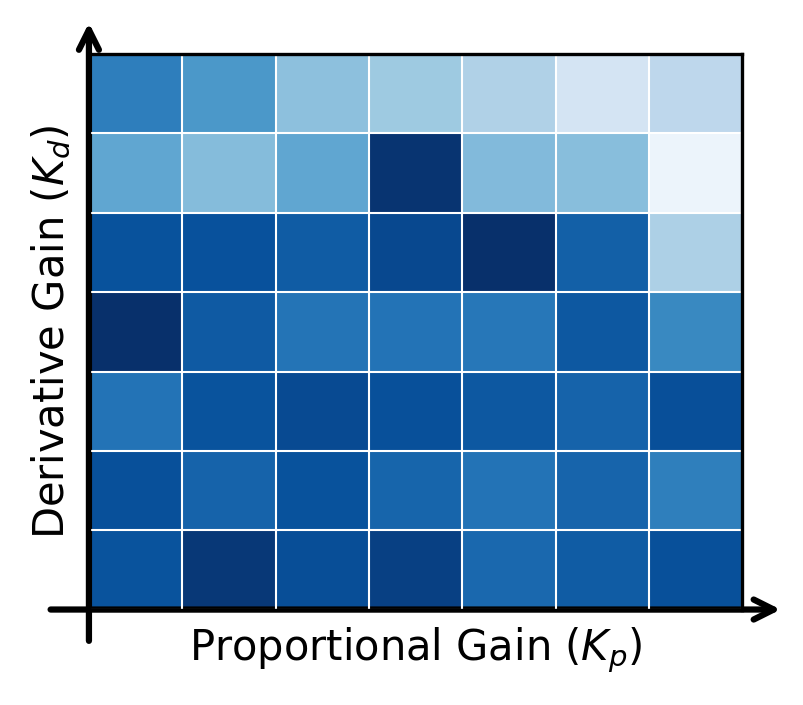}
    }
  \end{minipage}
  
  \vspace{3mm}
  
  \begin{minipage}[t]{0.48\textwidth}
    \centering
    \subcaptionbox{Real vs. sim wrist joint trajectories. \label{fig:s2r-success}}{%
      \includegraphics[width=\linewidth]{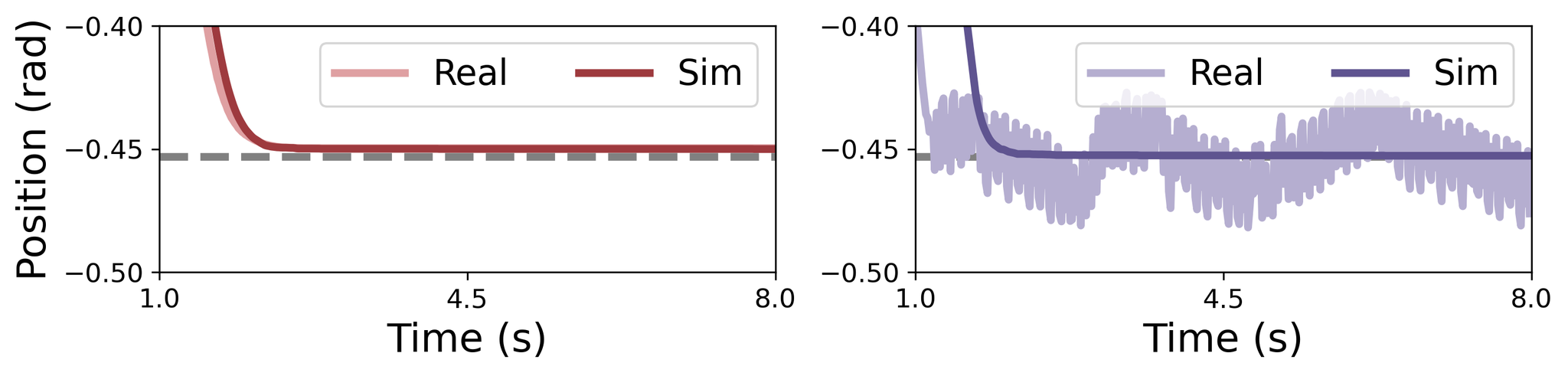}
    }
  \end{minipage}
  
  \caption{\textbf{Stiff and overdamped gain settings reduce sim2real transferability.} The Sim2Real trajectory error (Eq.\ref{eq:traj_fidelity}) is consistently larger (light blue) in the stiff and overdamped regime (a-c). The primary Sim2Real failure mode is high-frequency oscillation (d). }
  \label{fig:sim2real_trajectory}
  \vspace{-5pt}
\end{figure}

\myparagraph{Trajectory Error.}
Stiff and overdamped gain settings exhibit the largest sim-to-real trajectory error (Fig.~\ref{fig:sim2real_trajectory}). The dominant failure mode is high-frequency oscillation, which persists even with domain randomization (Fig.~\ref{fig:s2r-joint-nodr}). Notably, the low-level controller itself is stable under smooth commands; the oscillation only appears during closed-loop policy execution. When we compare the distribution of policy observations between sim and real (Fig.~\ref{fig:sim2real-ood}), stiff, overdamped gains produce real-world observations that are highly unlikely under the simulation distribution, whereas compliant, overdamped gains yield closely overlapping distributions.

\myparagraph{Statistical Significance.}
To verify that the stiff-overdamped gain region $\mathcal{G}^{\text{SO}}$ produces significantly larger sim-to-real error, we fit OLS regression on log-transformed trajectory error with $\log_2 \mathbf{K}_\text{p}$ and $\log_2 \mathbf{K}_\text{d}$ as predictors, confirming that both higher $\mathbf{K}_\text{p}$ and higher $\mathbf{K}_\text{d}$ are significant predictors of increased error. We then apply Bonferroni-corrected one-sided Mann-Whitney U tests ($\alpha \approx 0.017$, correcting for 3 conditions) under the null hypothesis
\begin{equation}
\mathcal{H}_0 \colon \varepsilon(\mathcal{G}^{\text{SO}}) \leq \varepsilon(\mathcal{G} \setminus \mathcal{G}^{\text{SO}})
\end{equation}
where $\varepsilon$ denotes the sim-to-real trajectory error. $\mathcal{H}_0$ is rejected in all three conditions with $p \ll \alpha$ (Table~\ref{tab:sim2real_statistical_analysis} in Appendix~\ref{sec:appendix-sim2real-stats}).

% \vspace{10}

\begin{myresult}[res:sim2real-policy-freq]{Effect of Policy Frequency}
Lowering the policy frequency (increasing the zero-order-hold duration $\Delta t$ per policy action) reduces the prevalence of high-frequency oscillation during sim-to-real transfer.
\end{myresult}

\begin{wrapfigure}{r}{0.48\columnwidth}
  \vspace{-10pt}
  \centering
  \includegraphics[width=0.48\columnwidth]{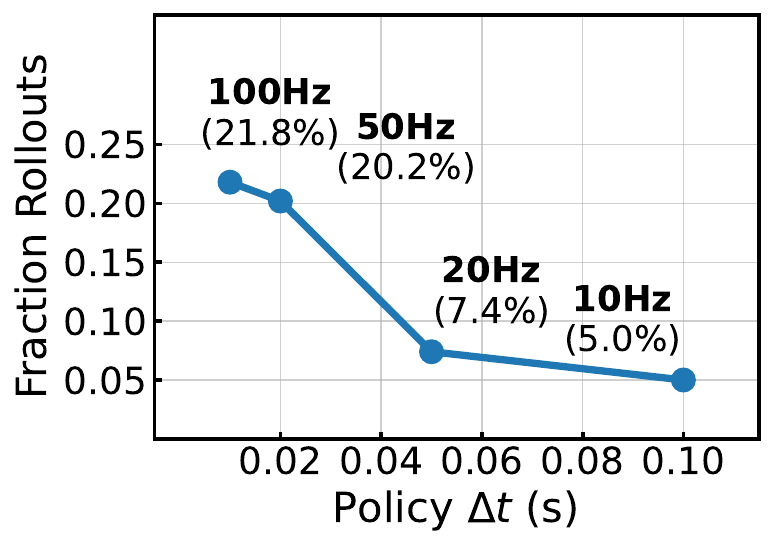}
    \caption{Jitter Failures vs.\ $\Delta t$.}
    \label{fig:policy-freq-jitter}
      \vspace{-10pt}
\end{wrapfigure}

We detect jitter failures by computing the maximum per-joint standard deviation of joint velocity during the final 2 seconds of each rollout, flagging trajectories exceeding a threshold of 0.04 rad/s; this metric reliably separates the two modes, as settled rollouts have a median velocity standard deviation of 0.001 rad/s while jittering rollouts have a median of 0.675 rad/s. As shown in Fig.~\ref{fig:policy-freq-jitter}, the fraction of jitter failures across the full gain grid drops from 21.8\% at 100\,Hz to 5.0\% at 10\,Hz, with the sharpest reduction occurring between 50\,Hz and 20\,Hz. 

\myparagraph{Discussion.}
The lower SysID modeling error under stiff, overdamped gains is consistent with these dynamics filtering out nonlinearities that are difficult for idealized simulated actuators to capture: high damping suppresses high-frequency effects such as joint flexibility and transmission dynamics, while high stiffness reduces sensitivity to steady-state errors from imperfect gravity compensation or stiction.

However, the inverse relationship between SysID accuracy and closed-loop transfer quality suggests that gain settings should be evaluated in the context of the full policy loop, not in isolation. We hypothesize that stiff, overdamped controllers amplify small modeling errors because they respond to position and velocity deviations with high torques. When the policy reacts to noise or unmodeled dynamics, the controller aggressively tracks these commands, pushing the system further from states encountered in simulation. Naively choosing the gains that minimize modeling error can therefore paradoxically increase the closed-loop sim-to-real gap.

Lower policy frequency reduces oscillation in a manner consistent with 
\cite{gangapurwala2023learning}: with more time between commands, joints settle 
before the next action is issued, reducing the opportunity for the policy to 
react to transient out-of-distribution observations and amplify them into 
oscillation. This provides a simple mitigation strategy for the failure mode 
identified in Result~\ref{res:sim2real-transfer}, at the cost of reduced 
temporal resolution.

\section{Conclusion and Remarks}

We have presented a systematic study of how position controller gains shape learning dynamics across three paradigms of modern robot learning. Our findings reveal that gains function not as behavioral parameters, but as an inductive bias that modulates the learning interface between policy and environment. Behavior cloning favors compliant, overdamped regimes; reinforcement learning adapts to any gain setting given compatible hyperparameters; and sim-to-real transfer suffers with stiff, overdamped configurations. These results provide both conceptual clarity and practical guidance for a widely used yet underexplored design decision.

Our framework also raises questions for adjacent areas. Modern humanoid robots increasingly use RL-trained whole-body tracking policies as low-level controllers, analogous to the PD controllers studied here. Yet how their compliance shapes high-level policy learning remains unexplored.
Similarly, paradigms that learn manipulation skills from human videos \cite{qiu2025humanoid, grauman2022ego4d} or wearable devices~\cite{chi2024universal} typically treat observed next timestep state as the action label, implicitly assuming perfect target tracking, which our results suggest may be suboptimal for imitation learning. Whether these gain-dependent trends generalize to such cross-embodiment or whole-body control settings remains an open question, and we hope our findings offer a useful lens for investigating these directions.

\section*{Acknowledgements}
We thank the members of the Improbable AI lab for the helpful discussions and feedback on the paper. This research was financially partially supported by the Ministry of Trade, Industry, and Energy (MOTIE), Korea, under the ``Global Industrial Technology Cooperation Center program'' supervised by the Korea Institute for Advancement of Technology (KIAT) (Grant No. P0028435).
This work was also partly supported by the Sony Research Award.

\bibliography{references}

\clearpage 
\appendix

\section{Analytical Characterization of Gain-Dependent Error Sensitivity}
\label{app:lipschitz}

\subsection{Analytical Proof of Gain-Dependent Error Sensitivity}
\label{app:lipschitz}

We formalize the empirical observation that compliant and overdamped controller gains attenuate action prediction errors during behavior cloning. We analyze a simplified 1-DOF system and prove that the steady-state position error variance under stochastic action noise is proportional to $\mathbf{K}_p / \mathbf{K}_d$.

\myparagraph{Setup.} Consider a 1-DOF point mass $m$ controlled by a PD controller. The continuous-time dynamics under action (position target) $a(t) = q_{\mathrm{des}}(t)$ are:
\begin{equation}
    m\ddot{q} = \mathbf{K}_p(a - q) - \mathbf{K}_d \dot{q}
    \label{eq:dynamics}
\end{equation}
The controller gains are $\mathbf{K} = (\mathbf{K}_p, \mathbf{K}_d)$ with $\mathbf{K}_p, \mathbf{K}_d > 0$. We define the natural frequency and damping ratio:
\begin{equation}
    \omega_n = \sqrt{\frac{\mathbf{K}_p}{m}}, \qquad
    \zeta = \frac{\mathbf{K}_d}{2\sqrt{m\,\mathbf{K}_p}}
    \label{eq:params}
\end{equation}

\begin{tcolorbox}[colback=gray!5, colframe=gray!50, boxrule=0.4pt, arc=0pt, breakable, left=4pt, right=4pt, top=4pt, bottom=4pt]
\begin{theorem}[Gain-Dependent Error Variance]
\label{thm:variance}
Consider the system~\eqref{eq:dynamics} with $\mathbf{K}_p, \mathbf{K}_d > 0$. Suppose the policy $\pi$ produces i.i.d.\ action errors $\delta a(t)$ with variance $\sigma^2$ around the expert action. Then the steady-state position error variance is:
\begin{equation}
    \mathrm{Var}[\delta q] \;=\; \frac{\sigma^2 \, \mathbf{K}_p}{2\,\mathbf{K}_d}
    \label{eq:variance_result}
\end{equation}
In particular, the variance is independent of the mass $m$, and is minimized when $\mathbf{K}_p / \mathbf{K}_d$ is small.
\end{theorem}
\end{tcolorbox}

\begin{proof}
Suppose the expert action is $a^*(t)$ and the learned policy predicts $\hat{a}(t) = a^*(t) + \delta a(t)$. Subtracting the nominal trajectory from the perturbed one, the position error $\delta q(t) = q_{\mathrm{noised}}(t) - q_{\mathrm{clean}}(t)$ satisfies:
\begin{equation}
    m\,\delta\ddot{q} + \mathbf{K}_d\,\delta\dot{q} + \mathbf{K}_p\,\delta q = \mathbf{K}_p\,\delta a(t)
    \label{eq:perturbation_ode}
\end{equation}
This is a damped harmonic oscillator driven by the action error. Dividing by $m$ and substituting~\eqref{eq:params}:
\begin{equation}
    \delta\ddot{q} + 2\zeta\omega_n\,\delta\dot{q} + \omega_n^2\,\delta q = \omega_n^2\,\delta a(t)
    \label{eq:standard_form}
\end{equation}
Since the policy makes independent errors at each timestep, we model $\delta a(t)$ as white noise---the continuous-time analog of uncorrelated random inputs---with variance parameter $\sigma^2$. Eq.~\eqref{eq:standard_form} is a standard second-order oscillator driven by the effective noise input $n(t) = \omega_n^2\,\delta a(t)$, whose variance parameter is $\sigma_n^2 = \omega_n^4\,\sigma^2$. We apply the following classical result:

\begin{tcolorbox}[colback=gray!5, colframe=gray!50, boxrule=0.4pt, arc=0pt, breakable, left=4pt, right=4pt, top=4pt, bottom=4pt]
\begin{theorem}[Mean-Square Response of a Second-Order System~{\cite{crandall2014random}}]
\label{thm:crandall}
Consider the oscillator $\ddot{y} + 2\zeta\omega_n\dot{y} + \omega_n^2 y = n(t)$, where $n(t)$ is white noise with variance parameter $\sigma_n^2$. The steady-state mean-square response is:
\begin{equation}
    E[y^2] = \frac{\sigma_n^2}{4\zeta\omega_n^3}
\end{equation}
\end{theorem}
\end{tcolorbox}

\noindent Since $\delta a(t)$ has zero mean, the steady-state mean perturbation is $E[\delta q] = 0$, so $E[\delta q^2] = \mathrm{Var}[\delta q]$. Applying Theorem~\ref{thm:crandall} with $\sigma_n^2 = \omega_n^4\,\sigma^2$:
\begin{equation}
    \mathrm{Var}[\delta q]
    = \frac{\omega_n^4\,\sigma^2}{4\zeta\omega_n^3}
    = \frac{\omega_n\,\sigma^2}{4\zeta}
\end{equation}
Substituting $\omega_n = \sqrt{\mathbf{K}_p/m}$ and $\zeta = \mathbf{K}_d/(2\sqrt{m\mathbf{K}_p})$:
\begin{equation}
    \frac{\omega_n}{4\zeta}
    = \frac{\sqrt{\mathbf{K}_p/m}}{4 \cdot \mathbf{K}_d/(2\sqrt{m\mathbf{K}_p})}
    = \frac{\mathbf{K}_p}{2\,\mathbf{K}_d}
\end{equation}
where $m$ cancels completely. Therefore $\mathrm{Var}[\delta q] = \sigma^2\,\mathbf{K}_p / (2\,\mathbf{K}_d)$.
\end{proof}

\myparagraph{Interpretation. }The result can be understood through two competing physical effects:

\begin{enumerate}
\item \textbf{Error injection amplified by $\mathbf{K}_p$.} From~\eqref{eq:perturbation_ode}, the right-hand side $\mathbf{K}_p\,\delta a(t)$ shows that each action error enters the dynamics as a force proportional to $\mathbf{K}_p$. Higher stiffness means the same prediction mistake produces a larger force, injecting more energy into the error. Damping does not appear on the right-hand side because the damping force $-\mathbf{K}_d\dot{q}$ acts on velocity, not on the action target. 

\item \textbf{Error decays more with $\mathbf{K}_d$}. Between errors, a perturbation evolves under the homogeneous
dynamics $m\,\delta\ddot{q} + \mathbf{K}_d\,\delta\dot{q}
+ \mathbf{K}_p\,\delta q = 0$. The damping force
$-\mathbf{K}_d\,\delta\dot{q}$ opposes velocity, continuously
removing kinetic energy from the perturbation. Higher
$\mathbf{K}_d$ means faster energy dissipation. The
steady-state variance is the equilibrium where the rate of
energy injected by action errors (proportional to
$\mathbf{K}_p$) equals the rate of energy dissipated by
damping (proportional to $\mathbf{K}_d$).

\end{enumerate}

\begin{tcolorbox}[colback=gray!5, colframe=gray!50, boxrule=0.4pt, arc=0pt, breakable, left=4pt, right=4pt, top=4pt, bottom=4pt]
\begin{corollary}[State-Space Impact of Policy Errors]
\label{cor:state_error}
Let $\epsilon_\pi(\mathbf{K})$ denote the RMS action prediction error (square root of validation MSE) of a behavior cloning policy under gain setting $\mathbf{K}$. By Theorem~\ref{thm:variance}, the resulting steady-state RMS position deviation from the expert trajectory is:
\begin{equation}
   \epsilon_{\text{eff}} %
    = \sqrt{\frac{\mathbf{K}_p}{2\,\mathbf{K}_d}}
    \cdot \epsilon_\pi(\mathbf{K})
    \label{eq:effective_error}
\end{equation}
Since closed-loop task success depends on state deviation rather than action prediction accuracy, this quantity---not $\epsilon_\pi$ alone---governs performance.
\end{corollary}
\end{tcolorbox}

\noindent This explains the observation in Fig.~\ref{fig:bc-loss} that policies with \emph{higher} training loss often achieve \emph{better} closed-loop performance. In the compliant regime, action targets are harder to fit (higher $\epsilon_\pi$), but the attenuation factor $\sqrt{\mathbf{K}_p/(2\mathbf{K}_d)}$ is sufficiently small that the resulting state deviation remains low. Conversely, stiff gains yield low training loss but amplify residual errors through the dynamics.

\begin{remark}[The Attenuation--Difficulty Tradeoff]
\label{rem:tradeoff}
Theorem~\ref{thm:variance} holds for \emph{fixed} action noise variance $\sigma^2$. In practice, $\epsilon_\pi(\mathbf{K})$ is itself a function of the gains: compliant gains produce action targets that are larger in magnitude and potentially harder to fit, which may increase $\epsilon_\pi$. The state deviation in Corollary~\ref{cor:state_error} is therefore the product of a \emph{decreasing} attenuation factor and a potentially \emph{increasing} prediction difficulty. Theorem~\ref{thm:variance} establishes the \emph{mechanism} (attenuation), while the experiments in Section~\ref{sec:result-bc} establish empirically that the attenuation effect dominates prediction difficulty across our six tasks.
\end{remark}

\subsection{Quantitative Stiffness Analysis of Decoupling-Gains Experiment} \label{sec:appendix-existence-proof}
\begin{figure}[h]
    \centering
    \includegraphics[width=0.7\linewidth]{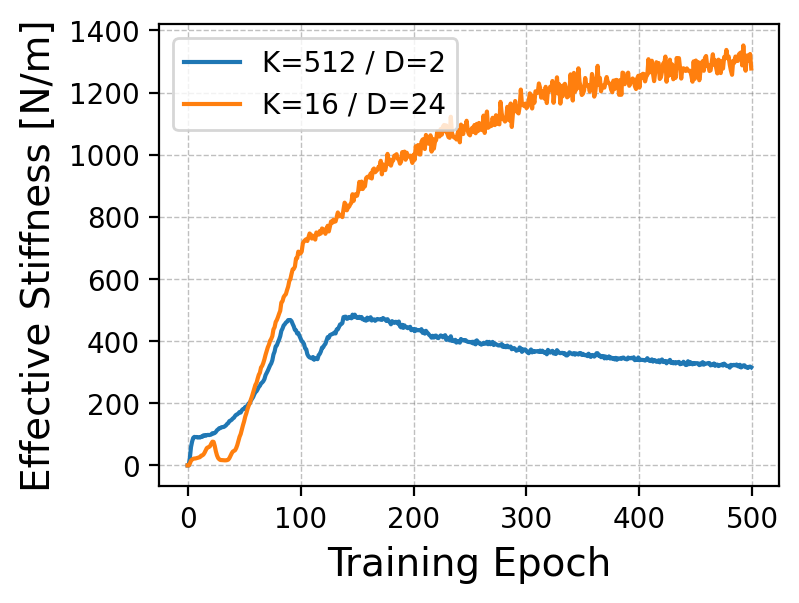}
    \caption{\textbf{Effective Cartesian stiffness throughout training for the two counterintuitive 
    pairings.} Despite $32\times$ lower actuator stiffness, the stiff-behavior policy achieves 
    ${\sim}5\times$ higher effective task-level stiffness than the compliant-behavior policy.}
    \label{fig:stiffness-quantification}
\end{figure}

Fig.~\ref{fig:stiffness-quantification} reports the effective Cartesian stiffness 
$\mathbf{K}_\text{eff} = |\mathbf{F}|/|\Delta \mathbf{x}|$ of each policy throughout training, 
measured via force-displacement system identification: a random translational force is applied 
to the end-effector and the steady-state displacement is recorded. Despite operating with 
$32\times$ lower actuator stiffness, the stiff-behavior policy converges to ${\sim}5\times$ 
higher effective task-level stiffness than the compliant-behavior policy.

\subsection{Behavior Cloning} \label{sec:appendix-bc}
\subsubsection{Task Descriptions} 

The six tasks we study are: Bimanual Handover, Dishrack Unload, Dishrack Load, Dishwasher Open, Mug Hang, and Block Stack (Figure \ref{fig:bc-tasks}). For all tasks besides Block Stack, we collect 100 teleoperated demonstrations with the Apple Vision Pro \cite{park2024avp, park2024dexhub} for each task. For Block Stack, we use motion-planned trajectories. These demonstrations are collected at 500Hz recording raw torques generated from operational space controller, and are retargeted to joint-level position targets for each gain setting at 50Hz for gain-dependent policy learning.

\begin{figure}[h]
    \centering
  \begin{minipage}{0.24\textwidth}
    \vspace{5pt}
    \centering
    \subcaptionbox{\centering Bimanual Handover\label{fig:bimanual-handover}}{%
      \includegraphics[width=\linewidth]{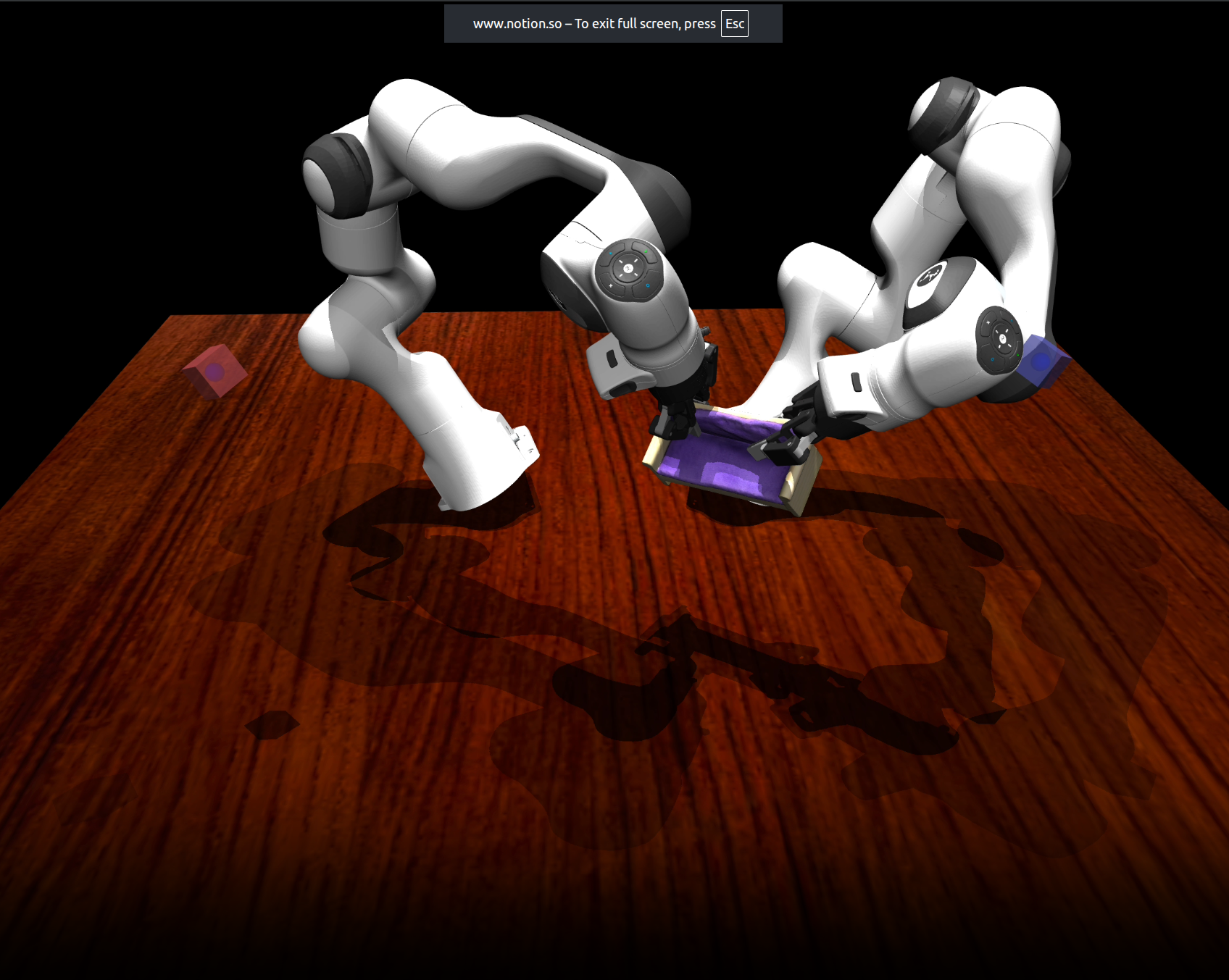}
    }
  \end{minipage}
  \hfill
  \begin{minipage}{0.24\textwidth}
    \vspace{5pt}
    \centering
    \subcaptionbox{\centering Dishrack Unload\label{fig:dishrack-unload}}{%
      \includegraphics[width=\linewidth]{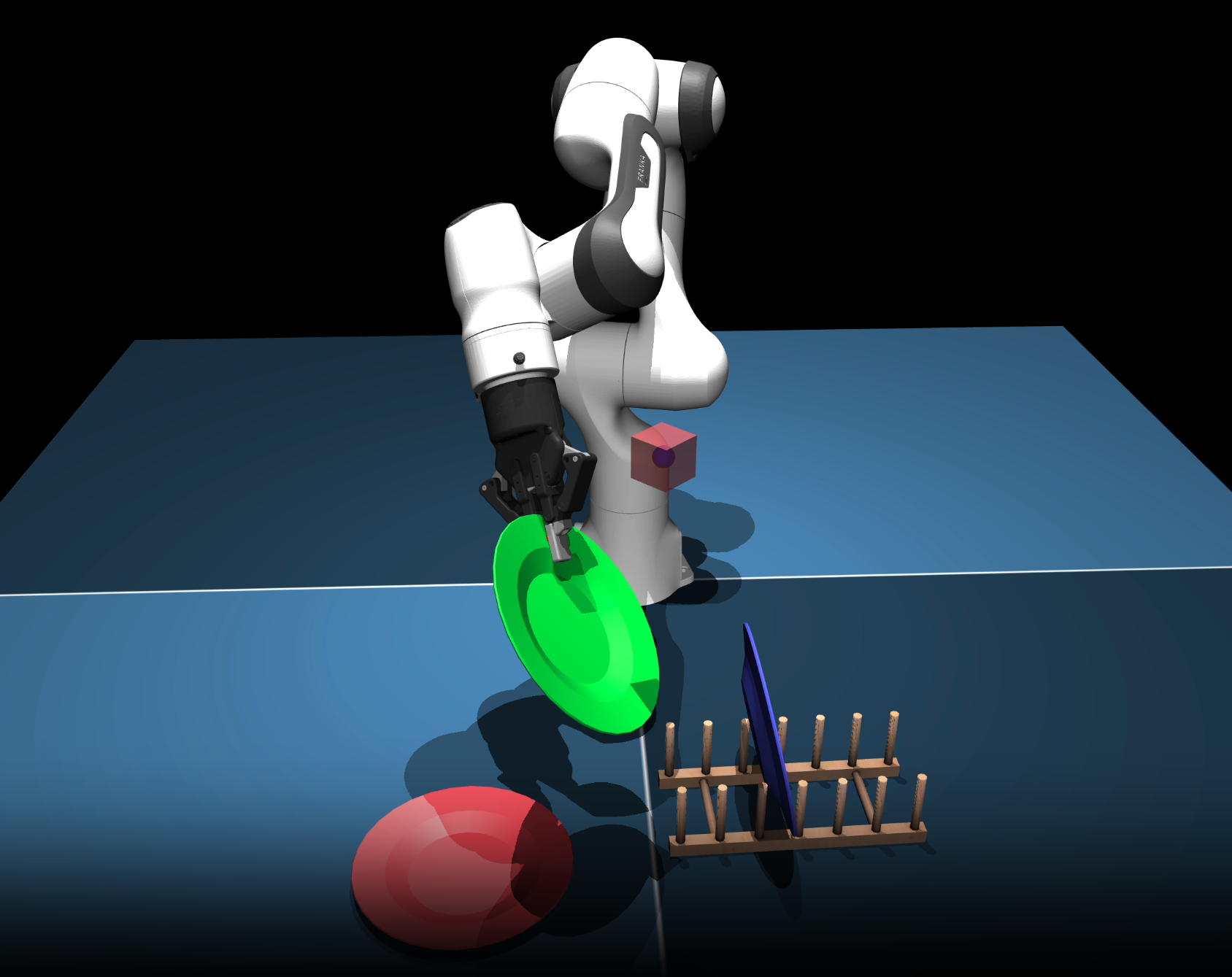}
    }
  \end{minipage}
  \begin{minipage}{0.24\textwidth}
    \vspace{5pt}
    \centering
    \subcaptionbox{\centering Dishrack Load\label{fig:disrack-load}}{%
      \includegraphics[width=\linewidth]{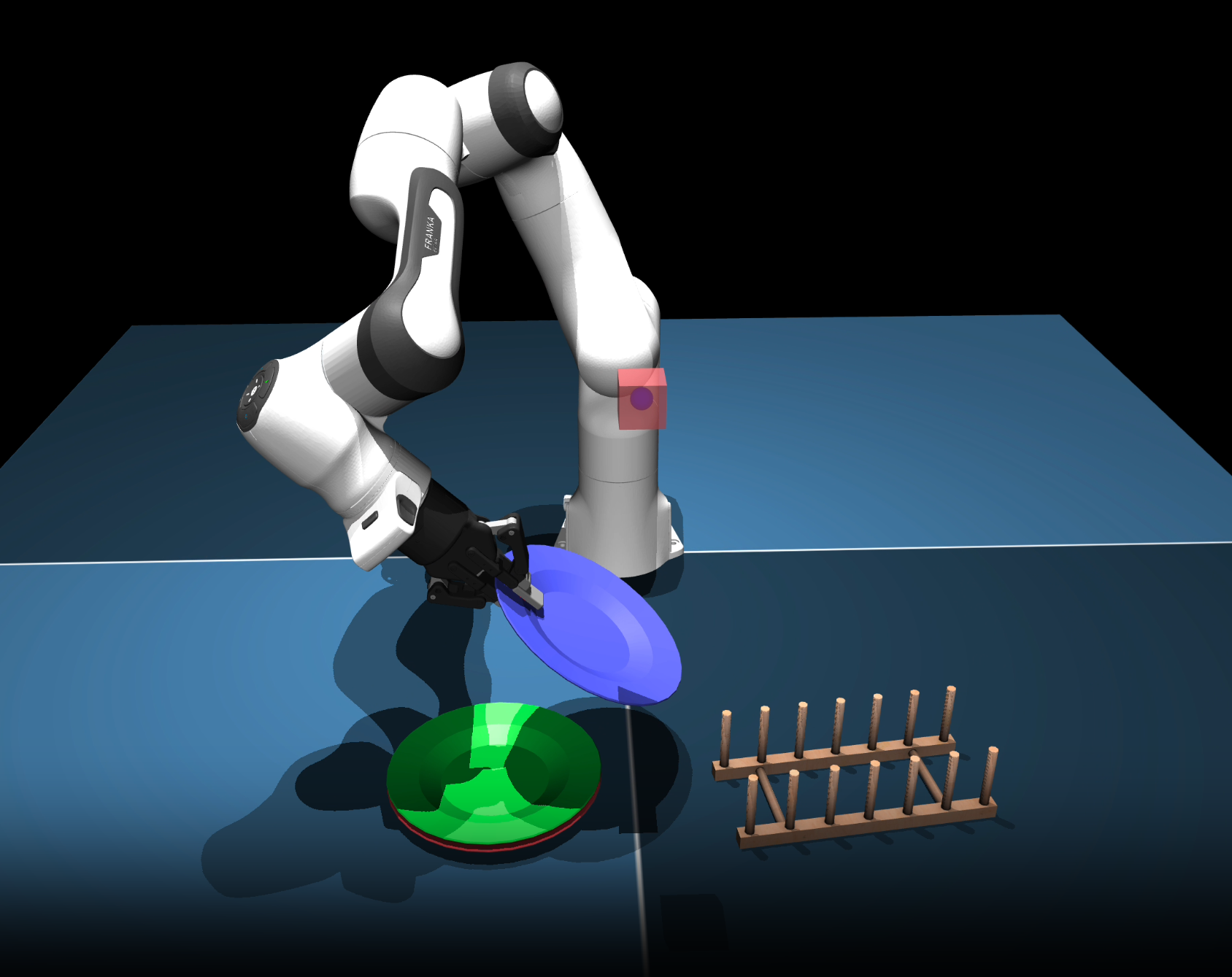}
    }
  \end{minipage}
  \hfill
  \begin{minipage}{0.24\textwidth}
    \vspace{5pt}
    \centering
    \subcaptionbox{\centering Dishwasher Open \label{fig:dishwasher-open}}{%
      \includegraphics[width=\linewidth]{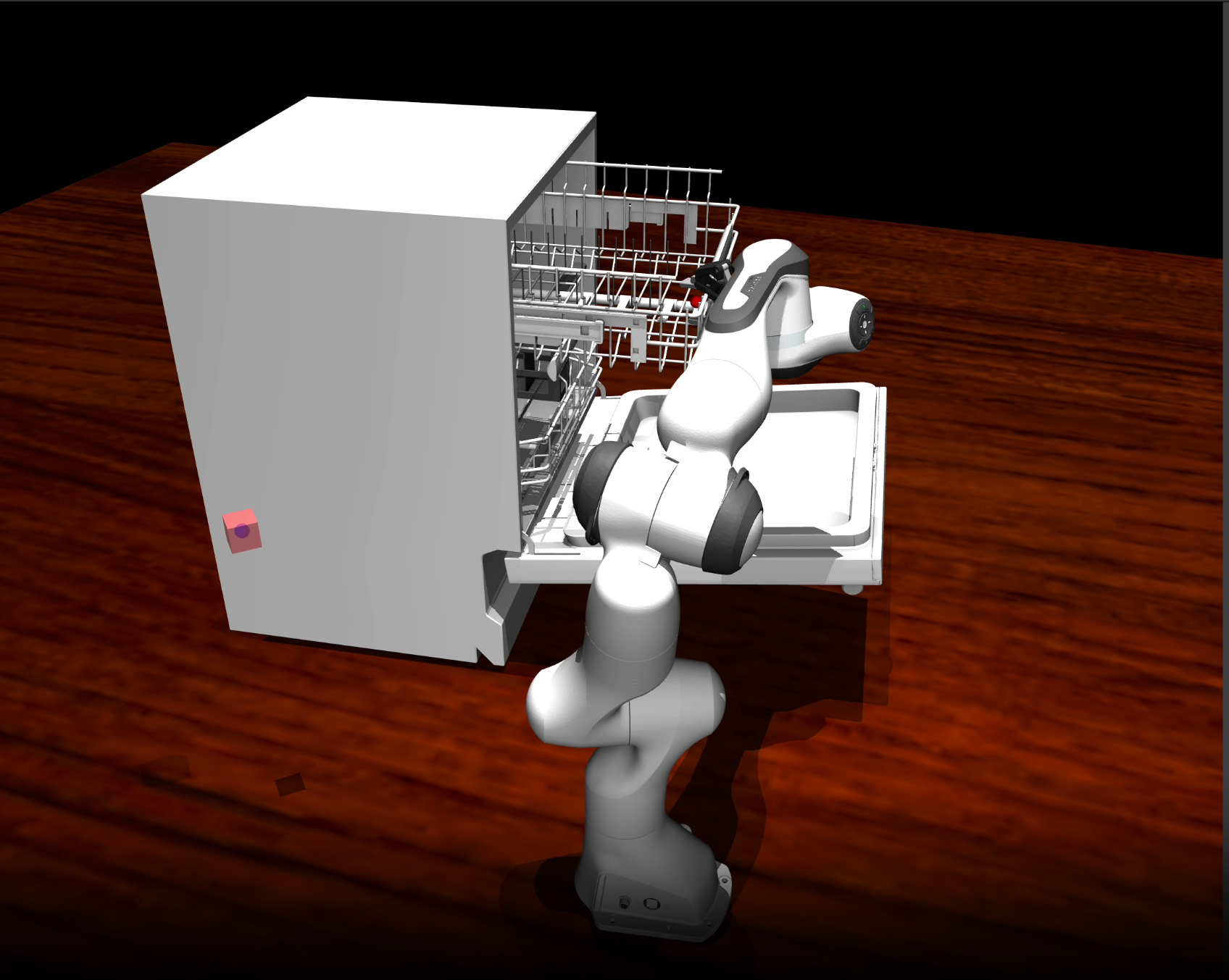}
    }
  \end{minipage}
    \begin{minipage}{0.24\textwidth}
    \vspace{5pt}
    \centering
    \subcaptionbox{\centering Mug Hang \label{fig:mug-hang}}{%
      \includegraphics[width=\linewidth]{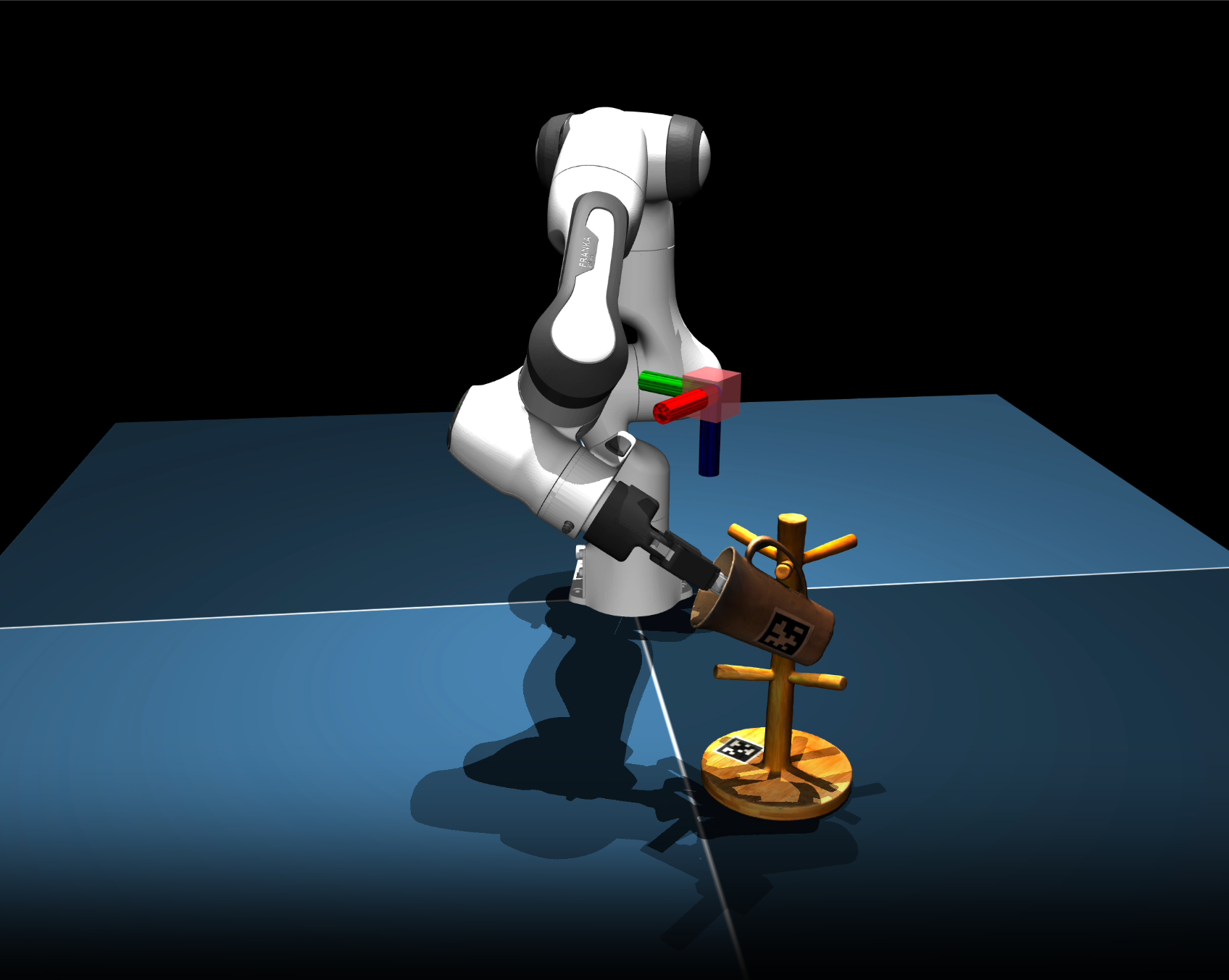}
    }
  \end{minipage}
  \hfill
    \begin{minipage}{0.24\textwidth}
    \vspace{5pt}
    \centering
    \subcaptionbox{\centering Block Stack \label{fig:block-stack}}{%
      \includegraphics[width=\linewidth]{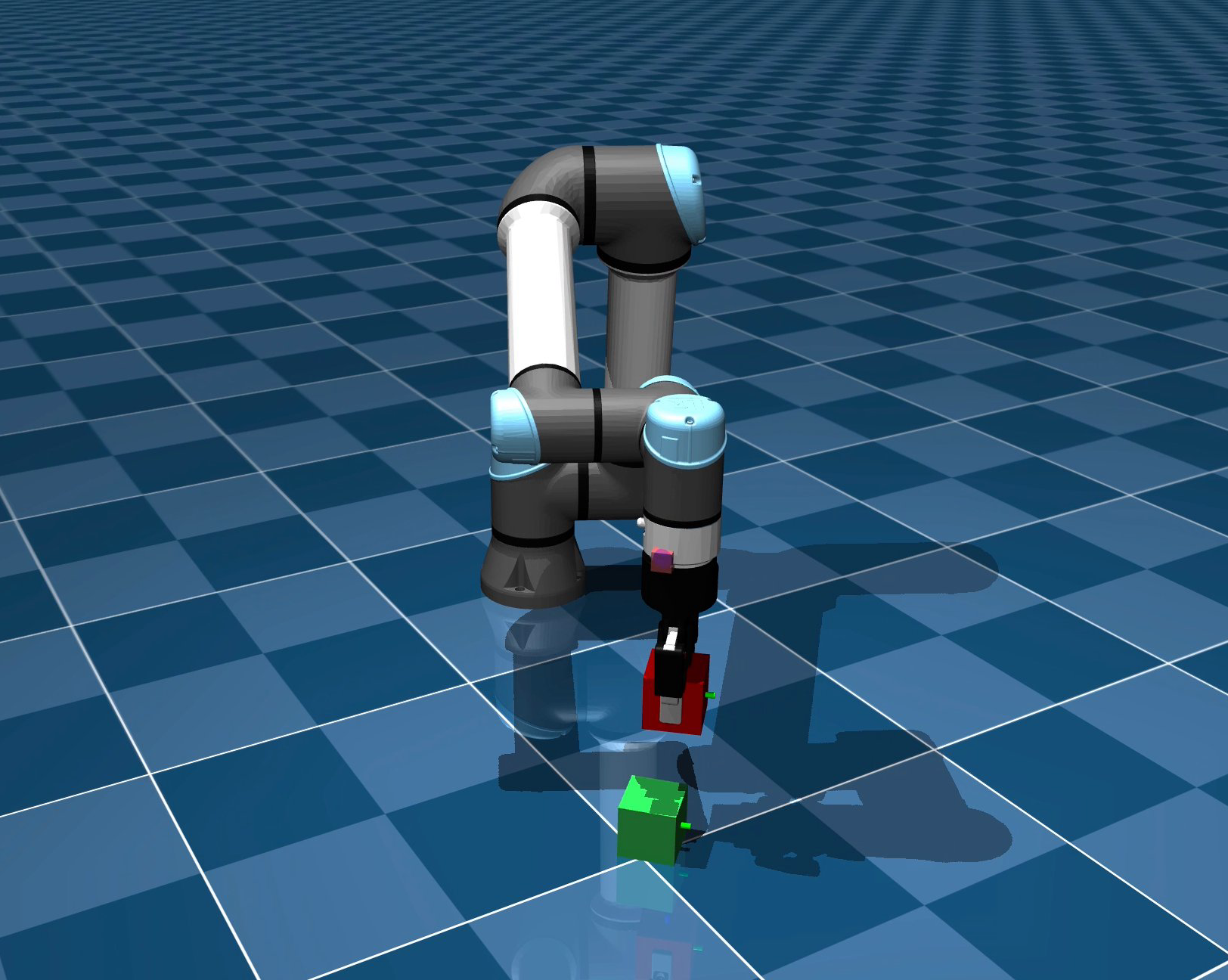}
    }
  \end{minipage}
\caption{Six tasks used for behavior cloning.}
\label{fig:bc-tasks}
\end{figure}

\subsubsection{Nominal Training Configuration} As a nominal configuration, we use VAE as a generative model with MLP network with observation size 10 and action chunk size 10, with privileged simulation states as inputs, using absolute joint as action space.

\subsubsection{Ablation Training Configurations} We present ablation results across dataset size (Figure \ref{fig:appendix-bc-dataset-size}), policy architectures (Figure \ref{fig:appendix-bc-architecture}), action chunk size (Figure \ref{fig:appendix-chunking}), action representation (Figure \ref{fig:appendix-action}), and control frequency (Figure \ref{fig:appendix-frequency}). Across all ablations, we observe a similar preference for compliant and overdamped gain regimes.

\subsubsection{Scaling Law}
Beyond absolute performance, the choice of controller gains also affects how efficiently policies improve with additional data. As shown in Fig.~\ref{fig:bc-scaling}, compliant and overdamped gains exhibit steeper scaling with dataset size, implying that data collection efforts yield greater returns in this regime. For practitioners with limited demonstration budgets, this makes gain selection a critical lever for maximizing policy performance.

\subsubsection{TPR Fidelity Validation} \label{sec:appendix-tpr-validation}

To quantify how faithfully Torque-to-Position Retargeting (TPR) preserves the original demonstration trajectories, we retarget a motion-planned Block Stacking trajectory to four representative gain configurations spanning the gain grid corners and evaluate at varying decimation rates (from $1\times$ at 500\,Hz down to $50\times$ at 10\,Hz). For each setting, we measure: (1) task success rate across 100 rollouts, and (2) joint-position MSE between the retargeted and original state trajectories.

As shown in Fig.~\ref{fig:tpr-validation}, both metrics remain robust up to $25\times$ decimation (20\,Hz): success rates stay ${\geq}90\%$ and joint-position MSE remains below $10^{-3}$ across all four gain configurations. Beyond $25\times$ decimation, success degrades for contact-rich phases of the task, as the zeroth-order hold assumption becomes less accurate when contact dynamics dominate between update steps. These results confirm that TPR produces near-identical state trajectories across gain settings at the policy frequencies used in our experiments (50\,Hz), validating the controlled comparison in Section~\ref{sec:exp-bc}.

\begin{figure}[h]
    \centering
    \includegraphics[width=1.0\columnwidth]{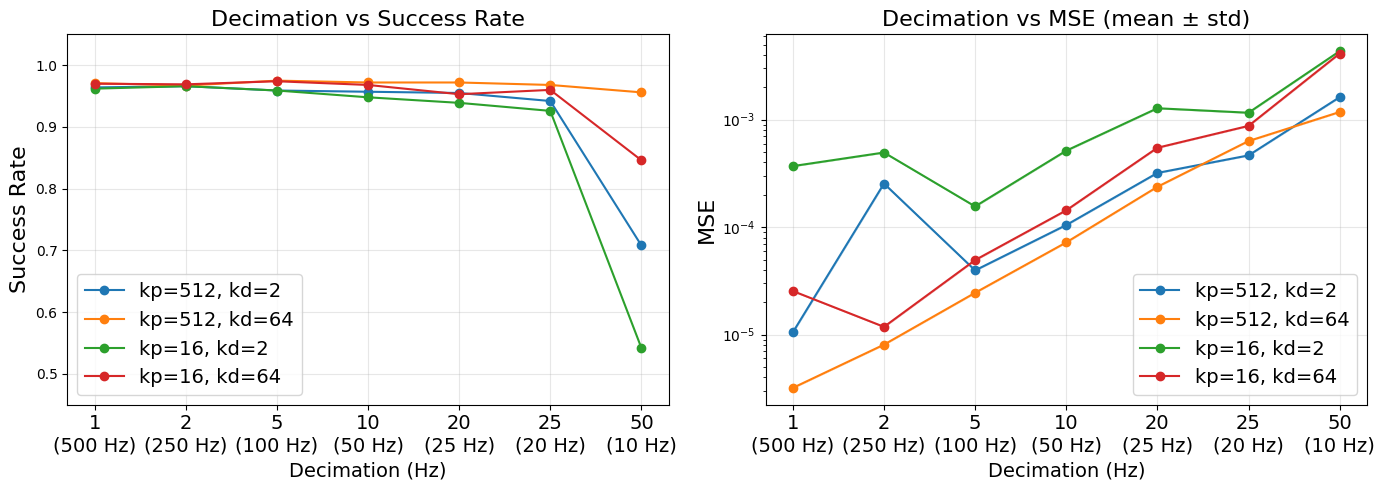}
    \caption{\textbf{TPR fidelity across gain configurations and decimation rates.} Success rate and joint-position MSE remain robust (${\geq}90\%$ success, MSE ${<}10^{-3}$) up to $25\times$ decimation (20\,Hz) for all gain settings.}
    \label{fig:tpr-validation}
\end{figure}

\subsubsection{Extension to Task-Space Position Control} \label{sec:appendix-tpr-taskspace}

While the TPR formulation in Section~\ref{sec:exp-bc} addresses joint-space position control, many manipulation systems instead use operational space control (OSC)~\cite{khatib2003unified} with SE(3) end-effector pose targets. OSC computes joint torques through a task-space impedance law:
\begin{equation}
    \boldsymbol{\tau} = \mathbf{J}^\top \mathbf{M}_x \left( \mathbf{K}_p \tilde{\mathbf{x}} - \mathbf{K}_d \dot{\mathbf{x}} \right) + \boldsymbol{\tau}_{\text{null}},
\end{equation}
where $\tilde{\mathbf{x}}$ is the pose error (position and orientation), $\dot{\mathbf{x}}$ is the task-space velocity, and $\mathbf{M}_x$ is the task-space inertia matrix. The gains $\mathbf{K}_p, \mathbf{K}_d \in \mathbb{R}^{6 \times 6}$ now operate in Cartesian space, separately controlling translational and rotational compliance.

TPR extends naturally to this setting. We collect demonstrations using torque control and record the task-space wrench $\mathbf{F}(t) = \mathbf{M}_x(\mathbf{K}_p \tilde{\mathbf{x}} - \mathbf{K}_d \dot{\mathbf{x}})$ along with the end-effector pose $\mathbf{x}(t)$ and velocity $\dot{\mathbf{x}}(t)$. Retargeting to a new gain configuration $(\mathbf{K}_p', \mathbf{K}_d')$ then yields:
\begin{equation}
    \mathbf{x}_{\text{des}}(t) = \mathbf{x}(t) + \mathbf{K}_p'^{-1}\left(\mathbf{F}(t) + \mathbf{K}_d' \dot{\mathbf{x}}(t)\right),
\end{equation}
where the orientation component is handled by converting the resulting axis-angle error to a quaternion displacement.

\subsubsection{Statistical Significance Analysis} \label{sec:appendix-bc-stats}

We provide a formal statistical analysis to verify that the compliant-overdamped gain region $\mathcal{G}^{\text{CO}}$ significantly outperforms its complement $\mathcal{G} \setminus \mathcal{G}^{\text{CO}}$ across all six BC tasks. For each task and gain cell, we evaluate $N{=}100$ closed-loop rollouts and record the binary success outcome.

\myparagraph{Logistic Regression.} We fit a binomial logistic regression with $\log_2 \mathbf{K}_\text{p}$ and $\log_2 \mathbf{K}_\text{d}$ as predictors. Across all tasks, the coefficient $\beta_{\mathbf{K}_\text{p}}$ is consistently negative and $\beta_{\mathbf{K}_\text{d}}$ is consistently positive (Table~\ref{tab:statistical_analysis}), confirming that lower stiffness and higher damping are significant predictors of success.

\myparagraph{Barnard's Exact Test.} We apply one-sided Barnard's exact tests with Bonferroni correction ($\alpha_{\text{adj}} \approx 0.0083$, correcting for 6 tasks) under the null hypothesis:
\begin{equation}
\mathcal{H}_0 \colon P(\text{success} \mid \mathcal{G}^{\text{CO}}) \leq P(\text{success} \mid \mathcal{G} \setminus \mathcal{G}^{\text{CO}})
\end{equation}
As shown in Table~\ref{tab:statistical_analysis}, $\mathcal{H}_0$ is rejected for every task with $p \ll \alpha_{\text{adj}}$, providing strong evidence that the compliant-overdamped regime yields significantly higher BC performance.

\begin{table}[h]
\centering
\caption{
Statistical analysis of BC results. Success rates for compliant-overdamped ($\mathcal{G}^{\text{CO}}$) vs.\ other gain regions, logistic regression coefficients on $\log_2 \mathbf{K}_\text{p}$ and $\log_2 \mathbf{K}_\text{d}$, and Bonferroni-corrected one-sided Barnard's exact test $p$-values. $\mathcal{H}_0$ is rejected in all cases.
}
\vspace{-5pt}
\label{tab:statistical_analysis}
\footnotesize
\setlength{\tabcolsep}{2pt}
\begin{tabular}{@{}l cc cc c@{}}
\toprule
& \multicolumn{2}{c}{\textbf{Success Rate}} & \multicolumn{2}{c}{\textbf{Logistic Reg.}} & \textbf{Barnard's} \\
\cmidrule(lr){2-3} \cmidrule(lr){4-5} \cmidrule(lr){6-6}
\textbf{Task} & $\mathcal{G}^{\text{CO}}$ & $\mathcal{G} \setminus \mathcal{G}^{\text{CO}}$ & $\beta_{\mathbf{K}_{\text{p}}}$ & $\beta_{\mathbf{K}_{\text{d}}}$ & $p$-value \\
\midrule
Dishwasher Opening & \textbf{70.1\%} & 45.6\% & $-0.074$ & $+0.277$ & $1.6 \times 10^{-51}$ \\
Bimanual Handover & \textbf{35.1\%} & 16.5\% & $-0.173$ & $+0.225$ & $6.2 \times 10^{-35}$ \\
Mug Hanging & \textbf{62.4\%} & 40.9\% & $-0.146$ & $+0.155$ & $1.4 \times 10^{-40}$ \\
Dishrack Unloading & \textbf{12.7\%} & 6.9\% & $-0.144$ & $+0.159$ & $1.2 \times 10^{-5}$ \\
Dishrack Loading & \textbf{28.2\%} & 14.8\% & $-0.169$ & $+0.217$ & $2.3 \times 10^{-19}$ \\
Block Stacking & \textbf{85.1\%} & 39.0\% & $-0.265$ & $+0.429$ & $6.9 \times 10^{-231}$ \\
\midrule
\textit{Pooled} & \multicolumn{2}{c}{---} & $-0.160$ & $+0.220$ & --- \\
\bottomrule
\end{tabular}
\end{table}

\subsection{User Study} \label{sec:appendix-user-study}
\subsubsection{Task Description}
The non-prehensile box manipulation task used in the user study is shown in Figure~\ref{fig:user-study-task}. For each trial, users teleoperate a Franka Research 3 Robot with a SpaceMouse in order to push the box from an initial pose to the goal (Figure \ref{fig:user-study-goal}). The box is always initialized to the left and off-axis relative to the goal (Figure \ref{fig:user-study-during}), but the precise pose is random. The goal pose is fixed in every trial.
\begin{figure}[t]
    \centering
  \begin{minipage}{0.24\textwidth}
    \vspace{5pt}
    \centering
    \subcaptionbox{\centering Task In-Progress\label{fig:user-study-during}}{%
      \includegraphics[width=\linewidth]{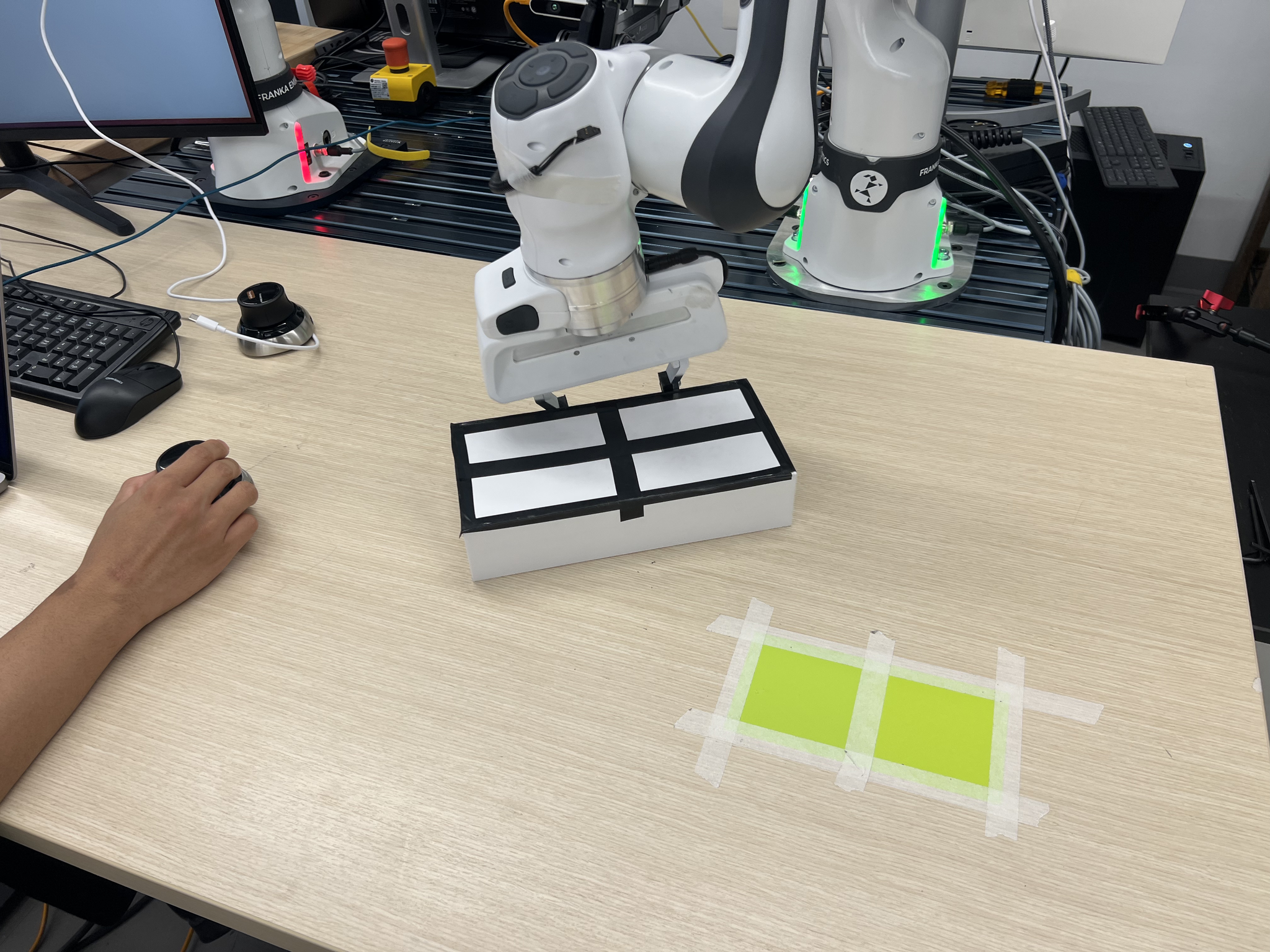}
    }
  \end{minipage}
  \hfill
  \begin{minipage}{0.24\textwidth}
    \vspace{5pt}
    \centering
    \subcaptionbox{\centering Task Complete\label{fig:user-study-goal}}{%
      \includegraphics[width=\linewidth]{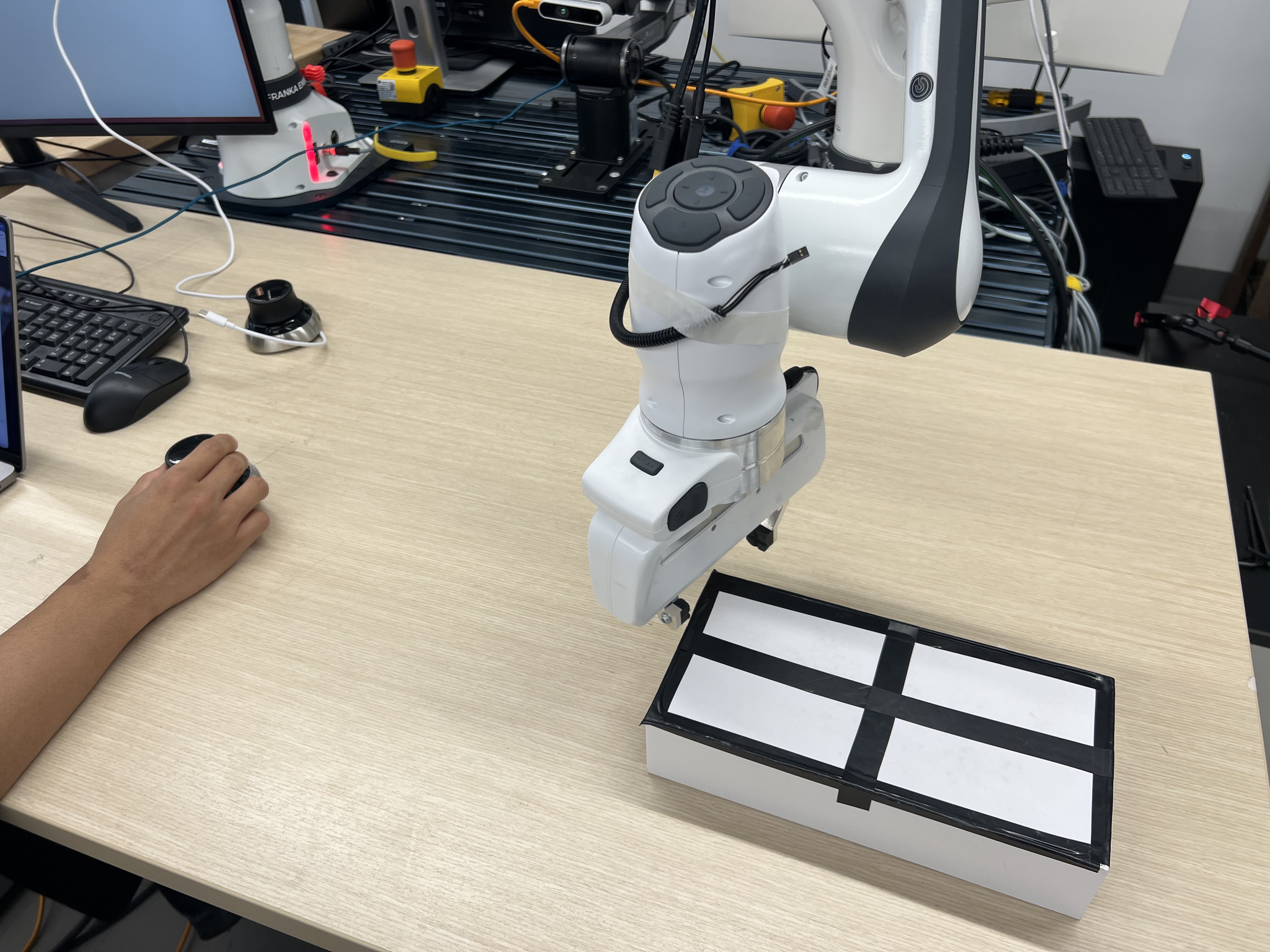}
    }
  \end{minipage}
\caption{\textbf{Non-prehensile box manipulation task for the user study.} A single trial of the task involves teleoperating the robot from a reset pose to make contact with the box, then pushing the box towards the goal. The task is complete when the green square is completely occluded by the box (b).}
\label{fig:user-study-task}
\end{figure}

\subsubsection{Experimental Design and Results}
As described in Section~\ref{sec:exp-bc}, the study collected 1,297 trials from 12 users over 1-hour sessions with randomized, blind gain presentation. The subjective rating is on a scale from 1--5, where 1 means the gain setting provides a completely unintuitive interface and 5 means a completely intuitive interface. Users complete the survey in Figure~\ref{fig:user_study_survey} after each trial. 

\begin{figure}[t]
    \centering
    \includegraphics[width=0.9\linewidth]{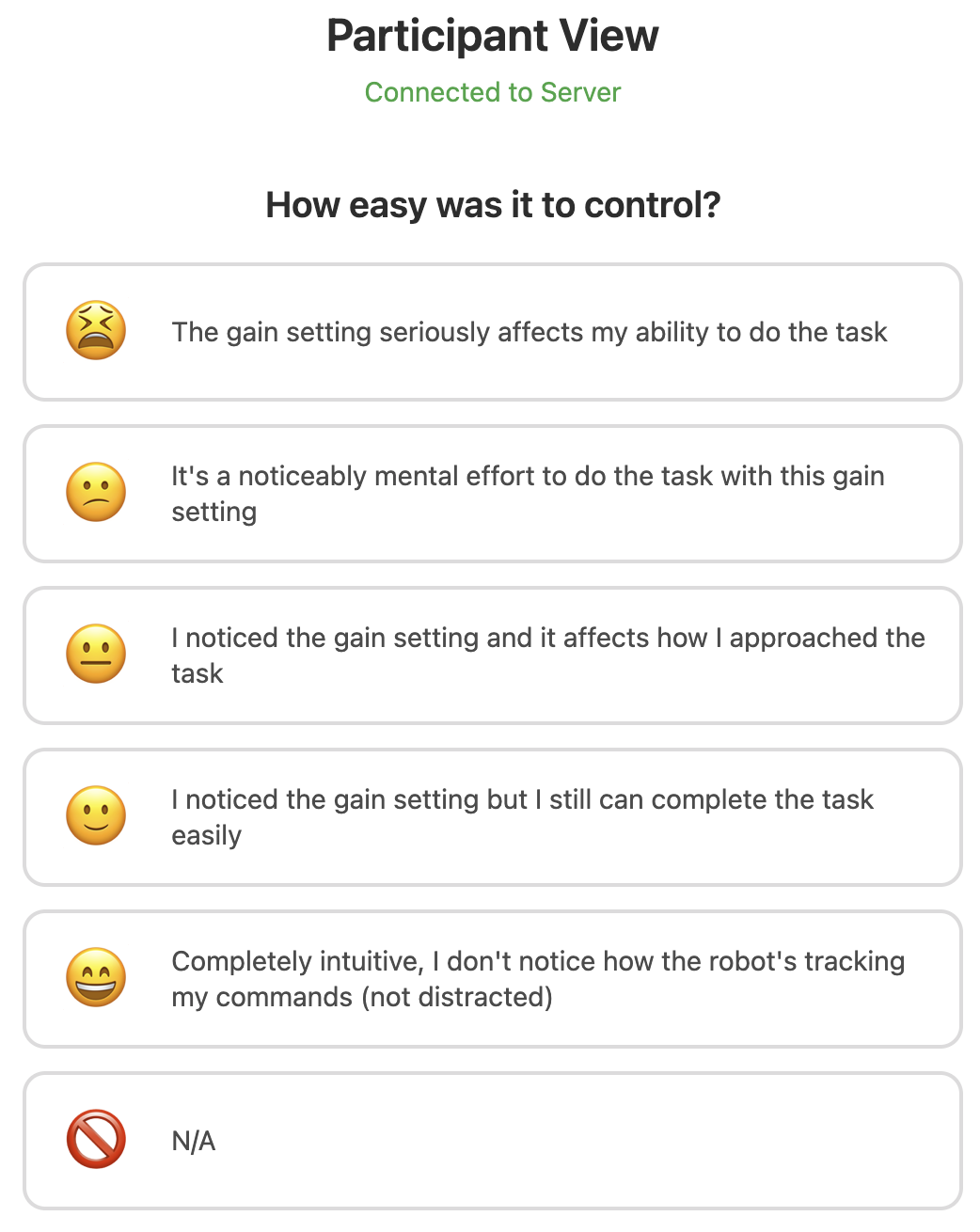}
    \caption{\textbf{User study survey.} After each trial, users complete the survey to rate their subjective experience teleoperating with a given gain setting.}
    \label{fig:user_study_survey}
\end{figure}

\subsection{Reinforcement Learning} \label{sec:rl-appendix}
\subsubsection{Task Descriptions}
The five tasks we study are: FR3 Joint-Reach, FR3 EE-Reach, FR3 Lift Cube, FR3 Open Drawer, and Unitree G1 Track Velocity (Figure \ref{fig:rl-tasks}). Each task is derived from the IsaacLab \cite{nvidia2025isaaclabgpuacceleratedsimulation} template environments. 

\begin{figure}[h]
    \centering
  \begin{minipage}{0.24\textwidth}
    \vspace{5pt}
    \centering
    \subcaptionbox{\centering FR3 Joint-Reach\label{fig:joint-reach}}{%
      \includegraphics[width=\linewidth]{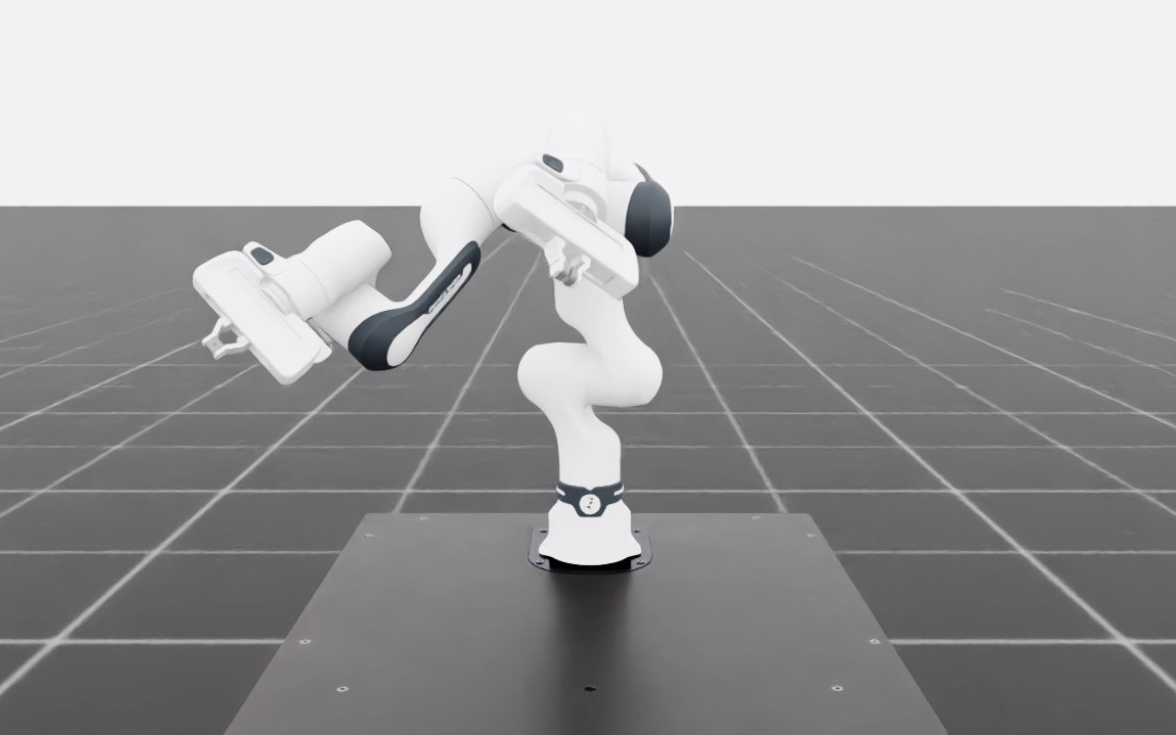}
    }
  \end{minipage}
  \hfill
  \begin{minipage}{0.24\textwidth}
    \vspace{5pt}
    \centering
    \subcaptionbox{\centering FR3 EE-Reach\label{fig:ee-reach}}{%
      \includegraphics[width=\linewidth]{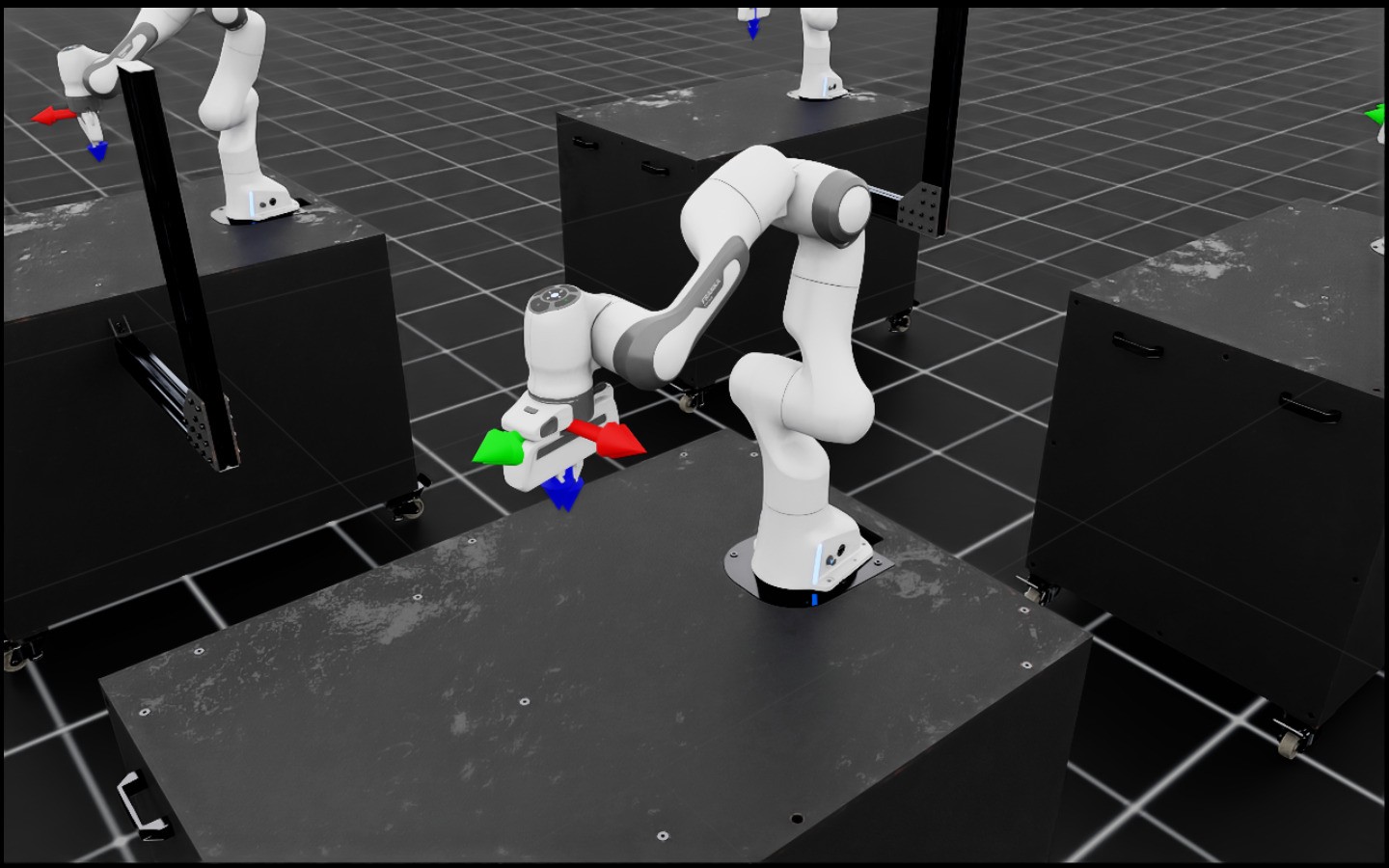}
    }
  \end{minipage}
  \begin{minipage}{0.24\textwidth}
    \vspace{5pt}
    \centering
    \subcaptionbox{\centering FR3 Lift Cube\label{fig:lift-cube}}{%
      \includegraphics[width=\linewidth]{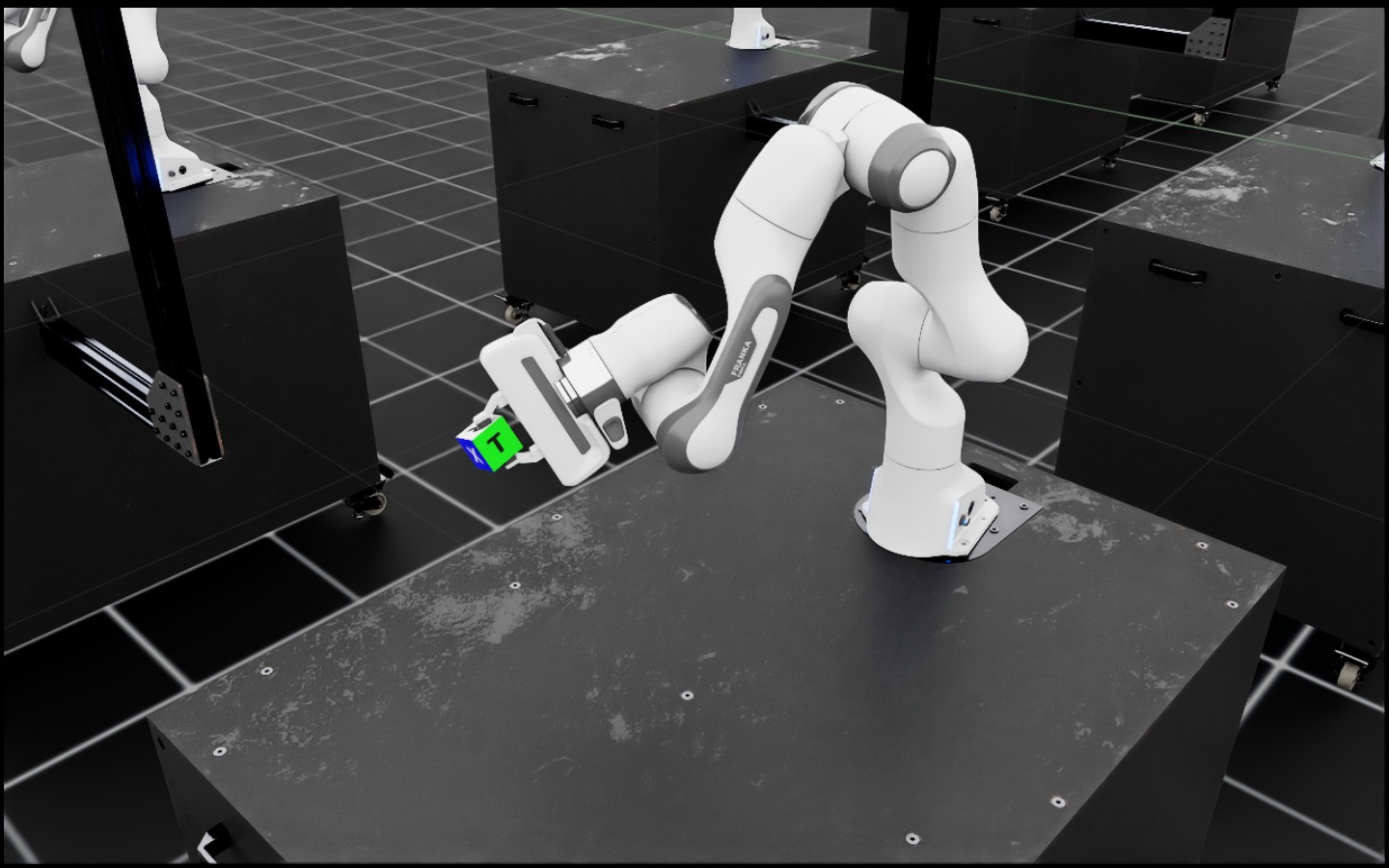}
    }
  \end{minipage}
  \hfill
  \begin{minipage}{0.24\textwidth}
    \vspace{5pt}
    \centering
    \subcaptionbox{\centering FR3 Open Drawer \label{fig:open-Drawer}}{%
      \includegraphics[width=\linewidth]{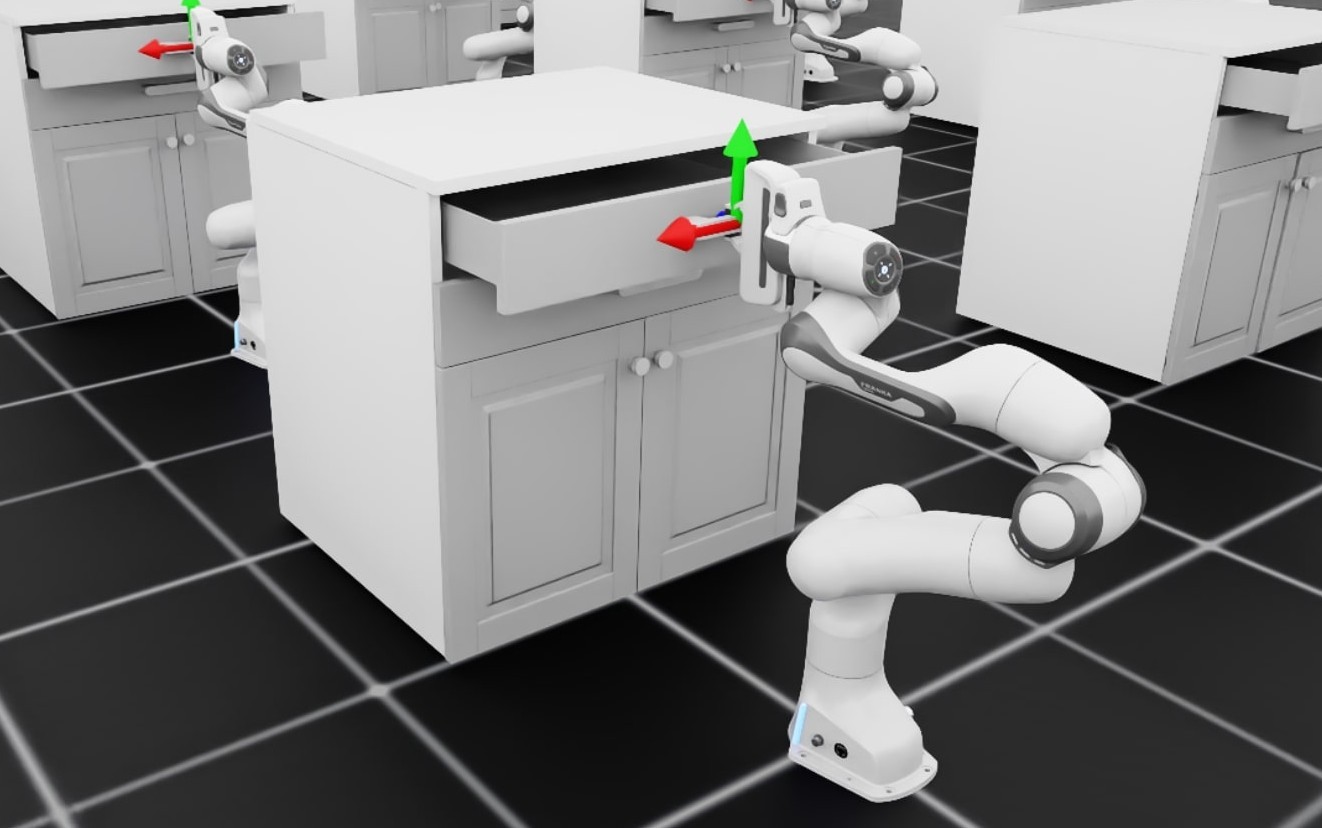}
    }
  \end{minipage}
    \begin{minipage}{0.24\textwidth}
    \vspace{5pt}
    \centering
    \subcaptionbox{\centering Unitree G1 Track Velocity \label{fig:g1}}{%
      \includegraphics[width=\linewidth]{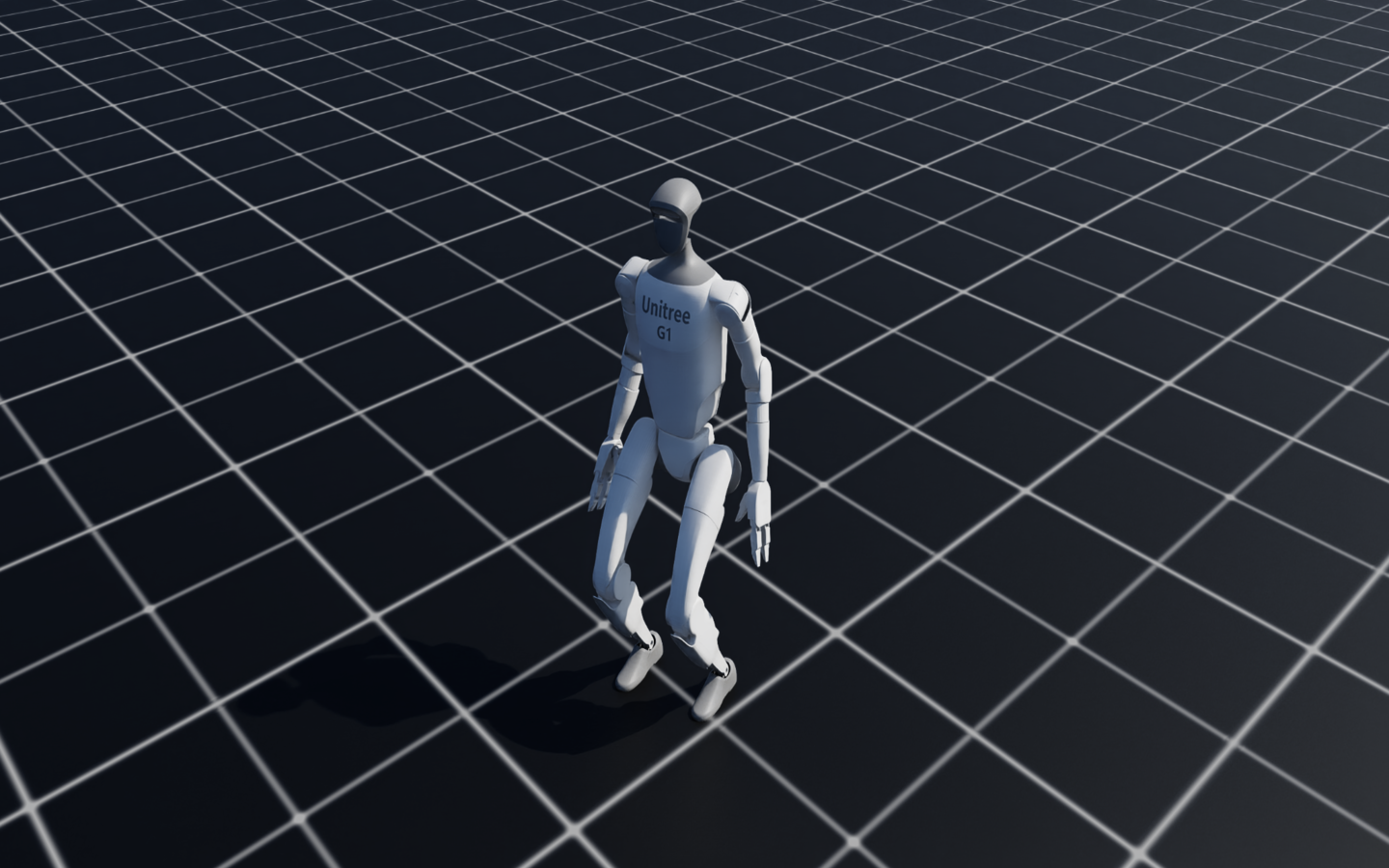}
    }
  \end{minipage}
  \hfill
\caption{\textbf{Five tasks for online RL solution existence proof.} For each task, we trained a successful policy for 8+ gain configurations spanning the range of stiff / compliant, overdamped / underdamped.}
\label{fig:rl-tasks}
\end{figure}

\subsubsection{Action Representations} 

\begin{figure}[H]
\centering
\resizebox{\columnwidth}{!}{%
\begin{tikzpicture}[
    node distance=0.4cm and 0.5cm,
    block/.style={draw, rounded corners, minimum height=0.7cm, minimum width=1.0cm, align=center, font=\small},
    bigblock/.style={draw, rounded corners, minimum height=0.9cm, minimum width=2.2cm, align=center, font=\small},
    op/.style={draw, circle, inner sep=1pt, font=\small},
    arr/.style={-{Stealth[length=2mm]}, thick},
    lbl/.style={font=\footnotesize, text=black!60},
]

\node[block] (pi) {$\pi_\theta(\mathbf{s}_t)$};
\node[op, right=0.5cm of pi] (mul) {$\times$};
\node[op, right=0.4cm of mul] (sum) {$+$};
\node[bigblock, right=0.5cm of sum] (pd) {$\mathbf{K}_p(\mathbf{q}_{\text{des}} {-} \mathbf{q}) + \mathbf{K}_d(\dot{\mathbf{q}}_{\text{des}} {-} \dot{\mathbf{q}})$};
\node[block, right=0.5cm of pd] (robot) {Robot};

\node[above=0.3cm of mul, font=\small] (alpha) {$\boldsymbol{\alpha}$};
\node[below=0.45cm of sum, font=\small] (ref) {$\mathbf{q}_{\text{ref}}$};

\draw[arr] (pi) -- (mul);
\draw[arr] (alpha) -- (mul);
\draw[arr] (mul) -- (sum);
\draw[arr] (ref) -- (sum);
\draw[arr] (sum) -- node[above, lbl] {$\mathbf{q}_{\text{des}}$} (pd);
\draw[arr] (pd) -- node[above, lbl] {$\boldsymbol{\tau}$} (robot);

\draw[arr] (robot.south) -- ++(0,-1.0) -| node[below, near start, lbl] {$\mathbf{s}_t = (\mathbf{q}, \dot{\mathbf{q}}, \ldots)$} (pi.south);

\end{tikzpicture}%
}
\caption{\textbf{Action representation.} The policy output is scaled by a per-joint-group vector $\boldsymbol{\alpha}$ and added to a reference position $\mathbf{q}_{\text{ref}}$ to produce the position target $\mathbf{q}_{\text{des}}$ sent to the PD controller.}
\label{fig:action-rep}
\end{figure}
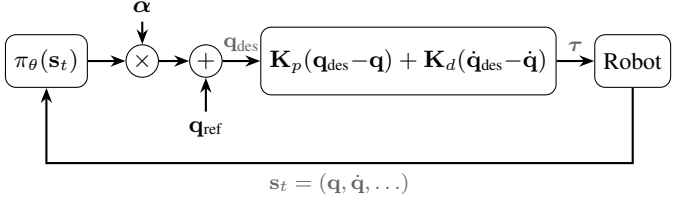

For all tasks, the position target sent to the low-level PD controller at each timestep is:
\begin{align}
\mathbf{q}_{\text{des}}(t) &= \boldsymbol{\alpha} \odot \pi_\theta(\mathbf{s}_t) + \mathbf{q}_{\text{ref}}(t) \\
\boldsymbol{\alpha} &= [\underbrace{\alpha_1, \ldots, \alpha_1}_{\mathcal{G}_1},\; \underbrace{\alpha_2, \ldots, \alpha_2}_{\mathcal{G}_2}] \nonumber
\end{align}
where $\mathbf{q}_{\text{ref}}(t)$ is an offset equal to either the current joint position $\mathbf{q}(t)$ or the default joint position $\mathbf{q}_0$, depending on the task. Joints are partitioned into two groups, $\mathcal{G}_1$ and $\mathcal{G}_2$, with shared scale factors $\alpha_1$ and $\alpha_2$ respectively (Table~\ref{tab:rl-action}). Both scale factors are tuned via computational hyperparameter optimization~\cite{optuna_2019} to adapt the action space to each gain setting.

\begin{table}[H]
\centering
\caption{Action representation across RL tasks.}
\label{tab:rl-action}
\footnotesize
\setlength{\tabcolsep}{3pt}
\begin{tabular}{l c l l c}
\toprule
\textbf{Task} & $\mathbf{q}_{\text{ref}}(t)$ & $\mathcal{G}_1$ & $\mathcal{G}_2$ & Gripper \\
\midrule
FR3 Joint-Reach  & $\mathbf{q}(t)$  & $q_{0\text{--}3}$ (elbow) & $q_{4\text{--}6}$ (wrist) & -- \\
FR3 EE-Reach     & $\mathbf{q}(t)$  & $q_{0\text{--}3}$ (elbow) & $q_{4\text{--}6}$ (wrist) & -- \\
FR3 Lift Cube    & $\mathbf{q}(t)$  & $q_{0\text{--}3}$ (elbow) & $q_{4\text{--}6}$ (wrist) & binary \\
FR3 Open Drawer  & $\mathbf{q}(t)$  & $q_{0\text{--}3}$ (elbow) & $q_{4\text{--}6}$ (wrist) & binary \\
G1 Locomotion    & $\mathbf{q}_0$  & $q_{0\text{--}13} $ (lower) &  $q_{13\text{--}36}$ (upper)          & -- \\
\bottomrule
\end{tabular}
\end{table}

For tasks with a gripper (Table~\ref{tab:rl-action}), the policy outputs an
additional continuous value that is thresholded at zero, commanding the fingers
to either fully opened (0.04\,m) or fully closed (0.0\,m). The gripper joint
gains are held fixed across all experiments.

\subsubsection{Success Criteria}
For each policy trained during hyperparameter optimization, we record the success rate across 100 simulated trials according to the success metrics in Table \ref{tab:rl-success}. We evaluate the best (highest reward) checkpoint for each policy.
\begin{table}[H]
\centering
\caption{Success criteria for each RL task.}
\label{tab:rl-success}
\footnotesize
\setlength{\tabcolsep}{4pt}
\begin{tabular}{l l l}
\toprule
\textbf{Task} & \textbf{Criterion} & \textbf{Threshold} \\
\midrule
FR3 Joint-Reach & $\|\mathbf{q} - \mathbf{q}_{\text{goal}}\| < \epsilon$ & $\epsilon = 0.1$ rad \\[4pt]
FR3 EE-Reach    & $\|\mathbf{p} - \mathbf{p}_{\text{goal}}\| < \epsilon_p,$ & $\epsilon_p = 0.02$ m \\
                & $\|\Delta\theta\| < \epsilon_r$ & $\epsilon_r = 0.1$ rad \\[4pt]
FR3 Lift Cube   & $\|\mathbf{p}_{\text{obj}} - \mathbf{p}_{\text{goal}}\| < \epsilon$ & $\epsilon = $ 0.05m \\[4pt]
FR3 Open Drawer & $d_{\text{drawer}} > d_{\text{min}}$ & $d_{\text{min}} = 0.2$ \\[4pt]
G1 Locomotion   & $\|\dot{\mathbf{q}}\| / \|\dot{\mathbf{q}}_{\text{goal}}\| > \rho$ & $\rho = 0.4$ \\
\bottomrule
\end{tabular}
\end{table}
In order to adapt the environment to new gain settings, we leverage computational hyperparameter optimization with Optuna to tune the action scales and, for the FR3 EE-Reach task, the reward term weights. We use the TPE optimizer with default hyperparameters, where the objective function is the task success rate.
\subsubsection{PPO Hyperparameters}
We use largely the same PPO hyperparameters as the IsaacLab \cite{nvidia2025isaaclabgpuacceleratedsimulation} template environments. Hyperparameters, including any changes we made, are reproduced here (Table \ref{tab:PPO-shared} and Table \ref{tab:PPO-differ}).
\begin{table}[H]
\centering
\caption{PPO hyperparameters shared across all tasks.}
\label{tab:PPO-shared}
\footnotesize
\begin{tabular}{l c}
\toprule
\textbf{Hyperparameter} & \textbf{Value} \\
\midrule
Algorithm & PPO (SKRL) \\
Discount factor $\gamma$ & $0.99$ \\
GAE $\lambda$ & $0.95$ \\
Learning epochs & $5$ \\
Clip range (ratio) & $0.2$ \\
Clip range (value) & $0.2$ \\
Grad norm clip & $1.0$ \\
LR scheduler & KL-Adaptive \\
Activation & ELU \\
Seed & $42$ \\
\bottomrule
\end{tabular}
\end{table}

\begin{table}[H]
\centering
\caption{PPO hyperparameters that vary across tasks.}
\label{tab:PPO-differ}
\footnotesize
\setlength{\tabcolsep}{3pt}
\begin{tabular}{l c c c c}
\toprule
\textbf{Hyperparameter} & \textbf{Reach} & \textbf{Lift} & \textbf{Drawer} & \textbf{G1} \\
\midrule
Network layers          & [64,64]       & [256,128,64] & [256,128,64] & [256,128,128] \\
Learning rate           & $1\text{e-}3$  & $1\text{e-}3$  & $5\text{e-}4$  & $1\text{e-}3$ \\
KL threshold            & $0.01$         & $0.01$         & $0.008$        & $0.01$ \\
Rollouts                & $24$           & $24$           & $96$           & $24$ \\
Mini-batches            & $4$            & $4$            & $96$           & $4$ \\
Entropy coeff.          & $0.01$         & $0.01$         & $0.001$        & $0.008$ \\
Value loss coeff.       & $1.0$          & $1.0$          & $2.0$          & $1.0$ \\
Min log std             & $-3.0$         & $-3.0$         & $-20.0$        & $-20.0$ \\
State preprocess      & RSS            & RSS            & --             & -- \\
Value preprocess      & RSS            & RSS            & --             & -- \\
Timesteps               & $24$k          & $48$k          & $38.4$k        & $12$k \\
\bottomrule
\\[-6pt]
\multicolumn{5}{l}{\scriptsize RSS = RunningStandardScaler.}
\end{tabular}
\end{table}

\subsection{Sim2Real}
\label{app:sim2real}

\subsubsection{System Identification Data Collection}

For each gain configuration $(\mathbf{K}_p, \mathbf{K}_d)$, the real robot
executes a sinusoidal reference trajectory
$\mathbf{q}_{\text{des}}(t) = \mathbf{q}_0 + 0.1 \sin(\pi t / 50)$
applied uniformly across all joints for 4 seconds.
During execution, we log joint positions $\mathbf{q}$, joint velocities
$\dot{\mathbf{q}}$, and desired positions $\mathbf{q}_{\text{des}}$ at 50\,Hz. The low-level torque controller on the real robot runs at 1\,kHz.

To match the real-world setup, the IsaacLab simulation environment updates
position commands at 50\,Hz with a physics simulation rate of 100\,Hz. We use 100\,Hz rather than 1\,kHz
physics to keep RL training times tractable. We note that this fidelity gap
between the real robot's 1\,kHz control loop and the simulator's 100\,Hz
physics rate contributes to the sim-to-real discrepancy that system
identification aims to minimize.

\subsubsection{System Identification Procedure}

For each gain configuration, we use CMA-ES~\cite{nomura2024cmaessimplepractical} to
optimize simulation parameters per-actuator $\psi$ (Table \ref{tab:sysid-sim-params}) to minimize the discrepancy
between real and simulated response trajectories.

\begin{table}[H]
\centering
\caption{System identification parameter bounds. Parameters are optimized per-actuator.}
\label{tab:sysid-sim-params}
\footnotesize
\setlength{\tabcolsep}{4pt}
\begin{tabular}{l c c}
\toprule
\textbf{Parameter} & \textbf{Lower} & \textbf{Upper} \\
\midrule
Stiffness $K_p$       & $1$ & $1024$ \\
Damping $K_d$         & $1$ & $1024$ \\
Armature              & $0$ & $0.5$ \\
Static friction       & $0.01$ & $1.0$ \\
Dynamic friction ratio & $0$ & $1.0$ \\
Viscous friction      & $0$ & $1.0$ \\
\bottomrule
\end{tabular}
\end{table}

The objective function is the sum of spectral MSE losses for
joint positions and velocities:
\begin{equation}
\mathcal{L}(\psi) =
\mathcal{L}_{\text{spec}}(\mathbf{q}^{\text{real}},\, \mathbf{q}^{\text{sim}}(\psi))
+ \mathcal{L}_{\text{spec}}(\dot{\mathbf{q}}^{\text{real}},\, \dot{\mathbf{q}}^{\text{sim}}(\psi))
\end{equation}
where $\mathcal{L}_{\text{spec}}$ computes the mean squared error
between the discrete Fourier transforms of the
simulated and real trajectories. Matching in the frequency domain
encourages the optimizer to capture oscillatory behavior and
damping characteristics.

CMA-ES runs for 200 iterations with an initial step size of
$\sigma = 3.0$, independently for all 49 gain configurations. We visualize the system identification result against the real-world trajectory for four gain settings in Figure \ref{fig:sysid-traj-result}.

\begin{figure*}[h]
    \centering
  \begin{minipage}{0.48\textwidth}
    \vspace{5pt}
    \centering
    \subcaptionbox{\centering Compliant, Overdamped ($K_p=16$, $K_d=24$)\label{fig:sysid-CO}}{%
      \includegraphics[width=\linewidth]{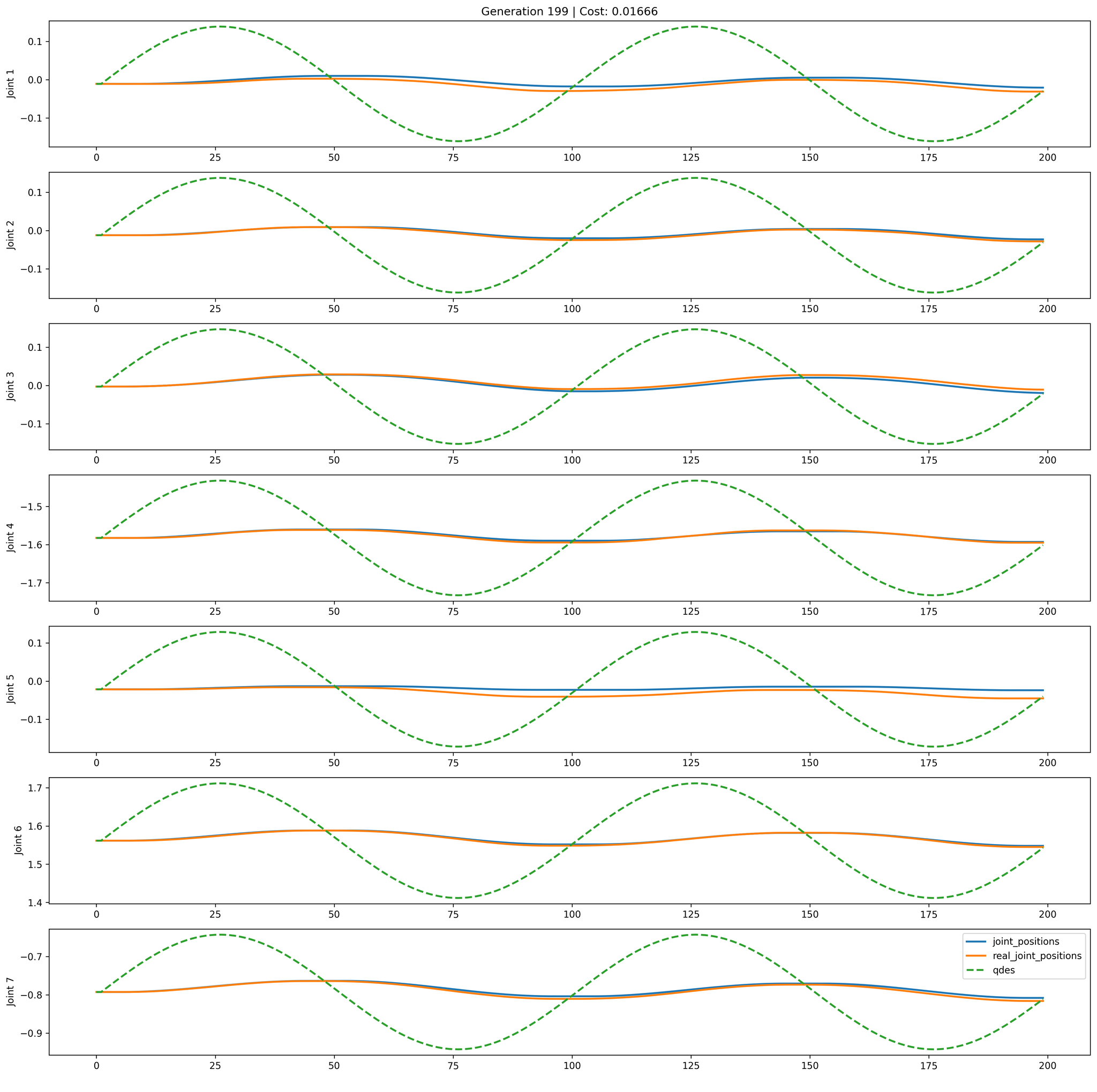}
    }
  \end{minipage}
  \hfill
  \begin{minipage}{0.48\textwidth}
    \vspace{5pt}
    \centering
    \subcaptionbox{\centering Stiff, Overdamped ($K_p=512$, $K_d=24$)\label{fig:sysid-SO}}{%
      \includegraphics[width=\linewidth]{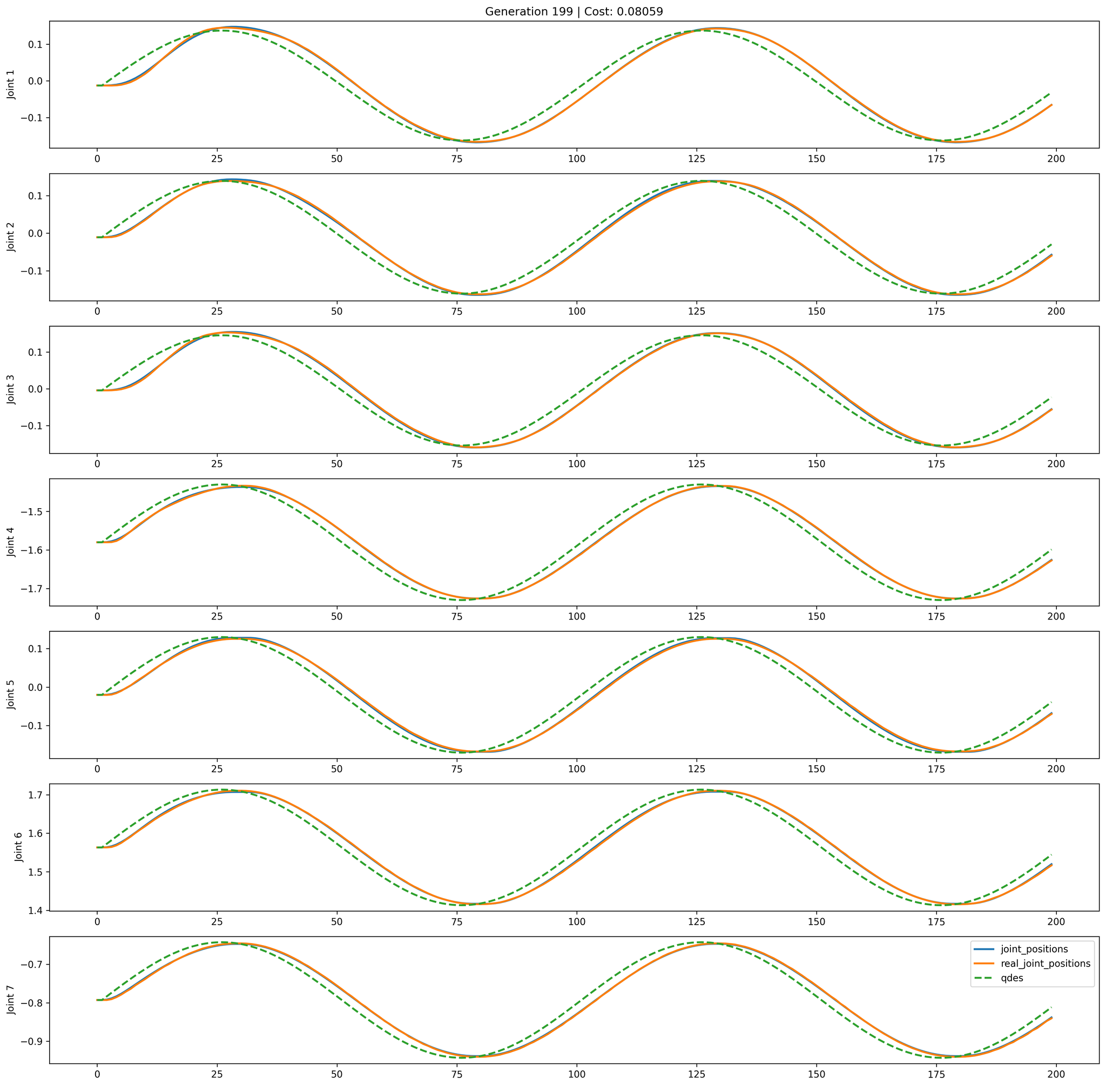}
    }
  \end{minipage}
  \begin{minipage}{0.48\textwidth}
    \vspace{5pt}
    \centering
    \subcaptionbox{\centering Compliant, Underdamped ($K_p=16$, $K_d=2$)\label{fig:sysid-CU}}{%
      \includegraphics[width=\linewidth]{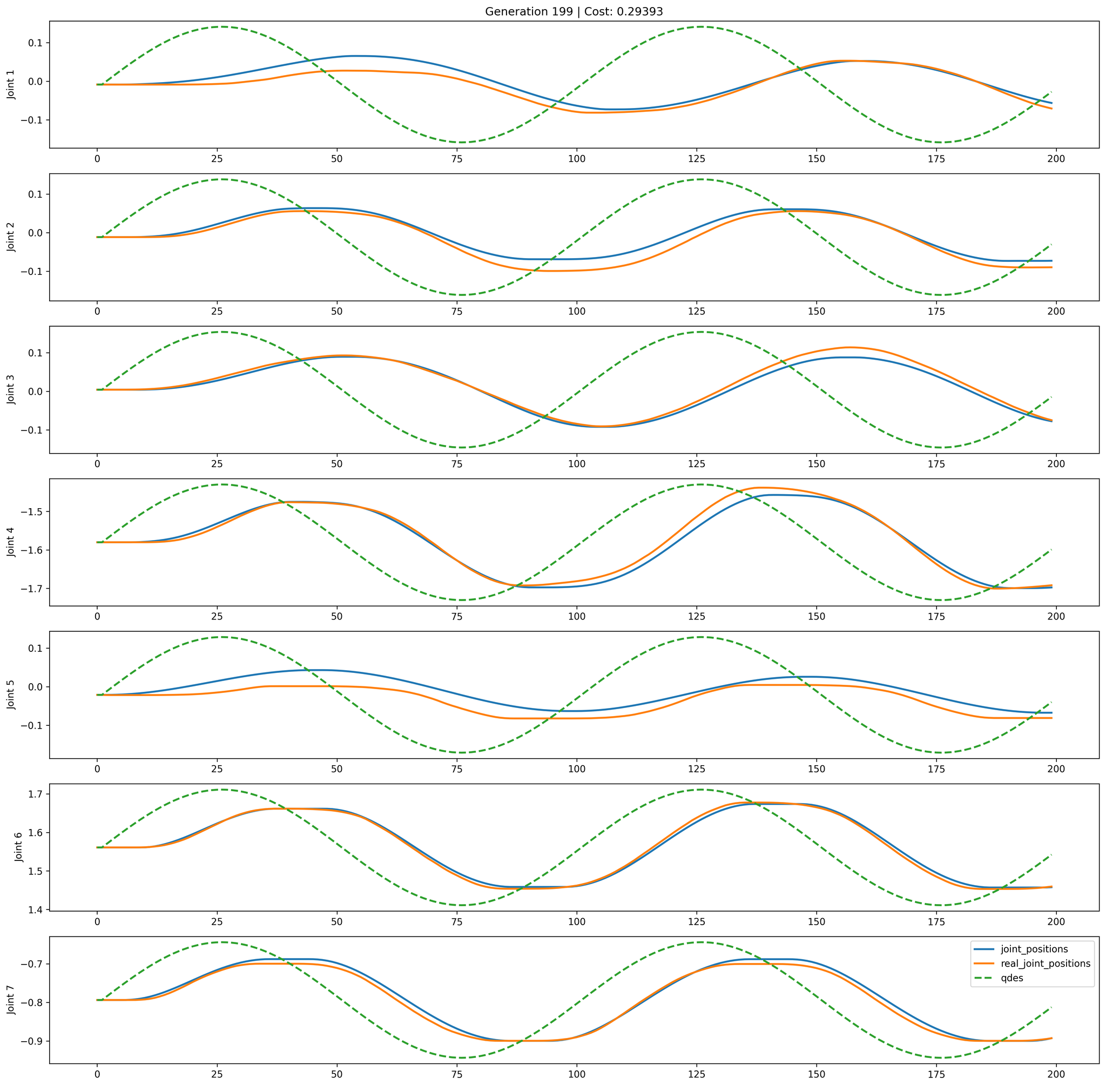}
    }
  \end{minipage}
  \hfill
  \begin{minipage}{0.48\textwidth}
    \vspace{5pt}
    \centering
    \subcaptionbox{\centering Stiff, Underdamped ($K_p=512$, $K_d=2$) \label{fig:sysid-SU}}{%
      \includegraphics[width=\linewidth]{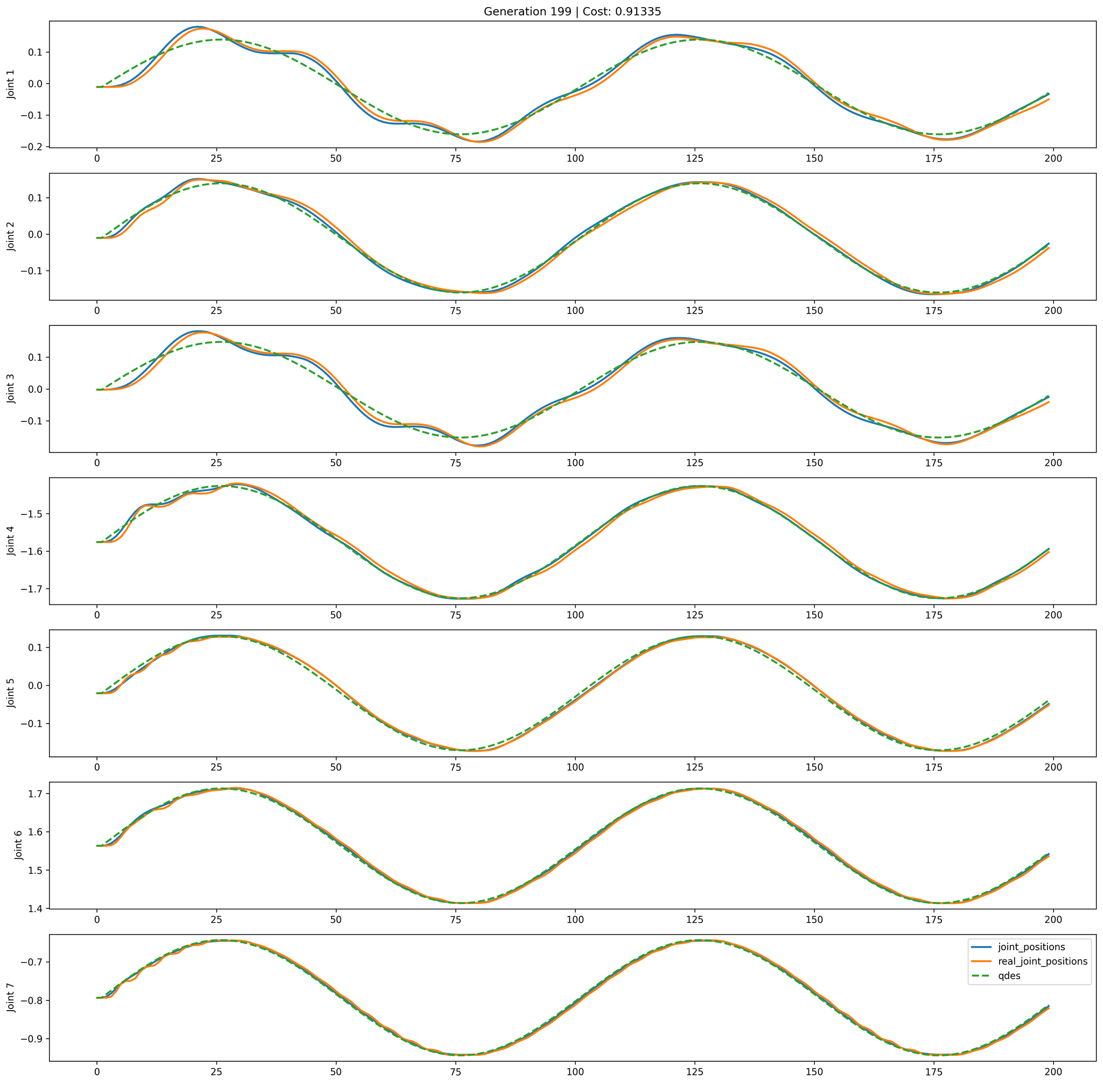}
    }
  \end{minipage}
  \hfill
\caption{\textbf{System identification result for sample gain settings in each gain regime.} We show commanded positions (green), real-world achieved positions (orange), and simulation positions (blue) achieved with the optimal actuator parameters.}
\label{fig:sysid-traj-result}
\end{figure*}

\subsubsection{Training Deployable Policies}

We train deployable FR3 Joint-Reach and FR3 EE-Reach policies. To discover policies that respect the
real robot's limits, we modify the outer-loop Optuna objective to a
two-stage formulation that always prefers constraint-satisfying
configurations over violating ones:
\begin{equation}
\mathcal{J} =
\begin{cases}
1 + r_{\text{success}} & \text{if all } v_c \leq \bar{v}_c \\[4pt]
r_{\text{success}} \displaystyle\prod_{c \in \mathcal{C}} \phi_c
& \text{otherwise}
\end{cases}
\end{equation}
where $v_c$ and $\bar{v}_c$ are the violation rate and allowed
threshold for each constraint
$c \in \mathcal{C} = \{$position, velocity, torque, torque rate$\}$,
and the penalty terms are:
\begin{equation}
\phi_c =
\begin{cases}
1 & \text{if } v_c \leq \bar{v}_c \\
\max\bigl(0,\; 1 - (v_c - \bar{v}_c)\bigr) & \text{otherwise}
\end{cases}
\end{equation}
Since $\mathcal{J} \in [1, 2]$ when all constraints are met and
$\mathcal{J} \in [0, 1)$ otherwise, feasible configurations are
always ranked above infeasible ones. Within each regime, higher
success rate is favored.

We set $\bar{v}_c = 0$ for position, velocity, and torque
constraints, requiring zero violations. For torque rate, we allow
$\bar{v}_c = 0.2$, since the real robot enforces torque rate
limiting at 1\,kHz as a hardware safety layer; as long as the
learned policy does not rely on frequent high-rate torque switching,
occasional violations in simulation are acceptable.

  \subsection{Real-World Deployment}

We deploy learned policies on a Franka FR3 robot using the
\texttt{aiofranka}\cite{aiofranka} library, which provides an asynchronous interface
for real-time torque control. The deployment system consists of two nested control loops:

\noindent \textbf{Inner loop (1\,kHz).} A joint-space impedance
    controller computes torques as:
    \begin{equation}
        \boldsymbol{\tau} = \mathbf{K}_p (\mathbf{q}_{\text{des}} - \mathbf{q})
        - \mathbf{K}_d \dot{\mathbf{q}} + \tau_{\text{ff}}
    \end{equation}
    where $\mathbf{q}_{des} \in \mathbb{R}^7$ is the commanded joint
    position, $\mathbf{q}, \dot{\mathbf{q}} \in \mathbb{R}^7$ are
    the current joint positions and velocities, and
    $\mathbf{K}_p, \mathbf{K}_d \in \mathbb{R}^{7 \times 7}$ are
    diagonal stiffness and damping gain matrices. Before commanding
    the robot, torques are clamped to the torque limits, and
    torque rates are limited to
    $|\dot{\tau}_i| \leq 990$\,Nm/s.

\noindent \textbf{Outer loop (50\,Hz).} The learned policy outputs
    actions at 50\,Hz, which are converted to position setpoints
    $\mathbf{q}_{des}$ for the inner impedance controller.  

\subsection{Policy Frequency Ablation}
We perform an ablation across policy frequency (10Hz, 20Hz, and 100Hz, in addition to the nominal policy frequency of 50Hz). These experiments vary the amount of time that position commands are zero-order-held before a new position command is issued. We re-use the 50Hz system identification parameters across all policy frequency ablations (the physics simulation rate and real-world controller rate remain the same regardless of the policy frequency). We re-train policies for the new policy frequencies and roll them out on the real robot in the same way. The only change is that the outer loop outputs actions at the required frequency of the new policy.

\begin{figure*}[h]
  \centering
  
  \begin{minipage}[t]{0.33\textwidth}
    \vspace{0pt}
    \centering
    \subcaptionbox{\centering 50 Trajectories\label{fig:dataset50-appendix}}{%
      \includegraphics[width=\linewidth]{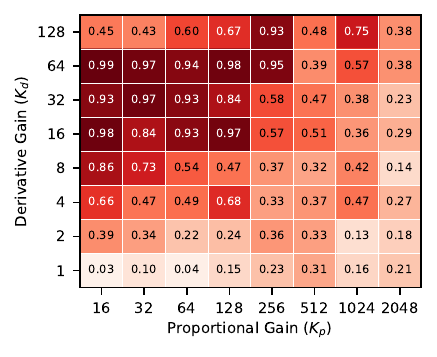}
    }
  \end{minipage}%
  \hfill%
  \begin{minipage}[t]{0.33\textwidth}
    \vspace{0pt}
    \centering
    \subcaptionbox{\centering 100 Trajectories\label{fig:dataset100-appendix}}{%
      \includegraphics[width=\linewidth]{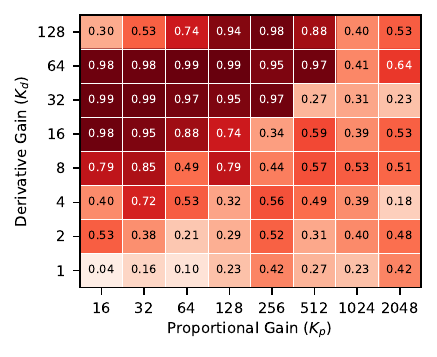}
    }
  \end{minipage}%
  \hfill%
  \begin{minipage}[t]{0.33\textwidth}
    \vspace{0pt}
    \centering
    \subcaptionbox{\centering 900 Trajectories \label{fig:dataset900-appendix}}{%
      \includegraphics[width=\linewidth]{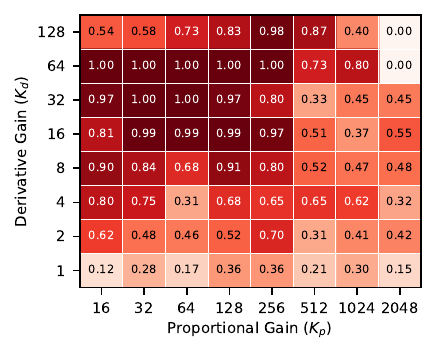}
    }
  \end{minipage}%
  \caption{\textbf{Behavior cloning performance across dataset size.} Success rate per gain setting for the Block Stack task. The preference for compliant and overdamped gain settings is maintained across dataset sizes (a-c).}
  \label{fig:appendix-bc-dataset-size}
  \vspace{-5pt}
\end{figure*}

\begin{figure*}[h]
  \centering
  
  \begin{minipage}[t]{0.33\textwidth}
    \vspace{0pt}
    \centering
    \subcaptionbox{\centering Regression\label{fig:regression-appendix}}{%
      \includegraphics[width=\linewidth]{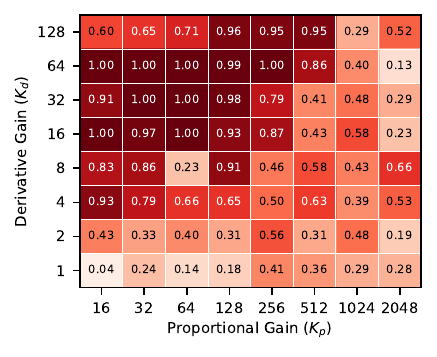}
    }
  \end{minipage}%
  \hfill%
  \begin{minipage}[t]{0.33\textwidth}
    \vspace{0pt}
    \centering
    \subcaptionbox{\centering VAE\label{fig:vae-appendix}}{%
      \includegraphics[width=\linewidth]{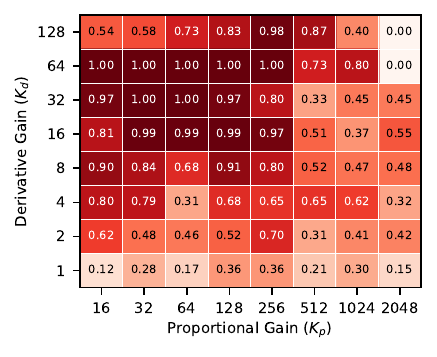}
    }
  \end{minipage}%
  \hfill%
  \begin{minipage}[t]{0.33\textwidth}
    \vspace{0pt}
    \centering
    \subcaptionbox{\centering Diffusion Policy \label{fig:diffusion-appendix}}{%
      \includegraphics[width=\linewidth]{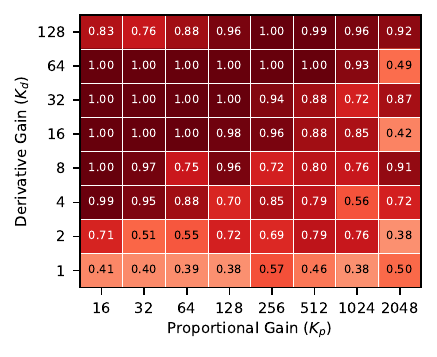}
    }
  \end{minipage}%
  \caption{\textbf{Behavior cloning performance across policy architectures.} Success rate per gain setting for the Block Stack task. The preference for compliant and overdamped gain settings is maintained across policy architectures (a-c).}
  \label{fig:appendix-bc-architecture}
  \vspace{-5pt}
\end{figure*}

\begin{figure*}[h]
  \raggedright
  
  \begin{minipage}[t]{0.33\textwidth}
    \vspace{0pt}
    \centering
    \subcaptionbox{\centering No Action Chunking\label{fig:nochunk-appendix}}{%
      \includegraphics[width=\linewidth]{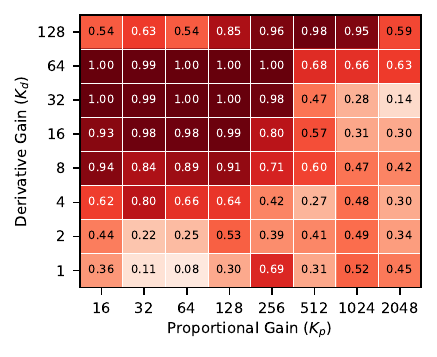}
    }
  \end{minipage}%
  \hspace{2em}%
  \begin{minipage}[t]{0.33\textwidth}
    \vspace{0pt}
    \centering
    \subcaptionbox{\centering Action Chunk Size 10\label{fig:chunk-appendix}}{%
      \includegraphics[width=\linewidth]{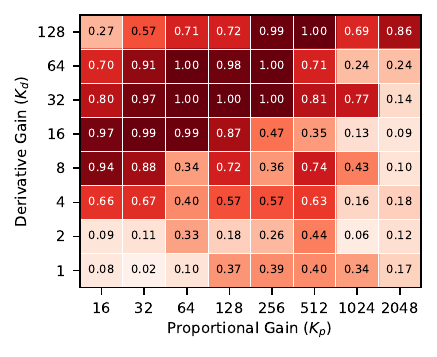}
    }
  \end{minipage}%
 
  \caption{\textbf{Behavior cloning performance across action chunk size.} Success rate per gain setting for the Block Stack task. The preference for compliant and overdamped gain settings is observed when predicting both single actions (a) and action chunks (b).}
  \label{fig:appendix-chunking}
  \vspace{-5pt}
\end{figure*}

\begin{figure*}[h]
  \raggedright
  
  \begin{minipage}[t]{0.33\textwidth}
    \vspace{0pt}
    \centering
    \subcaptionbox{\centering Absolute Joint Action\label{fig:abs-appendix}}{%
      \includegraphics[width=\linewidth]{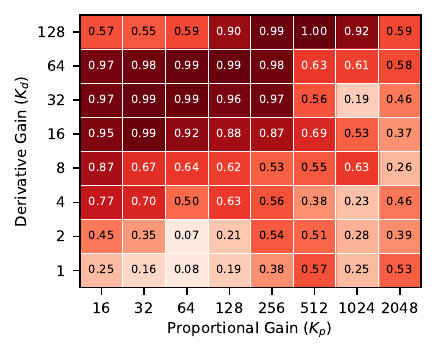}
    }
  \end{minipage}%
  \hspace{2em}%
  \begin{minipage}[t]{0.33\textwidth}
    \vspace{0pt}
    \centering
    \subcaptionbox{\centering Delta Joint Action\label{fig:delta-appendix}}{%
      \includegraphics[width=\linewidth]{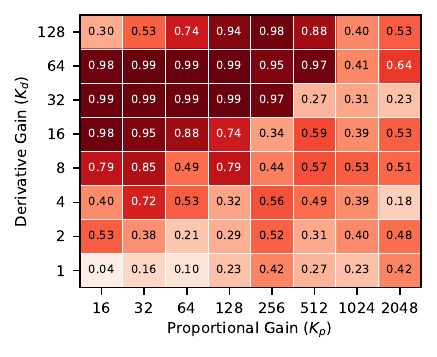}
    }
  \end{minipage}%
 
  \caption{\textbf{Behavior cloning performance across action representations.} Success rate per gain setting for the Block Stack task. The preference for compliant and overdamped gain settings is observed when predicting both absolute (a) and relative (b) joint position actions.}
  \label{fig:appendix-action}
  \vspace{-5pt}
\end{figure*}

\begin{figure*}[h]
  \raggedright
  
  \begin{minipage}[t]{0.33\textwidth}
    \vspace{0pt}
    \centering
    \subcaptionbox{\centering 10Hz Control Frequency\label{fig:10hz-appendix}}{%
      \includegraphics[width=\linewidth]{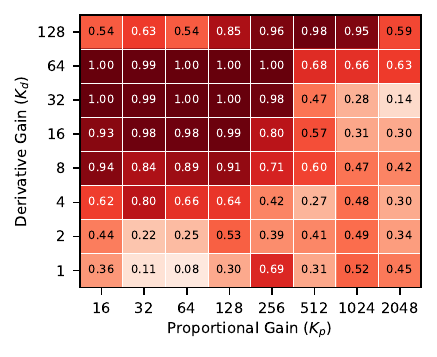}
    }
  \end{minipage}%
  \hspace{2em}%
  \begin{minipage}[t]{0.33\textwidth}
    \vspace{0pt}
    \centering
    \subcaptionbox{\centering 50Hz Control Frequency\label{fig:50hz-appendix}}{%
      \includegraphics[width=\linewidth]{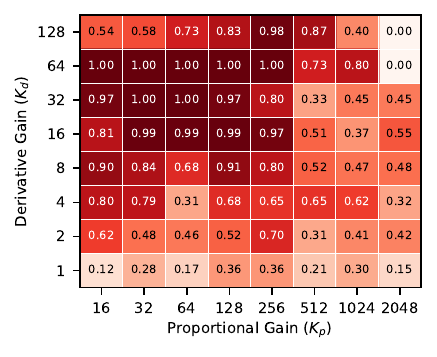}
    }
  \end{minipage}%
 
  \caption{\textbf{Behavior cloning performance across control frequencies.} Success rate per gain setting for the Block Stack task. The preference for compliant and overdamped gain settings is observed when predicting actions at 10Hz (a) and 50Hz (b).}
  \label{fig:appendix-frequency}
  \vspace{-5pt}
\end{figure*}

\captionsetup{font=small}
\captionsetup[subfigure]{font=small, skip=2pt}
\captionsetup[subfigure]{labelformat=parens}

\begin{figure*}[t]
\centering

\begin{minipage}[t]{0.50\textwidth}
  \vspace{0pt}
  \centering

  \subcaptionbox{Block Stacking with UR\label{fig:scaling-handover}}{%
    \includegraphics[width=\linewidth]{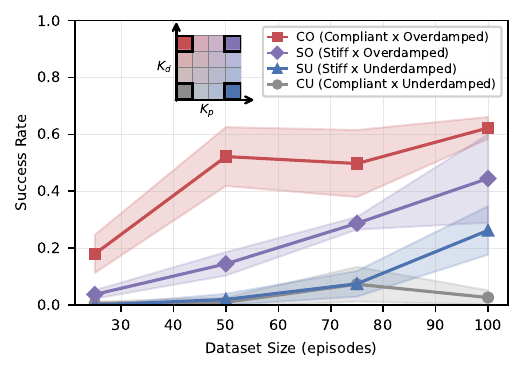}

  }
\end{minipage}
\hfill
\begin{minipage}[t]{0.48\textwidth}
  \vspace{0pt}
  \centering

  \subcaptionbox{Block Stacking with OSC\label{fig:scaling-dishrack-unload}}{%
    \includegraphics[width=0.48\linewidth]{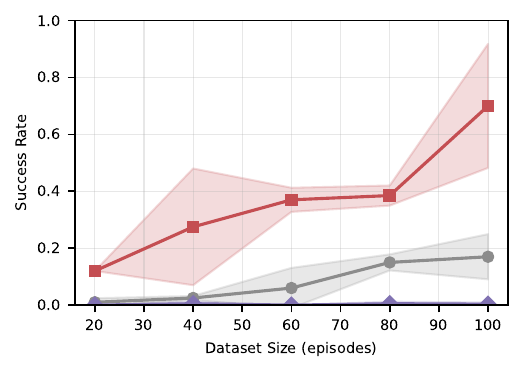}
  }\hfill
  \subcaptionbox{Dishrack Loading\label{fig:scaling-dishwasher}}{%
    \includegraphics[width=0.48\linewidth]{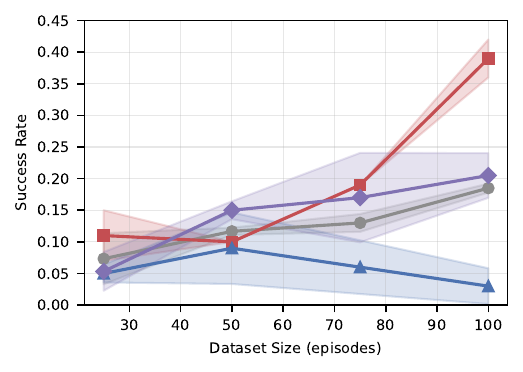}
  }

  \vspace{2mm}

  \subcaptionbox{Dishwasher Opening\label{fig:scaling-mug-hang}}{%
    \includegraphics[width=0.48\linewidth]{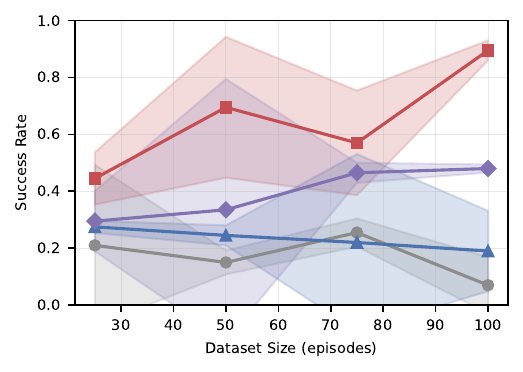}
  }\hfill
  \subcaptionbox{Mug Hanging\label{fig:scaling-ur-stack}}{%
    \includegraphics[width=0.48\linewidth]{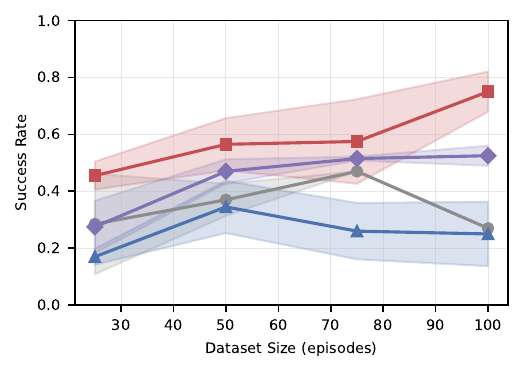}
  }
\end{minipage}

\caption{\textbf{Offline imitation learning scales more favorably under compliant and overdamped gains.}
Success rate as a function of dataset size across tasks and robot embodiments.
Policies trained with low stiffness and high damping achieve higher success with fewer demonstrations,
while stiff or weakly damped controllers exhibit poorer data scaling.}
\label{fig:bc-scaling}
\end{figure*}

\subsection{Sim-to-Real Analysis}

We compute the NN error as:
\begin{equation}
\label{eq:nn-error}
\mathcal{E}_{\text{NN}} = \text{RMS}\left( \pi_\theta(\mathbf{s}_t^{\text{real}}) - \pi_\theta(\mathbf{s}_t^{\text{sim}}) \right)
\end{equation}
This measures the trajectory-wise difference between the policy's outputs
in simulation and on the real robot, when the initial and goal
configurations are matched. The NN error for each of our reaching tasks
(Figure~\ref{fig:appendix-sim2real_nn}) is well correlated with the
trajectory error (Figure~\ref{fig:appendix-sim2real_trajectory}),
suggesting that the sim-to-real gap is primarily caused by the policy
receiving out-of-distribution states on the real robot, rather than
instability in the low-level controller.

\begin{figure*}[h]
  \centering
  
  \begin{minipage}[t]{0.33\textwidth}
    \vspace{0pt}
    \centering
    \subcaptionbox{\centering Joint-Reach Sim2Real trajectory error\\without Domain Randomization\label{fig:s2r-joint-appendix}}{%
      \includegraphics[width=\linewidth]{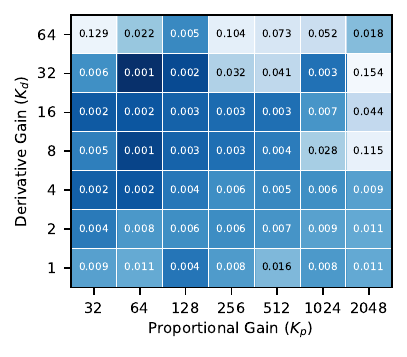}
    }
  \end{minipage}%
  \hfill%
  \begin{minipage}[t]{0.33\textwidth}
    \vspace{0pt}
    \centering
    \subcaptionbox{\centering Joint-Reach Sim2Real trajectory error\\with Domain Randomization\label{fig:s2r-dr-appendix}}{%
      \includegraphics[width=\linewidth]{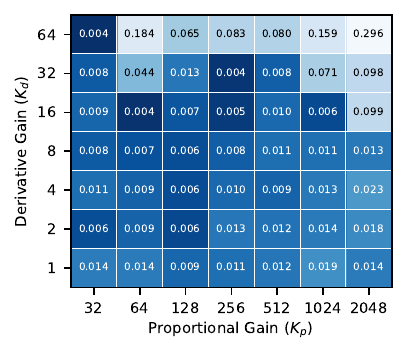}
    }
  \end{minipage}%
  \hfill%
  \begin{minipage}[t]{0.33\textwidth}
    \vspace{0pt}
    \centering
    \subcaptionbox{\centering EE-Reach Sim2Real trajectory error\\without Domain Randomization\label{fig:s2r-ee-appendix}}{%
      \includegraphics[width=\linewidth]{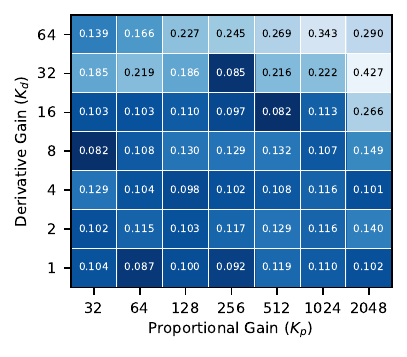}
    }
  \end{minipage}%
  \caption{\textbf{Stiff and overdamped gain settings reduce sim2real transferability.} The Sim2Real trajectory error (Eq.\ref{eq:traj_fidelity}) is consistently larger (light blue) in the stiff and overdamped regime (a-c).}
  \label{fig:appendix-sim2real_trajectory}
  \vspace{-5pt}
\end{figure*}

\begin{figure*}[h]
  \centering
  
  \begin{minipage}[t]{0.33\textwidth}
    \vspace{0pt}
    \centering
    \subcaptionbox{\centering Joint-Reach Sim2Real NN error\\without Domain Randomization\label{fig:s2r-joint-nn}}{%
      \includegraphics[width=\linewidth]{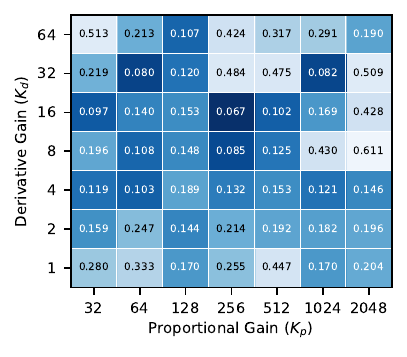}
    }
  \end{minipage}%
  \hfill%
  \begin{minipage}[t]{0.33\textwidth}
    \vspace{0pt}
    \centering
    \subcaptionbox{\centering Joint-Reach Sim2Real NN error\\with Domain Randomization\label{fig:s2r-dr-nn}}{%
      \includegraphics[width=\linewidth]{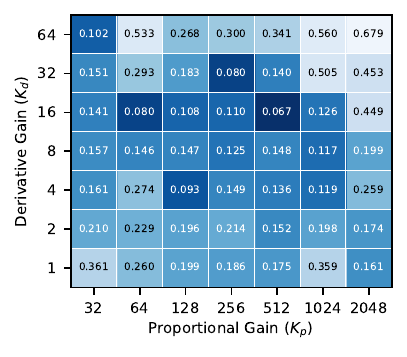}
    }
  \end{minipage}%
  \hfill%
  \begin{minipage}[t]{0.33\textwidth}
    \vspace{0pt}
    \centering
    \subcaptionbox{\centering EE-Reach Sim2Real NN error\\without Domain Randomization\label{fig:s2r-ee-nn}}{%
      \includegraphics[width=\linewidth]{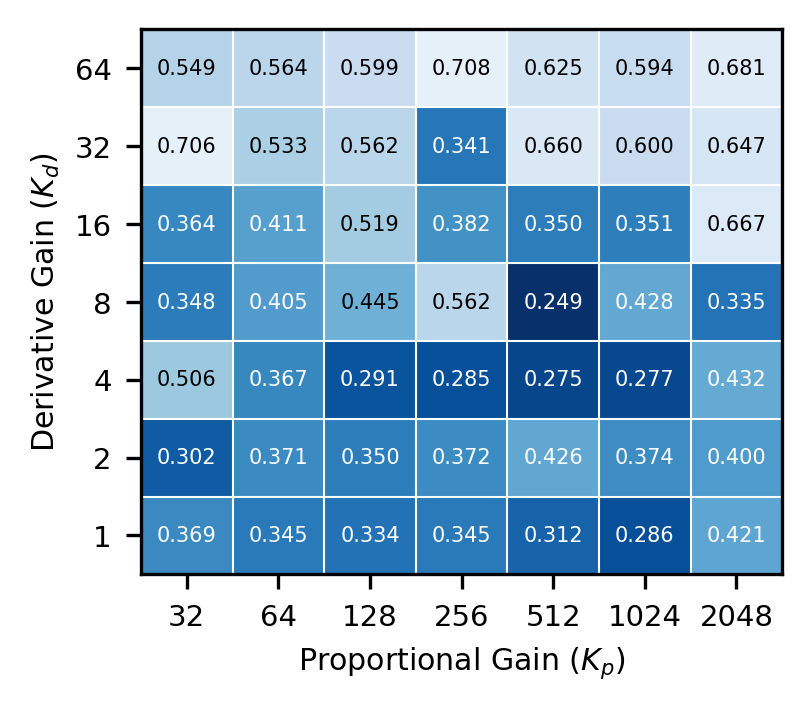}
    }
  \end{minipage}%
  \caption{\textbf{Stiff and overdamped gains increase Sim2Real NN error.} The Sim2Real NN error (Eq.\ref{eq:nn-error}) is consistently larger (light blue) in the stiff and overdamped regime (a-c).}
  \label{fig:appendix-sim2real_nn}
  \vspace{-5pt}
\end{figure*}

\subsubsection{Statistical Significance Analysis} \label{sec:appendix-sim2real-stats}

We provide a formal statistical analysis to verify that the stiff-overdamped gain region $\mathcal{G}^{\text{SO}}$ produces significantly larger sim-to-real trajectory error than its complement $\mathcal{G} \setminus \mathcal{G}^{\text{SO}}$ across all three sim-to-real conditions. For each gain cell, we compute the trajectory error (Eq.~\ref{eq:traj_fidelity}) averaged over 30 real-world rollouts.

\myparagraph{OLS Regression.} We fit ordinary least squares regression on log-transformed trajectory error with $\log_2 \mathbf{K}_\text{p}$ and $\log_2 \mathbf{K}_\text{d}$ as predictors. Across all conditions, both coefficients $\beta_{\mathbf{K}_\text{p}}$ and $\beta_{\mathbf{K}_\text{d}}$ are consistently positive (Table~\ref{tab:sim2real_statistical_analysis}), confirming that higher stiffness and higher damping are significant predictors of increased sim-to-real error.

\myparagraph{Mann-Whitney U Test.} We apply one-sided Mann-Whitney U tests with Bonferroni correction ($\alpha_{\text{adj}} \approx 0.017$, correcting for 3 conditions) under the null hypothesis:
\begin{equation}
\mathcal{H}_0 \colon \varepsilon(\mathcal{G}^{\text{SO}}) \leq \varepsilon(\mathcal{G} \setminus \mathcal{G}^{\text{SO}})
\end{equation}
As shown in Table~\ref{tab:sim2real_statistical_analysis}, $\mathcal{H}_0$ is rejected for every condition with $p \ll \alpha_{\text{adj}}$, providing strong evidence that the stiff-overdamped regime yields significantly larger sim-to-real transfer error.

\begin{table}[h]
\centering
\caption{
Statistical analysis of Sim2Real results. Mean trajectory error for stiff-overdamped ($\mathcal{G}^{\text{SO}}$) vs.\ other gain regions, OLS regression coefficients on $\log_2 \mathbf{K}_\text{p}$ and $\log_2 \mathbf{K}_\text{d}$, and Bonferroni-corrected one-sided Mann-Whitney U test $p$-values. $\mathcal{H}_0$ is rejected in all cases.
}
\vspace{-5pt}
\label{tab:sim2real_statistical_analysis}
\footnotesize
\setlength{\tabcolsep}{2pt}
\begin{tabular}{@{}l cc cc c@{}}
\toprule
& \multicolumn{2}{c}{\textbf{Mean Error}} & \multicolumn{2}{c}{\textbf{OLS Reg.}} & \textbf{Mann-Whitney} \\
\cmidrule(lr){2-3} \cmidrule(lr){4-5} \cmidrule(lr){6-6}
\textbf{Condition} & $\mathcal{G}^{\text{SO}}$ & $\mathcal{G} \setminus \mathcal{G}^{\text{SO}}$ & $\beta_{\mathbf{K}_{\text{p}}}$ & $\beta_{\mathbf{K}_{\text{d}}}$ & $p$-value \\
\midrule
Joint-Reach (no DR) & \textbf{0.043} & 0.010 & $+0.298$ & $+0.087$ & $1.9 \times 10^{-36}$ \\
Joint-Reach (with DR) & \textbf{0.060} & 0.018 & $+0.354$ & $+0.215$ & $1.7 \times 10^{-20}$ \\
EE-Reach (no DR) & \textbf{0.198} & 0.123 & $+0.063$ & $+0.127$ & $3.5 \times 10^{-36}$ \\
\bottomrule
\end{tabular}
\end{table}

\end{document}